\documentclass{article}

\pdfoutput=1

\usepackage{arxiv}

\usepackage[utf8]{inputenc} 
\usepackage[T1]{fontenc}    
\usepackage{hyperref}       
\usepackage{url}            
\usepackage{booktabs}       
\usepackage{amsfonts}       
\usepackage{nicefrac}       
\usepackage{microtype}      
\usepackage{lipsum}		
\usepackage{graphicx}
\usepackage{doi}

\usepackage{multicol}
\usepackage{mathtools}
\usepackage{subfigure}
\usepackage{amssymb}
\usepackage{amsthm}

\usepackage{algorithm}
\usepackage{algorithmic}

\def \wb {{\mathbf w}} 
\def \nb {{\mathbf n}}
\def \cM {{\mathcal{M}}}

\def \cR {{\mathbb{R}}}

\def\Er{\mathcal{E}}
\def\heteG{G}
\def\heteB{B}
\def\client{N} 
\def\sample{S} 
\def\step{K} 
\def\smooth{\beta} 
\newcommand{\lip}{L} 
\newcommand{\clip}{\mathcal{L}}
\def\batch{b}
\def\epoch{E}

\newtheorem{assumption}{Assumption}
\newtheorem{proposition}{Proposition}
\newtheorem{lemma}{Lemma}
\newtheorem{theorem}{Theorem}
\newtheorem{definition}{Definition}
\newtheorem{remark}{Remark}
\newtheorem{corollary}{Corollary}
\newtheorem*{metatheorem*}{Meta Theorem}

\usepackage{xcolor}
\usepackage{tikz}
\definecolor{header1}{cmyk}{.9,.5,0,.35}

\usepackage[capitalise,nameinlink]{cleveref}

\title{Differentially Private Federated Learning with Laplacian Smoothing}

\date{}

\author{Zhicong Liang\\
	Department of Mathematics\\
	Hong Kong University of Science and Technology\\
	\texttt{zliangak@connect.ust.hk} \\
	\And
	Bao Wang \\
	Department of Mathematics\\
	University of Utah\\
	\texttt{bwang@math.utah.edu} \\
	\AND
	Quanquan Gu \\
	Department of Computer Science\\
	University of California, Los Angeles \\
	\texttt{qgu@cs.ucla.edu} \\
	\And
	Stanley Osher \\
	Department of Mathematics\\
	University of California, Los Angeles \\
	\texttt{sjo@math.ucla.edu} \\
	\And
	Yuan Yao \\
	Department of Mathematics \\
	Hong Kong University of Science and Technology \\
	\texttt{yuany@ust.hk} \\
}

\renewcommand{\shorttitle}{DP-Fed-LS}

\hypersetup{
pdftitle={Differentially Private Federated Learning with Laplacian Smoothing},
pdfauthor={Zhicong Liang, Bao Wang, Quanquan Gu, Stanley Osher, Yuan Yao},
pdfkeywords={Differential privacy ,Federated learning, Laplacian smoothing},
}

\begin{document}
\maketitle

\begin{abstract}
Federated learning aims to protect data privacy by collaboratively learning a model without sharing private data among users. However, an adversary may still be able to infer the private training data by attacking the released model. Differential privacy provides a statistical protection against such attacks at the price of 
significantly degrading the accuracy or utility of the trained models. In this paper, we investigate a utility enhancement scheme based on Laplacian smoothing for differentially private federated learning (DP-Fed-LS), where the parameter aggregation with injected Gaussian noise is improved in statistical precision without losing privacy budget. 
Our key observation is that the aggregated gradients in federated learning often enjoy a type of smoothness, i.e. sparsity in the graph Fourier basis with polynomial decays of Fourier coefficients as frequency grows, which can be exploited by the Laplacian smoothing efficiently. Under a prescribed differential privacy budget, convergence error bounds with tight rates are provided for DP-Fed-LS with uniform subsampling of heterogeneous Non-IID data, revealing possible utility improvement of Laplacian smoothing in effective dimensionality and variance reduction, among others. Experiments over MNIST, SVHN, and Shakespeare datasets show that the proposed method can improve model accuracy with DP-guarantee and membership privacy under both uniform and Poisson subsampling mechanisms.
\end{abstract}


\section{Introduction}\label{sec:introduction}
In recent years, we have already witnessed the great success of machine learning (ML) algorithms in handling large-scale and high-dimensional data
\cite{he:2016, devlin:bert, silver:2016, berner2019dota, senior2020improved}. Most of these models are trained 
in a centralized manner
by gathering all data into a single database. However, in applications like mobile keyboard development \cite{hard2018federated} and vocal classifier such as ``Hey Siri" \cite{apple}, sensitive data are distributed in the devices of users, who are not willing to share their own data with others. Federated learning (FL), proposed in \cite{mcmahan:2016}, provides a solution that data owners can collaboratively learn a useful model without disclosing their private data. In FL, a server, and multiple data owners, referred to as clients, are involved in maintaining a global model. They no longer share the private data but the updated models trained on these data. 
In each communication round, the server will distribute the latest global model to a random subset of selected clients (active clients), who will perform learning starting from the received global model based on their private data, and then upload the locally updated models back to the server. The server then aggregates these local models to construct a new global model and start another communication round until convergence. There have been various studies on such a distributed learning since its inception \cite{konevcny:2016, suresh:2017, smith:2017}.

In some cases, however, 
federated learning is not sufficient to protect the sensitive data by 
simply decoupling the model training from the direct access to the raw training data 
\cite{Shokri:2017,fredrikson:2014,Fredrikson:2015}. 
Information about raw data could be identified from a well-trained model. 
In some extreme cases, a neural network can even memorize the whole training set with its huge number of parameters. 
For example, an adversary may infer the presence of particular records in training \cite{Shokri:2017} or even recover the identity (e.g. face images) in the training set by attacking the released model \cite{Fredrikson:2015, fredrikson:2014}.
Differential privacy (DP) provides us with a solution to defend against these threats \cite{Dwork:2004, dwork2006our}. DP guarantees privacy in a statistical way that the well-trained models are not sensitive to the change of an individual record in the training set. This task is usually fulfilled by adding noise, calibrated to the model's sensitivity, to the outputs or the updates. 

One major deficiency of DP lies in its potential significant degradation of the utility of the models due to the noise injection.
Laplacian smoothing (LS) has recently been shown to be a good choice for reducing noise in noisy gradient, e.g. in stochastic gradient descent (SGD)
\cite{Osher:2018}, and thus promising for utility improvement in machine learning with DP \cite{wang:2019}.

In this paper, we introduce Laplacian smoothing to improve the utility of the differentially private federated learning (DP-Fed) while maintaining the same DP budget. 
The major contributions of our work are fourfold:
\begin{itemize}
    \item Laplacian smoothing of the \textit{federated average of gradients} is introduced to the differentially private federated learning, based on the observation that aggregated gradients in federated learning are often smooth or sparse in Fourier coefficients with polynomial decays. Laplacian smoothing can reduce variance with improved estimates of such gradients. We denote the proposed algorithm as DP-Fed-LS.
    
    \item Tight upper bounds are established for differential privacy budget guarantees. Our DP bounds are based on a new set of closed-form privacy bounds derived for both uniform and Poisson subsampling mechanisms, which are tighter than existing results in previous studies \cite{wanglingxiao:2019, bun:2018, Mironov2019sampled} while relaxing their requirements. 

    \item Convergence bounds are developed for DP-Fed-LS in strongly-convex, general-convex, and non-convex settings under our differential privacy budget bounds. The rates on convergence and communication complexity match those on federated learning without DP \cite{karimireddy2020scaffold}, while our results extend to include the effect of differential privacy and Laplacian smoothing; as well as our rates match the ones of empirical risk minimization (ERM) via SGD with differential privacy in a centralized setting \cite{bassily2019private,wanglingxiao:2019}.
    
    \item The utility of Laplacian smoothing in DP-Fed is demonstrated by training a logistic regression model over MNIST, a convolutional neural network (CNN) over extended SVHN, in an \textbf{IID} 
    fashion, and a long short-term memory (LSTM) model over Shakespeare dataset in a \textbf{Non-IID} setting. These experiments show that DP-Fed-LS improves accuracy while providing at least the same DP-guarantees and membership privacy as DP-Fed with two subsampling mechanisms across different datasets.
\end{itemize}

\begin{table}
\renewcommand{\arraystretch}{1.5}
\centering
\caption{Utility guarantee of $(\varepsilon,\delta)$ differential privacy (upper part) and rate of communication round needed to achieve $\epsilon$ accuracy (lower part) for $\mu$ strongly-convex and non-convex optimization problems. $\ddagger$ denotes that logarithmic factors are ignored here. See \ref{sec:table_compare} for more details. $\dagger$ denotes that no client subsampling is used. In full participation scenarios, $\tau=1$ and  $\log(\sample)=\log(\client)$. $\ast$ after DP-Fed-LS further denotes the specific setting of \textbf{IID} $(\heteG=0)$ with $\step \gg 1$. For centralized settings \cite{bassily:2014, wang2017differentially, wang:2019}, $N$ denotes the number of data points, while in federated learning, $N$ denotes the number of clients. DP-SRM \cite{wanglingxiao:2019} is a distributed setting where $\client$ and $\tilde{n}$ denote the number of clients and number of samples owned by each client, respectively. They consider data-level DP while we consider user-level DP. 
}
\begin{tabular}{cccc}
\hline
Method & $\mu$ strongly-convex & non-convex  \\ 
\hline
DP-SGD \cite{bassily:2014} & \footnotesize{$\frac{d\log^3(\client/\delta)}{\mu \client^2 \varepsilon^2}$} & $-$ \\

DP-SVRG $\text{\cite{wang2017differentially}}^\dagger$ & \footnotesize{$\frac{d\log(\client)\log(1/\delta)}{\mu \client^2\varepsilon^2}$} & \footnotesize{$\frac{\sqrt{d\log(1/\delta)}}{\client \varepsilon}$}\\

DP-SRM \cite{wanglingxiao:2019} & $-$ & \footnotesize{$\frac{\sqrt{d\log(1/\delta)}}{\client \tilde{n}\varepsilon}$} \\

DP-SGD-LS \cite{wang:2019} & $-$ & \footnotesize{$\frac{\sqrt{\tilde{d}_\sigma \log(1/\delta)}}{\client \varepsilon}$} \\

DP-Fed-LS$^{\ast}$ & \footnotesize{$\frac{d_\sigma \log(\sample)\log(1/\delta)}{\mu_\sigma \client^2 \varepsilon^2}$} & \footnotesize{$\frac{\sqrt{\tilde{d}_\sigma \log(1/\delta)}}{\client \varepsilon}$} \\

\hline

Fed-Avg $\text{\cite{li:2019}}^{\dagger}$ & \footnotesize{$\frac{\varsigma^2(0)}{\mu^2\client \step \epsilon} \!+\! \frac{\heteG^2 \step}{\mu^2 \varepsilon}$} & $-$
\\

Fed-Avg $\text{\cite{khaled2020tighter}}^{\dagger}$ & \footnotesize{$\frac{\varsigma^2(0) + \heteG^2}{\mu \client \step \epsilon}\!+\! \frac{\varsigma(0)+\heteG}{\mu\sqrt{\epsilon}} \!+\! \frac{\client \heteB}{\mu} ^\ddagger$} & $-$
\\

Fed-Avg \cite{karimireddy2020scaffold} & \footnotesize{$\frac{\varsigma^2(0)}{\mu \sample \step \epsilon} \!+\! \frac{(1-\tau)\heteG^2}{\mu\sample\epsilon} \!+\! \frac{\heteG
}{\mu\sqrt{\epsilon}} \!+\! \frac{\heteB^2}{\mu} ^\ddagger$} & \footnotesize{$\frac{\varsigma^2(0)}{\sample \step \epsilon^2} \!+\! \frac{(1-\tau)\heteG^2}{\sample \epsilon^2} \!+\! \frac{\heteG}{\epsilon^{2/3}} \!+\! \frac{\heteB^2}{\epsilon}$}
\\

DP-Fed-LS &  \footnotesize{$\frac{\varsigma^2(\sigma)}{\mu_\sigma \sample \step \epsilon} \! + \! \frac{(1-\tau)\heteG^2}{\mu_\sigma\sample\epsilon}\! +\! \frac{\heteG
}{\mu_\sigma\sqrt{\epsilon}} \!+\! \frac{(1+4\sigma)\heteB^2}{\mu_\sigma} \!+\! \frac{d_\sigma \lip\nu_1^2}{\mu_\sigma\sample^2\epsilon}^\ddagger$} & \footnotesize{$\frac{\varsigma^2(\sigma)}{\sample \step \epsilon^2} \!+\! \frac{(1-\tau)\heteG^2}{\sample \epsilon^2} \!+\! \frac{\heteG}{\epsilon^{2/3}} \!+\! \frac{(1+4\sigma^2)\heteB^2}{\epsilon} \!+\! \frac{\tilde{d}_\sigma \lip^2 \nu_1^2}{\sample^2\epsilon^2}$} \\
\hline

\end{tabular}
\label{tbl:rate-comparision}
\end{table}

\section{Background and Related Works}
\label{sec:background}
\

\textit{\textcolor{header1}{Risk of Federated Learning.}} Despite its decoupling of training from direct access to raw data, federated learning may suffer from the risk of privacy leakage by unintentionally allowing malicious clients to participate in the training \cite{hitaj:2017, melis2019exploiting, zhu:2019}.
In particular, model poisoning attacks are introduced in \cite{bagdasaryan2020backdoor, bhagoji:2019}. Even though we can ensure the training is private, the released model may also leak sensitive information about the training data. Fredrikson et al. introduce the model inversion attack that can infer sensitive features or even recover the input given a model \cite{fredrikson:2014,Fredrikson:2015}. Membership inference attacks can determine whether a record is in the training set by leveraging the ubiquitous overfitting of machine learning models \cite{Shokri:2017, yeom:2018, sablayrolles:2019}. In these cases, simply decoupling the training from direct access to private data is insufficient to guarantee data privacy.

\textit{\textcolor{header1}{Differential Privacy.}} Differential privacy comes as a solution for privacy protection. Dwork et al. consider output perturbation with noise calibrated according to the sensitivity of the function \cite{Dwork:2004, Dwork2009calibrating}. 
Gradient perturbation \cite{bassily:2014, Abadi:2016} receives lots of recent attention in ML applications since it admits the public training process and ensures DP guarantee even for a non-convex objective. 
Feldman et al. argue that one can amplify the privacy guarantee by hiding the intermediate results of contractive iterations \cite{Feldman:2018}. Papernot et al. propose PATE that bridges the target model and training data by multiple teacher models \cite{papernot:2017, papernot:2018}. Mironov proposes a natural relaxation of DP based on R\'enyi divergence (RDP), which allows tighter analysis of composite heterogeneous mechanisms \cite{mironov:2017renyi}. Wang et al. provide a tight 
numerical upper bound on RDP parameters for randomized mechanism with uniform subsampling \cite{wang2019subsampled}. Furthermore, they extend their bound to the case of Poisson subsampling \cite{zhu2019poission}, which is the same as the one in \cite{Mironov2019sampled}. Our work is based on these two numerical results, and we derive new closed-form bounds which are more precise or tighter than previous works \cite{wanglingxiao:2019,Mironov2019sampled,bun:2018}.

\textit{\textcolor{header1}{Differential Privacy in Distributed Settings.}} 
DP has been applied in many distributed learning scenarios. Pathak et al. propose the first DP training protocol in distributed setting \cite{pathak:2010}. Jayaraman et al. \cite{jayaraman:2018} reduce the noise needed in \cite{pathak:2010}.
Zhang et al. propose to decouple the feature extraction from the training process \cite{zhang:2019}, where clients only need to extract features with frozen pre-trained convolutional layers and perturb them with Laplace noise. 
However, this method needs to introduce extra edge servers besides the central server in the standard federated learning. 
Agarwal et al. \cite{agarwal:2018} further take both communication efficiency and privacy into consideration. 

Geyer et al. \cite{Geyer:2017} and McMahan et al. \cite{mcmahan:2017} consider a similar problem setting as this paper, which applies the Gaussian mechanism in federated learning to ensure DP. However, Geyer et al. \cite{Geyer:2017} only train models over MNIST, with repetition of the data across different clients, which is unrealistic in applications. McMahan et al. \cite{mcmahan:2018} use moment accountant in \cite{Mironov2019sampled, zhu2019poission}, and show that given a sufficiently large number of clients ($\sim$ 760K in their example), their models suffer no utility degradation.
However, in many scenarios, one has to deal with a much smaller number of clients, which will induce a large noise level with the same DP constraint, significantly reducing the utility of the models. This motivates us to leverage Laplacian smoothing to mitigate the utility degradation due to DP, broadening its scope of application. And we further provide convergence bounds and evaluate the membership privacy of our method by the membership inference attack.

\section{Differentially Private Federated Learning with Laplacian Smoothing} 
\label{sec:method}
In this section, we formulate the basic scheme of private (noisy) federated learning with Laplacian smoothing. Consider the following distributed optimization model,
\begin{equation*}
    \min_{w} f(w) \coloneqq \frac{1}{\client} \sum_{j=1}^\client f_j(w),
\end{equation*}
where $f_j$ represent the loss function of client $j$, and $N$ is the number of clients. Here $f_j(w) = \mathbb{E}_{x_j} f_j(w, x_j, y_j)$, where $\mathbb{E}_{x_j}$ is the expectation over the dataset of the $j$-th client. 

We propose \textit{differentially private federated learning with Laplacian smoothing} (DP-Fed-LS), which is summarized in Algorithm~\ref{algorithm:1}, to solve the above optimization problem. In each communication round $t$, the server distributes the global model $w^{t}$ to a selected subset out of $\client$ total clients. These selected (active) clients will perform $\step$ steps mini-batch SGD to update the models on their private data,
and send back the model update $\Delta^{t,\step}_j$s, from which the server will aggregate and yield a new global model $w^{t+1}$. This process will be repeated until the global model converges. We call a setting \textbf{IID} if data of different clients are sampled from the same distribution independently, otherwise we call it \textbf{Non-IID} \cite{mcmahan:2016}.

In each update of the mini-batch SGD, we bound the local model $w_j^{t,i}, i\in[\step]$ within a $\clip$-ball ($\clip>0$) centering around $w^{t}$ by clipping: clip($v,\clip$) $\leftarrow {v} / {\max}(1, \| v \|_2/\clip)$. 
In each round, we regard the aggregation of locally-trained models as the  \textit{federated average of gradients}, where we add calibrated Gaussian noise ${\bf{n}} \sim \mathcal{N}(\textbf{0},\nu^2 \textbf{I})$ to guarantee DP. Then we apply Laplacian smoothing with a smoothing factor $\sigma$ on the noisy aggregated \textit{federated average of gradients} (Eq. (*) in Algorithm~\ref{algorithm:1}), to stabilize the training while preserving DP based on the post-processing lemma (Proposition 2.1  of \cite{Dwork:2014book}). It will reduces to DP-Fed if ${\bf{A}}_{\sigma}=\textbf{I}, \ \mbox{i.e.}\ \sigma=0$.

\begin{algorithm}[htbp]
   \caption{Differentially-Private Federated Learning with Laplacian Smoothing (DP-Fed-LS)}
   \label{algorithm:1}
   \small
\begin{multicols}{2}
\begin{algorithmic}
    \STATE \hskip-0.75em \textit{parameters:}
    \STATE  activate client fraction $\tau \in (0,1]$
    \STATE  total communication round $T$
    \STATE  clipping parameter $\clip$
    \STATE local and global learning rate $\eta_l, \eta_g$
    \STATE  noise level $\nu$

   \bigskip
   
   \STATE \hskip-0.75em {\bfseries function} \textsc{ClientUpdate}($j$, $w^{t}$)
        \STATE $w_j^{t,0}$ $\leftarrow w^{t}$
        \FOR{$i=0$ {\bfseries to} $\step-1$}
            \STATE $g_j(w_j^{t, i}) \leftarrow \mbox{mini-batch gradient}$
            \STATE $w_j^{t,i+1} \leftarrow w^{t} + \textsc{Clip}\big( w_j^{t,i} - \eta_l g_j(w_j^{t,i}) - w^{t}, \clip \big)$
        \ENDFOR
        \STATE return $\Delta^{t}_{j} \leftarrow {w_j^{t, \step}}-{w^{t}} $
        
    \smallskip

   \STATE \hskip-0.75em {\bfseries function} \textsc{Clip}($v$, $\clip$) return ${v} / {\max}(1, \| v \|_2/\clip)$
   
   \smallskip
   
   \STATE \hskip-0.75em {\bfseries Server executes:}
   \STATE initialize $w^0$
   \FOR{$t=0$ {\bfseries to} $T-1$}
   \STATE $\mathcal{S}_{t}$ $\leftarrow$ (a random subset of clients selected by uniform or Possion subsampling with ratio $\tau$)
   \STATE $\sample \leftarrow |\mathcal{S}_t|$
   
   \FOR{client $j\in \mathcal{S}_{t}$ {\bfseries in parallel}}
   \STATE $\Delta^{t}_{j} \leftarrow \textsc{ClientUpdate}(j, w^{t})$ 
   \ENDFOR
   \STATE $\Delta^{t} \leftarrow \frac{\eta_g}{\sample}  \textbf{A}_\sigma^{-1} \big( \sum_{j=1}^\sample \Delta_j^{t} +  \mathcal{N}(\textbf{0}, \nu^2\textbf{I}) \big)  \hfill{(*)} $
   \STATE ${w^{t+1}} \leftarrow {w^{t}} + \Delta^{t} $
   \ENDFOR

   \STATE Output  $\bar{w}^T =\sum_{t=0}^T a_t w^t, \ a_t\in [0,1]$ \& $\sum a_t=1$
   

\end{algorithmic}
\end{multicols}
\end{algorithm}

\subsection{Laplacian Smoothing}
To understand the Laplacian smoothing in DP-Fed-LS, consider the following general iteration:
\begin{equation}
    {w}^{t+1} = {w}^{t}- \eta {\bf{A}}_{\sigma}^{-1}  \nabla f(w^t, x_{i_t}, y_{i_t}),
\end{equation}
where $\eta$ is the learning rate and $f(w, x_{i_t}, y_{i_t})$ is the loss of a given model with parameter $w$ on the training data $\{ x_{i_t},y_{i_t} \}$.
In Laplacian smoothing \cite{Osher:2018}, we let ${\bf{A}}_{\sigma}={\bf{I}}+\sigma {\bf{L}}$, where ${\bf L} \in \mathbb{R}^{d \times d}$ is the 1-dimensional Laplacian matrix of a cycle graph, i.e. ${\bf A}_\sigma$ a circulant matrix whose first row is $(1+2\sigma,-\sigma,0,\cdots,0,-\sigma)$ with $\sigma\geq 0$ being a constant. 
When $\sigma=0$, Laplacian smoothing stochastic gradient descent reduces to SGD. 

Laplacian smoothing can be effectively implemented by using the fast Fourier transform.
To be specific, for any 1-D signal $v$ (here $ \nabla f(w^t, x_{i_t}, y_{i_t})$), we would like to calculate $u=\textbf{A}_\sigma^{-1} v$. Since $v = \textbf{A}_\sigma u = u - \sigma d * u$, where $d = [-2,1,0,...,0,1]^T$ and $*$ denotes the convolutional operator. We have the following equality by exploiting the 1-D fast Fourier transform (FFT) 
\begin{equation}\label{eq:fft}
    \text{fft}(v) = \text{fft}(u) \cdot \big(1 - \sigma \cdot \text{fft}(d) \big),
\end{equation}
where $\cdot$ is pointwise multiplication. In other words, the Laplacian matrix ${\bf{L}}$ has eigenvectors defined by the Fourier basis, which diagonalizes convolutions via 1-D fast Fourier transform. Going back to Eq. (\ref{eq:fft}), we solve $u$ by applying the inverse Fourier transform
\begin{equation*}
    u = \text{ifft} \Big( \frac{\text{fft}({v})}{{1}-\sigma \cdot \text{fft}({d})} \Big).
\end{equation*}



The motivation behind Laplacian smoothing lies in that if the target parameter $v$ is smooth under Fourier basis, then when it is contaminated by Gaussian noise, i.e.
$ \tilde{v}=v+{\bf n}$, $v\in \mathbb{R}^d, {\bf n}\sim {\mathcal N}(\textbf{0},\nu^2 \textbf{I}),$
a smooth approximation of $\tilde{v}$ is helpful to reduce the noise. The Laplacian smoothing estimate is defined by 
\begin{equation}
    \hat{v}_{LS}:=\arg\min_u \|u-\tilde{v}\|^2 + \sigma \|\nabla u\|^2,
\end{equation}
where $\nabla$ is a 1-dimensional gradient operator such that $\bf{L}=\nabla^T \nabla$. It satisfies 
${\bf A}_\sigma \hat{v}_{LS} = \tilde{v} = v+{\bf n}$.
The following proposition characterizes the prediction error of Laplacian smoothing estimate $\hat{v}_{LS}$.

\begin{proposition}[Bias-Variance decomposition] \label{prop:ls-risk} Let the graph Laplacian have eigen decomposition $\Delta {\bf e}_i = \lambda_i {\bf e}_i$ with eigenvalues $0=\lambda_1\leq \lambda_2\leq \ldots\leq \lambda_d$ and the first eigenvector ${\bf e_1}={\bf 1}/\sqrt{d}$. Then the mean square error (risk) of estimate $\hat{v}_{LS}$ admits the following decomposition,
\begin{equation*}
    \begin{aligned}
    \mathcal{R}(\hat{v}_{LS}):=\mathbb{E}\|\hat{v}_{LS} - v\|^2 
    &= \|({\bf I} - {\bf A}_\sigma^{\dagger}) v\|^2 + \mathbb{E} \|{\bf A}_\sigma^{\dagger} {\bf n}\|^2 \\
    &= \sum_{i} \frac{\sigma^2 \lambda_i^2}{(1+\sigma \lambda_i)^2} \langle v, {\bf e}_i\rangle^2 + \sum_{i} \frac{\nu^2}{(1+\sigma \lambda_i)^2} ,
    \end{aligned}
\end{equation*}
where the first term is called the {\bf{bias}} and the second term is called the {\bf{variance}}.
\end{proposition}

In the bias-variance decomposition of the risk above, if $\sigma=0$, the risk becomes bias-free with variance $d \nu^2$; if $\sigma>0$, bias is introduced while variance is reduced. The optimal choice of $\sigma$ must depend on an optimal trade-off between the bias and variance in this case. When the true parameter $v$ is smooth, in the sense that its projections $\langle v, {\bf e}_i\rangle \to 0$ rapidly as $i$ increases, the introduction of bias can be much smaller compared to the reduction of variance, hence the mean squared error (risk) can be reduced with Laplacian smoothing. A bias-variance trade-off with similar idea for graph neural network can be found in \cite{nt2019revisiting}. 

\subsection{Sparsity of aggregated gradients in the Fourier basis}
To verify that the true signal $v$ is smooth or sparse with respect to the Fourier basis, we show in Figure~\ref{fig:fft_of_signal} the magnitudes distribution in frequency domain of $v=\frac{1}{\sample} \sum_j \Delta_j^t$, the \textit{federated average of gradients} in non-DP federated learning under the fast Fourier transform. It is a typical example with experimental setting described in Section~\ref{sec:cnn} and four models at different training communication rounds ($t=1, 50, 100, 200$) are shown. One can see that from the log-log plot, as the communication round and frequency
grow, the magnitudes of Fourier coefficients demonstrate a power law decay with respect to the frequency, indicated by a linear envelope between $\log_{10}(\text{Magnitude})$ and $\log_{10}(\text{Frequency})$ when $\log_{10}(\text{Frequency})$ increases. In other words, it shows that the projections of magnitudes $\langle v, \textbf{e}_i \rangle \rightarrow 0$ at a polynomial rate when the frequency in Fourier basis is large enough, supporting the assumption above for variance reduction.

\begin{figure}[htbp]
    \centering
    \subfigure[]{\includegraphics[width=1.5in]{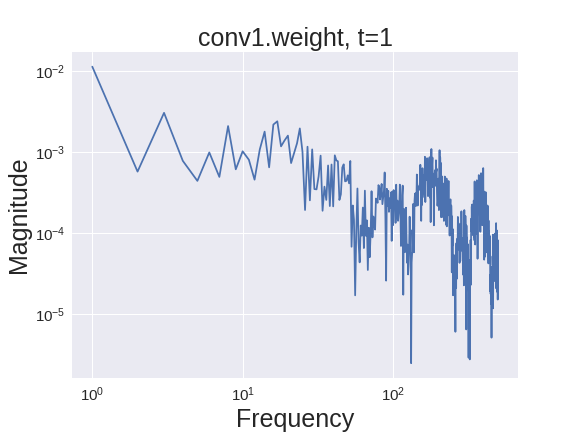}}
    \subfigure[]{\includegraphics[width=1.5in]{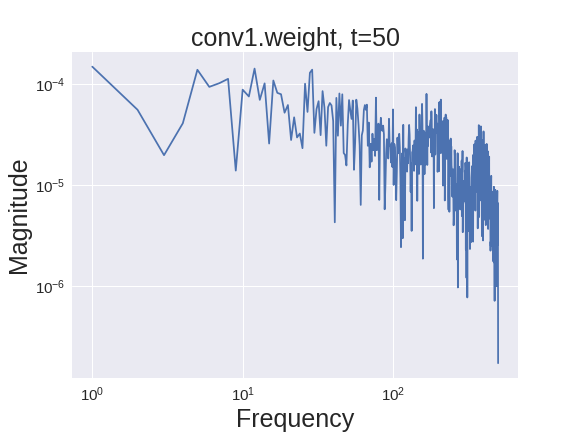}}
    \subfigure[]{\includegraphics[width=1.5in]{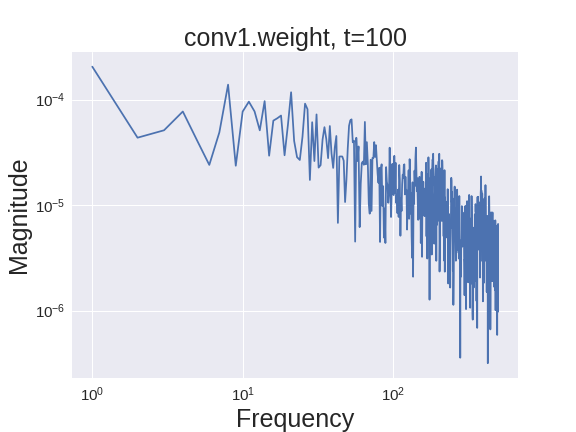}}
    \subfigure[]{\includegraphics[width=1.5in]{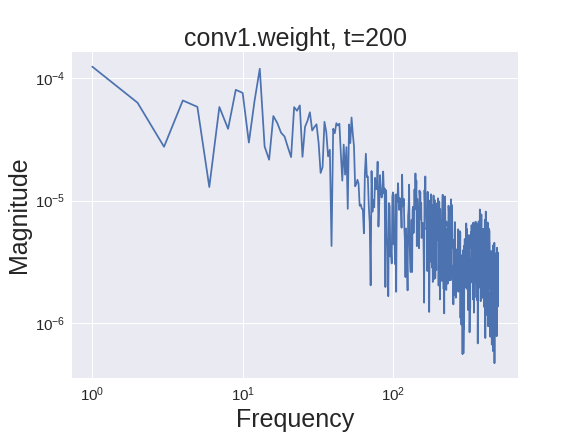}} \\
    \subfigure[]{\includegraphics[width=1.5in]{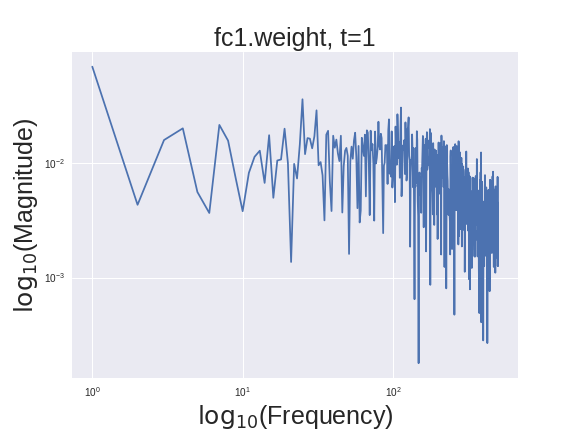}}
    \subfigure[]{\includegraphics[width=1.5in]{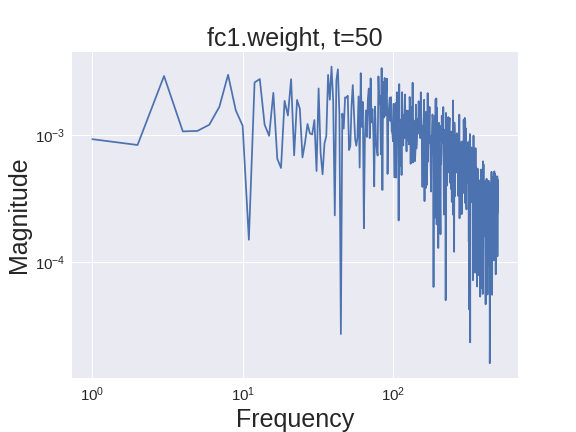}}
    \subfigure[]{\includegraphics[width=1.5in]{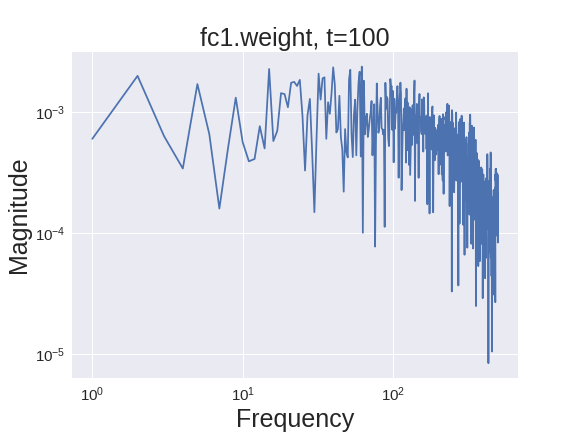}}
    \subfigure[]{\includegraphics[width=1.5in]{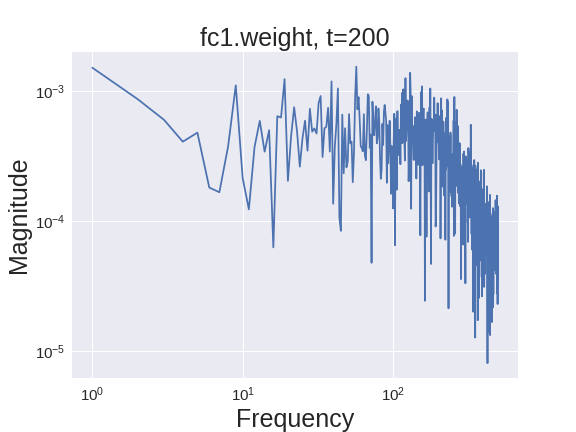}} \\
    \caption{Frequency distribution of federated average of gradients $v=\frac{1}{\sample} \sum_j \Delta_j^t$ over different communication rounds $t$ in non-DP federated learning, following the setting in Section~\ref{sec:cnn}. Here we use the first convolutional layer (conv1.weight) and the first fully-connected layer (fc1.weight) as an example. 
    }
    \label{fig:fft_of_signal}
\end{figure}

In Section \ref{sec:ls} (supplementary material), 
we demonstrate an additional classification example where Laplacian smoothing reaches improved estimates of smooth signals (parameters) against Gaussian noise. And we also give an example to show the trade-off in Proposition~\ref{prop:ls-risk}. Among a variety of usages such as reducing the variance of SGD on-the-fly, escaping spurious minima, and improving generalization in training many machine learning models including neural networks \cite{Osher:2018, wang:2019},
the Laplacian smooting in this paper particularly improves the utility when Gaussian noise is injected to federated learning for privacy, that will be discussed below. 

\section{Differential Privacy Bounds}
\label{sec:theoretical-main}
We provide closed-form DP guarantees for differentially private federated learning, with or without LS, under both scenarios that active clients are sampled with uniform subsampling or with Poisson subsampling. 
Let us recall the definition of differential privacy and R\'enyi differential privacy (RDP).

\begin{definition}[$(\varepsilon$,$\delta)$-DP \cite{Dwork:2014book}]\label{dp-definition}
A randomized mechanism $\mathcal{M}:\mathcal{D} \rightarrow \mathbb{R}^d$ satisfies ($\varepsilon$,$\delta$)-DP if for any two adjacent datasets $D, D' \in \mathcal{D}$ differing by only one element, and any output subset $O \subseteq \mathbb{R}^d$, it holds that
$$\mathbb{P}[\mathcal{M}(D)\in O] \leq e^{\varepsilon} \cdot \mathbb{P}[\mathcal{M}(D')\in O]+\delta.$$
\end{definition}

\begin{definition}[$(\alpha, \rho)$-RDP \cite{mironov:2017renyi}]\label{RDP}
For $\alpha>1 \ \mbox{and}\ \rho>0$, a randomized mechanism $\cM:\mathcal{D}\rightarrow\cR^d$ satisfies $(\alpha, \rho)$-R\'enyi DP, i.e. $(\alpha, \rho)$-RDP, if for all adjacent datasets $D, D^\prime \in \mathcal{D}$ differing by one element, it has 
$$D_{\alpha}\big(\cM(D)||\cM(D^\prime)\big):=\frac{1}{\alpha-1}\log \mathbb{E}\big(\cM(D)/\cM(D^\prime)\big)^\alpha\leq \rho.$$
\end{definition}

\begin{lemma}[From $(\alpha, \rho)$-RDP to $(\varepsilon,\delta)$-DP \cite{mironov:2017renyi}]\label{lemma:RDP_to_DP}
	If a randomized mechanism $\cM: \mathcal{D}\rightarrow\cR^d$ satisfies $(\alpha,\rho)$-RDP, then $\cM$ satisfies $(\rho+\log(1/\delta)/(\alpha-1),\delta)$-DP for all $\delta\in(0,1)$.
\end{lemma}

In federated learning, we consider the user-level DP. So the terms \textbf{element} and \textbf{dataset} in the definition will refer to a single client, and a set of clients respectively in our scenario. There are two ways to construct a subset of active clients. The first one is uniform subsampling, i.e. in each communication round, a subset of fixed size $\sample= \tau\cdot \client$ of clients are sampled uniformly. The second one is Poisson subsampling, which includes each client in the subset with probability $\tau$ independently. If we trace back to the definition, this subtle difference actually comes from the difference of how we construct the adjacent datasets $D$ and $D'$. For uniform subsampling, $D$ and $D'$ are adjacent if and only if there exist two samples $a \in D$ and $b \in D'$ such that if we replace $a$ in $D$ with $b$, then $D$ is identical with $D'$ \cite{Dwork:2014book}. However, for Poisson subsampling, $D$ and $D'$ are said to be adjacent if $D\cup \{a\}$ or $D \backslash \{a\}$ is identical to $D'$ for some sample $a$ \cite{Mironov2019sampled, zhu2019poission}. This subtle difference results in two different parallel scenarios below.

\subsection{DP guarantee for uniform subsampling}
\begin{lemma}[RDP for Uniform Subsampling]\label{lemma:sub_uniform}
Gaussian mechanism $\mathcal{M}=f(D)+\mathcal{N}(0,\nu^2)$ applied on a subset of samples drawn uniformly without replacement with probability $\tau$ satisfies $(\alpha,3.5\tau^2\alpha/\nu^2)$-RDP given $\nu^2\geq \frac{2}{3}$ and $\alpha-1 \leq \frac{2}{3}\nu^2 \ln \big(1/\alpha\tau(1+\nu^2)\big)$, where the sensitivity of $f$ is 1.
\end{lemma}

\begin{remark}
    Comparing with the result $(\alpha, 5\tau^2\alpha/\nu^2)$ in \cite{wanglingxiao:2019}, and $(\alpha, 6\tau^2\alpha/\nu^2)$ in \cite{bun:2018}, Lemma~\ref{lemma:sub_uniform} provides a tighter bound while relaxing their requirement on $\nu^2$ that $\nu^2\geq 1.5$ and $\nu^2 \geq 5$ respectively. 
\end{remark}

\begin{theorem}[Differential Privacy Guarantee for DP-Fed-LS with Uniform Subsampling]\label{Theorem-privacy-guarantee-federated-uniform}
For any $\delta \in (0,1)$, $\varepsilon>0$, DP-Fed or DP-Fed-LS with uniform subsampling, satisfies ($\varepsilon$,$\delta$)-DP when the variance of the injected Gaussian noise $\mathcal{N}(0, \nu^2)$ satisfies
\begin{equation} \label{eq-nu-uniform}
    \nu \geq \frac{\tau \clip}{\varepsilon}\sqrt{\frac{14T}{\lambda}\bigg( \frac{\log(1/\delta)}{1-\lambda} + \varepsilon\bigg)},
\end{equation}
if there exists $\lambda \in (0,1)$ such that $\nu^2/4\clip^2 \geq \frac{2}{3}$ and $\alpha-1 \leq \frac{\nu^2}{6\clip^2}\log(1/(\tau \alpha (1+\nu^2/4\clip^2)))$, where $\alpha= \log(1/\delta)/((1-\lambda)\varepsilon) +1$.
\end{theorem}

\subsection{DP guarantee for Poisson subsampling}
\begin{lemma}[RDP for Poisson Subsampling]\label{lemma:sub_poisson}
Gaussian mechanism $\mathcal{M}=f(D)+\mathcal{N}(0,\nu^2)$ applied to a subset that includes each data point independently with probability $\tau$ satisfies $(\alpha,2\tau^2\alpha/\nu^2)$-RDP given $\nu^2\geq \frac{5}{9}$ and $\alpha-1 \leq \frac{2}{3}\nu^2\log \big(1/\alpha\tau(1+\nu^2)\big)$, where the sensitivity of $f$ is 1.
\end{lemma}

\begin{remark}
    The bound in Lemma~\ref{lemma:sub_poisson} matches the bound of $(\alpha, 2\alpha\tau^2/\nu^2)$-DP in \cite{Mironov2019sampled}. However, we relax the requirement that $\nu \geq 4$ used in \cite{Mironov2019sampled}, and simplify the multiple requirements over $\alpha$ that $1<\alpha\leq \frac{\nu^2C}{2}-2\ln \nu$ and $\alpha \leq \frac{\nu^2 C^2/2-\ln 5-2\ln\nu}{C+\ln(\tau \alpha)+1/(2\nu^2)}$,  where $C= \ln \big( 1+\frac{1}{\tau(\alpha-1)} \big)$, to only one requirement. This makes our closed-form privacy bound in Theorem~\ref{Theorem-privacy-guarantee-federated-poisson} below more concise and easier to implement. 
\end{remark}

\begin{theorem}[Differential Privacy Guarantee for DP-Fed-LS with Poisson Subsampling]\label{Theorem-privacy-guarantee-federated-poisson}
For any $\delta \in (0,1)$, $\varepsilon>0$, DP-Fed or DP-Fed-LS with Poisson subsampling, satisfies ($\varepsilon$,$\delta$)-DP when its injected Gaussian noise $\mathcal{N}(0, \nu^2)$ is chosen to be
\begin{equation} \label{eq-nu-poisson}
    \nu \geq \frac{\tau \clip}{\varepsilon}\sqrt{\frac{2T}{\lambda}\bigg( \frac{\log(1/\delta)}{1-\lambda} + \varepsilon\bigg)},
\end{equation}
if there exists $\lambda \in (0,1)$ such that $\nu^2/\clip^2 \geq \frac{5}{9}$ and $\alpha-1 \leq \frac{2\nu^2}{3\clip^2}\log(1/(\tau \alpha (1+\nu^2/\clip^2)))$, where $\alpha=\log(1/\delta)/((1-\lambda)\varepsilon) +1$.
\end{theorem}

Theorem~\ref{Theorem-privacy-guarantee-federated-uniform} and Theorem~\ref{Theorem-privacy-guarantee-federated-poisson} characterize the closed-form relationship between $(\varepsilon,\delta)$-DP and the corresponding noise level $\nu$, based on the numerical results in  \cite{wang2019subsampled, zhu2019poission, Mironov2019sampled}.
As we can see later, they will also serve as backbone theorems when we analyse the optimization error bounds of DP-Fed-LS. The two conditions in the above theorems (lemmas) are used for inequality scaling. 
In practical implementation, we will do a grid search of $\lambda \in (0,1)$ and select the one that gives the smallest lower bound of $\nu$ while satisfying both 
conditions. After that, we set $\nu$ to its lower bound.

Proofs of Theorem~\ref{Theorem-privacy-guarantee-federated-uniform} and Theorem~\ref{Theorem-privacy-guarantee-federated-poisson} are given in the Section~\ref{sec:proof-theorem-uniform} and \ref{sec:proof-theorem-poisson} in the appendix, while the proof of lemmas are given in Section~\ref{sec:proof-lemma-uniform} and Section~\ref{sec:proof-lemma-poisson} in the supplementary material. These closed-form bounds are of similar rates as the numerical moment accountant \cite{wang2019subsampled, zhu2019poission, Mironov2019sampled} up to a constant (see Section~\ref{sec:comparison} in the supplementary material).

\section{Convergence with Differential Privacy Guarantee}
Here, we provide convergence guarantees for DP-Fed-LS in Algorithm~\ref{algorithm:1} with uniform subsampling. 
We collect several commonly used assumptions.
\begin{assumption}[$(\heteG,\heteB)$-BGD (Bound Gradient Dissimilarity) \cite{karimireddy2020scaffold}]\label{assumption-BGD}
There exist constants $\heteG \geq 0$ and $\heteB \geq 1$ such that
\begin{equation*}
       \frac{1}{\client} \sum_{j=1}^\client \| \nabla f_j (w) \|_2^2 \leq \heteG^2 + \heteB^2 \| \nabla f(w) \|_2^2, \ \forall w. 
\end{equation*}
\end{assumption}

\begin{assumption}\label{assumption-smooth}
$f_1, ..., f_\client$ are all $\smooth$-smooth: for all $u$ and $v$, $\| \nabla f_j(u) - \nabla f_j(v) \|\leq \smooth \|u-v \|$.
\end{assumption}

\begin{assumption}\label{assumption-mu-strongly-convex}
$f_1$, ..., $f_\client$ are all $\mu$-strongly convex: 
\begin{equation*}
    f_j(u) \geq f_j(v) + \langle u-v, \nabla f_j(v) \rangle + \frac{\mu}{2}\| u-v \|_2^2, \quad \mbox{for all} \ u,v.
\end{equation*}
\end{assumption}

\begin{assumption}\label{assumption-convex}
$f_1$, ..., $f_\client$ are all convex:
\begin{equation*}
    f_j(u) \geq f_j(v) + \langle u-v, \nabla f_j(v) \rangle, \quad \mbox{for all} \  u, v.
\end{equation*}
\end{assumption}

\begin{assumption}\label{assumption-gradient-variance}
Let $g_j(w)$ be a stochastic mini-batch gradient of client $j$. The variance of $g_j(w)$ under the vector norm given by $\|v \|_{\textbf{A}_\sigma^{-1}}^2= \langle v, \textbf{A}_\sigma^{-1} v \rangle$ in each device is bounded:
\begin{equation*}
    \mathbb{E}\| g_j(w) - \nabla f_j(w) \|_{\textbf{A}_\sigma^{-1}}^2 \leq \varsigma_j^2(\sigma) \quad \mbox{for all} \ j\in[\client].
\end{equation*}
We denote $\varsigma^2(\sigma)=\frac{1}{\client} \sum_{j=1}^\client \varsigma_j^2(\sigma)$.
\end{assumption}
Here for $\sigma=0$, it reduces to the common assumption in federated learning \cite{karimireddy2020scaffold}; for $\sigma>0$, variance could be significantly reduced as the discussions in Section \ref{sec:method}.

\begin{assumption}\label{assumption-G-lipschitz}
$f_1$, ..., $f_\client$ are all $\lip$-Lipschitz:
$
\|f_j(u) -f_j(v)\|_2 \leq \lip\|u-v \|_2 \quad \mbox{for all} \ u, v.
$
\end{assumption}

For simplicity, 
we use $\nu_\clip$ to represent $\nu$ in Theorem~\ref{Theorem-privacy-guarantee-federated-uniform} 
as a linear 
function of the clipping parameter $\clip$. Then we have $\nu_\clip = \clip\nu_1$.
We use $\tilde{\mathcal{O}}$ to denote asymptotic growth rate up to a logarithmic factor (including $\log \step$, $\log \sample$), while $\mathcal{O}$ up to a constant.


\begin{theorem}[Convergence Guarantees for DP-Fed-LS]\label{thm-convergenve}
Assuming the conditions in Theorem~\ref{Theorem-privacy-guarantee-federated-uniform} hold, with $\log(1/\delta)\geq\varepsilon$ and a proper constant step size $\eta_l$. Let $\clip=\eta_l \step \lip$, $\eta_g\geq \sqrt{\sample}$, and communication round $T=\frac{\varepsilon^2 \client^2}{C_0 \lip^2 \sample \log(1/\delta)}$, then DP-Fed-LS with uniform subsampling satisfies $(\varepsilon,\delta)$-DP and the following error bounds. 
\begin{itemize}
    \item \textbf{$\mu$ Strongly-Convex:} Under Assumption~\ref{assumption-BGD}, \ref{assumption-smooth}, \ref{assumption-mu-strongly-convex}, \ref{assumption-gradient-variance}, \ref{assumption-G-lipschitz}, it holds that
        \begin{equation*}
            \hspace*{-0.25cm}
                \Er(\bar{w}^T) = \mathbb{E}[f(\bar{w}^T)] - f(w^*) \leq \tilde{\mathcal{O}} \bigg(\frac{ (\frac{\varsigma^2(\sigma)}{\step}+ (1-\tau)\heteG^2 + d_\sigma) \lip^2 \log(1/\delta)}{\mu_\sigma \varepsilon^2 \client^2} \bigg).
        \end{equation*}
    \item \textbf{General-Convex:} Under Assumption~\ref{assumption-BGD}, \ref{assumption-smooth}, \ref{assumption-gradient-variance}, \ref{assumption-G-lipschitz}, it holds that
        \begin{equation*}
            \hspace*{-0.25cm} \Er(\bar{w}^T) = \mathbb{E}[f(\bar{w}^T)] - f(w^*) \leq \mathcal{O} \bigg( \frac{\sqrt{(\frac{\varsigma^2(\sigma)}{\step}+ (1-\tau)\heteG^2 + d_\sigma)D_\sigma \lip^2 \log(1/\delta)}} {\varepsilon \client} \bigg).
        \end{equation*}
    \item \textbf{Non-Convex:} Under Assumption~\ref{assumption-BGD}, \ref{assumption-smooth}, \ref{assumption-gradient-variance}, \ref{assumption-G-lipschitz}, it holds that 
    \begin{equation*}
        \hspace*{-0.25cm}
        \Er(\bar{w}^T)= \| \nabla f(\bar{w}) \|_{\textbf{A}_\sigma^{-1}}^2
        \leq
        \mathcal{O} \bigg( \frac{\sqrt{(\frac{\varsigma^2(\sigma)}{\step} + (1-\tau) \heteG^2 + \tilde{d}_\sigma)F_0  \smooth \lip^2 \log(1/\delta)}} {\varepsilon \client} \bigg).
    \end{equation*}
\end{itemize}
where $\mu_\sigma = \mu \Lambda_{\min}\geq \mu/(1+4\sigma)$, and the effective dimension $d_\sigma=\sum_{i=1}^d \Lambda_i$, $\tilde{d}_\sigma=\sum_{i=1}^d \Lambda_i^2$. Here $\Lambda_i = \frac{1}{1+2\sigma(1-\cos(2\pi i/d))}$ is the eigenvalue of $\textbf{A}_\sigma^{-1}$ while $\Lambda_{\min}$ is the smallest one. And for an optimum $w^*$, $D_\sigma=\|w^0-w^* \|_{\textbf{A}_\sigma}^2$, $F_0 = f(w^0) - f(w^*)$ and $C_0=\frac{14}{\lambda} (1+\frac{1}{1-\lambda})$.
\end{theorem}

The proof sketch of Theorem~\ref{thm-convergenve} can be found in Section~\ref{sec:full-statement-convergence} in Appendix. In this theorem, dominant errors are introduced by the variance of stochastic gradients ($\varsigma^2(\sigma)/\step$), heterogeneity of \textbf{Non-IID} data ($(1-\tau)\heteG^2$) and DP ($d_\sigma$ and $ \tilde{d}_\sigma$), in comparison to the initial error. Among the three dominant errors, the variance term $\varsigma^2(\sigma)/\step$ will diminish while the number of local iteration $\step$ grows large enough ($\step \gg 1$). What's more, the heterogeneity term $(1-\tau)\heteG^2$ will be reduced if subsampling ratio $\tau$ is high. Particularly, in \textbf{IID} ($\heteG=0$) or full-device participation $(\tau=1)$ setting, this term will vanish. Therefore the error term introduced by DP, of effective dimensionality $d_\sigma$ or $\tilde{d}_\sigma$, dominates the variance and heterogeneity terms in these scenarios, whose rates in Theorem~\ref{thm-convergenve} matches the optimal ones of ERM via SGD with differential privacy in centralized setting \cite{wang:2019, wang2017differentially}, as shown in the upper part of Table~\ref{tbl:rate-comparision}. In Table~\ref{tbl:rate-comparision}, the term $\log(\sample)$ of DP-Fed-LS comes from the numerator of learning rate $\tilde{\eta}$ in Theorem~\ref{thm-strongly-convex} in Appendix, implicitly involved in $\tilde{\mathcal{O}}$.

In particular when $\sigma=0$, the bounds above reduce to the standard DP-Fed setting. The benefit of introducing Laplacian smoothing ($\sigma > 0$) lies in the reduction of $\varsigma^2(\sigma)$ and the effective dimension $d_\sigma, \tilde{d}_\sigma \leq d_0=d$, although it might lose some curvature $\mu_\sigma \geq \mu/(1+4\sigma)$ in the strongly-convex case. 

The following corollary provides the communication complexity of DP-Fed-LS in Algorithm~\ref{algorithm:1} with uniform subsampling, with tight bounds on the number of communications $T$ to reach an optimization error $\epsilon$. It is derived from Theorem~\ref{thm-strongly-convex}, \ref{thm-convex} and \ref{thm-non-convex} in Appendix.

\begin{corollary}[Communication Complexity]\label{corollary-communication-round}
Assuming the same conditions in Theorem~\ref{thm-convergenve}, the communication complexity of DP-Fed-LS with uniform subsampling and fixed noise level $\nu_\clip=\clip \nu_1$ independent to $T$ satisfies the following rates to reach an $\epsilon$-optimality gap.

\begin{itemize}
    \item \textbf{$\mu$ Strongly-Convex:} 
        \begin{equation*}
            T = \tilde{\mathcal{O}} \bigg( \frac{(1+4\sigma)\smooth \heteB^2}{\mu_\sigma } + \frac{\varsigma^2(\sigma)}{\mu_\sigma \step \sample \epsilon} + \frac{d_\sigma \lip^2 \nu_1^2}{ \mu_\sigma \sample^2 \epsilon}  + \frac{\sqrt{\smooth} \heteG}{\mu_\sigma \sqrt{\epsilon}} + (1-\tau) \frac{\heteG^2}{\mu_\sigma \epsilon \sample} \bigg).
        \end{equation*}
    \item \textbf{General-Convex:} 
    \begin{equation*}
        T = \mathcal{O}\bigg( \frac{(1+4\sigma) \smooth \heteB^2 D_\sigma}{\epsilon} + \frac{\varsigma^2(\sigma) D_\sigma}{\step \sample \epsilon^2} + \frac{d_\sigma D_\sigma \lip^2 \nu_1^2  }{\sample^2 \epsilon^2} + \frac{\sqrt{\smooth} D_\sigma \heteG}{\epsilon^{3/2}} + (1-\tau) \frac{D_\sigma \heteG^2}{\epsilon^2 \sample} \bigg).
    \end{equation*}
    \item \textbf{Non-Convex:} 
    \begin{equation*}
        T = \mathcal{O} \bigg( \frac{(1+4\sigma)\smooth \heteB^2 F_0}{\epsilon} + \frac{\varsigma^2(\sigma) \smooth F_0}{\step \sample \epsilon^2} + \frac{\tilde{d}_\sigma F_0 \lip^2 \nu_1^2 \smooth }{\sample^2 \epsilon^2} + \frac{\smooth F_0 \heteG}{\epsilon^{3/2}} + (1-\tau)\frac{F_0 \smooth \heteG^2}{\epsilon^2 \sample} \bigg).
    \end{equation*}
\end{itemize}
\end{corollary}

\begin{remark}
In Corollary \ref{corollary-communication-round}, we regard that $\nu_1$ is a given constant independent to the communication round $T$, such that $\nu_{\clip}=\clip \nu_1$ and $\clip=\eta_l \step \lip$. In this case, if $\nu_1\geq8/3$ and $\alpha - 1\leq \frac{\nu_1^2}{6} \ln \frac{1}{\tau\alpha(1+\nu_1^2/4)}$, then $(\varepsilon,\delta)$-DP satisfying Eq~(\ref{eq-T}) can be achieved for any $\lambda\in(0,1)$.
\begin{equation}\label{eq-T}
    T \leq \frac{\lambda \varepsilon^2 \nu_1^2}{14 \tau^2 \big(\frac{\log(1/\delta)}{1-\lambda} +\varepsilon \big)}.
\end{equation}
More details can be found in Section~\ref{sec:table_compare} in supplementary.
Compared with the best known rates in federated average without DP \cite{karimireddy2020scaffold}, the communication complexity in Corollary~\ref{corollary-communication-round} involves an extra term for the injected noise $\nu_1$ in DP, while other terms match the best known rates, which are tighter than others in literature \cite{yu2019parallel,khaled2020tighter,li:2019} with the same $(\heteG,\heteB)$-BGD assumption, as shown in the lower part of Table~\ref{tbl:rate-comparision}.
\end{remark}

\section{Experimental Results}\label{sec:experiment}
We evaluate our proposed DP-Fed-LS on three benchmark classification tasks. 
We compare the utility of DP-Fed-LS ($\sigma>0$) and plain DP-Fed ($\sigma=0$) with varying $\varepsilon$ in $(\varepsilon, \delta)$-DP. Here we set $\delta=1/\client^{1.1}$ as \cite{mcmahan:2017}.
These three tasks include training a DP federated logistic regression on the MNIST dataset \cite{lecun:1998}, a convolution neural network (CNN) on the SVHN dataset \cite{netzer:2011} and a long short-term memory (LSTM) model over the Shakespeare dataset \cite{caldas2018leaf, mcmahan:2016}. 
Details about datasets and tasks will be discussed later. For logistic regression, we apply the privacy budget in Theorem~\ref{Theorem-privacy-guarantee-federated-uniform} and \ref{Theorem-privacy-guarantee-federated-poisson}. For CNN and LSTM, we apply the moment accountants in \cite{wang2019subsampled}\footnote{\url{https://github.com/yuxiangw/autodp}} and \cite{Mironov2019sampled}\footnote{\url{https://github.com/tensorflow/privacy/tree/master/tensorflow_privacy/privacy/analysis}} for uniform subsampling and Poisson subsampling, respectively. For moment accountants, we should provide a noise multiplier $z$ to control the noise level. Then we can compute the privacy budget with given communication round and subsampling ratio.

We report the average loss and average accuracy based on three independent runs. Hardwares we used for these experiments are NVIDIA GeForce GTX 1080Ti GPU (11G RAM) and Intel Xeon E5-2640 CPU.

\textit{\textcolor{header1}{Hyper-parameter Tuning.}} To comply with traditional neural network training, we replace local iteration step $\step$ with local epoch $\epoch$ and denote the batch size as $\batch$. For all the tasks, we tune the hyper-parameters such that DP-Fed achieves the best validation accuracy, and then apply the same settings to DP-Fed-LS. For example, the clipping parameter $\clip$ is involved since a large one will induce too much noise while a small one will 
deteriorate training. We will first start from a small value
and then increase it until the validation accuracy for DP-Fed no longer improves. Other parameters, including local learning rate $\eta_l$, local batch size $\batch$, local epoch $\epoch$ are borrowed from the literature \cite{caldas2018leaf, papernot:2018, mcmahan:2016}.
We fixed the global learning rate $\eta_g=1$. 

\subsection{Logistic regression with IID MNIST dataset}\label{sec:logistic}
We train a differentially private federated logistic regression on the MNIST dataset \cite{lecun:1998}. MNIST is a dataset of 28$\times$28 grayscale images of digit from 0 to 9, containing 60K training samples and 10K testing samples. We split 50K training samples into 1000 clients each containing 50 samples in an \textbf{IID} fashion \cite{mcmahan:2016} for uniform subsampling. For Poisson subsampling, we further lower the number of clients to 500 each containing 100 samples. The remaining 10K training samples are left for validation. We set the batch size $\batch=10$, local epoch $E=5$ \cite{mcmahan:2016}, sensitivity $\clip=0.3$ (by tuning described above), number of communication rounds $T=30$, activate client fraction $\tau=0.05$ and weight decay $\lambda_0=4e-5$. We use an initial local learning rate $\eta_l=0.1$ and decay it by a factor of $\gamma=0.99$ each communication round.

\textit{\textcolor{header1}{Improved test accuracy under the same privacy budget}}. From Table~\ref{tbl:mnist_acc}, we notice that DP-Fed-LS outperforms DP-Fed in almost all settings. In particular, when $\varepsilon$ is small, the improvement of DP-Fed-LS is remarkably large. We show the training curves in Figure~\ref{fig:mnist_loss_acc}, where we find that DP-Fed-LS converges slower than DP-Fed in both subsampling scenarios. However, DP-Fed-LS will generalize better than DP-Fed at the later stage of training. Other training curves are deferred to the Section~\ref{sec:curves} in supplementary material.

\begin{table}[htbp]
\centering
\caption{Test accuracy of logistic regression on MNIST with DP-Fed ($\sigma=0$) and DP-Fed-LS ($\sigma=1, 2, 3$) under different $(\varepsilon,1/\client^{1.1})$-DP guarantees and subsampling methods.}
\begin{tabular}{ccccc|cccc}
\hline
\multicolumn{5}{c}{Uniform Subsampling} & \multicolumn{4}{c}{Poisson Subsampling} \\
\hline
$\varepsilon$ & 6 & 7 & 8 & 9 & 6  & 7 & 8 & 9\\
$\sigma=0.0$ & 78.41 & 81.85 & 83.24 & 84.62 & 80.03 & 82.05 & 83.33 & 84.52 \\
$\sigma=1.0$ & 82.44 & \textbf{85.12} & 85.22 & 84.69 & 82.34 & \textbf{84.85} & \textbf{84.65} & 84.49\\
$\sigma=2.0$ & 83.33 & 84.65 & \textbf{85.31} & 85.27 & \textbf{83.43} & \textbf{84.85} & 84.39 & \textbf{85.87}\\
$\sigma=3.0$ &\textbf{83.60} & 83.53 & 85.18 & \textbf{85.35} & 82.94 & 84.29 & 84.16 & 84.79\\
\hline
\end{tabular}
\label{tbl:mnist_acc}
\end{table}

\begin{figure}[ht]
    \centering
    \subfigure[]{\includegraphics[width=1.5in]{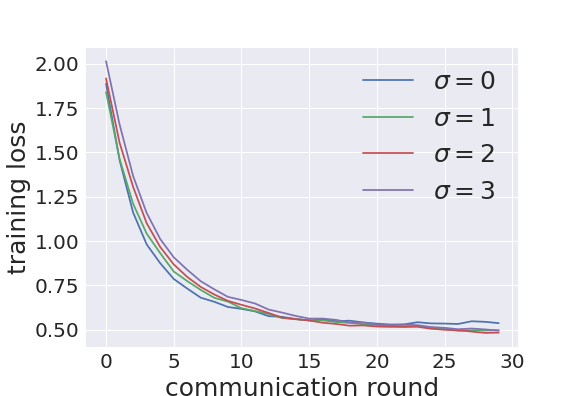}}
    \subfigure[]{\includegraphics[width=1.5in]{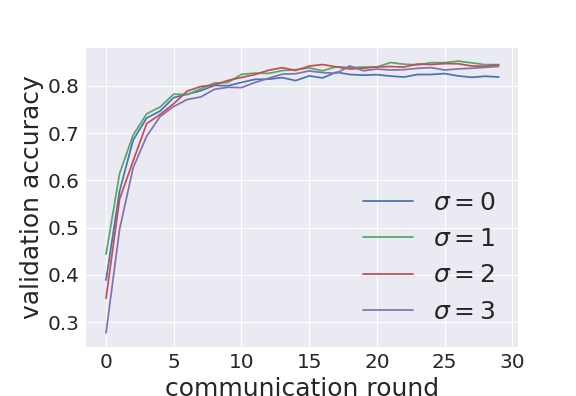}}
    \subfigure[]{\includegraphics[width=1.5in]{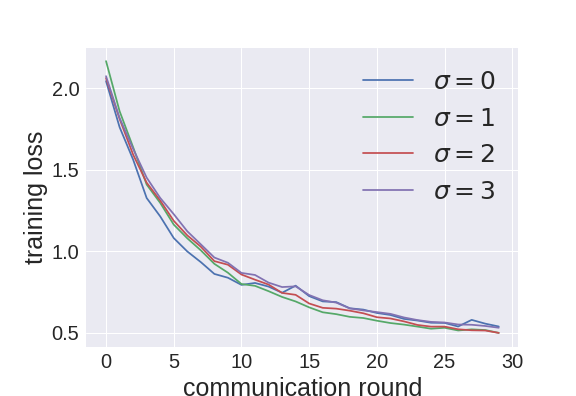}}
    \subfigure[]{\includegraphics[width=1.5in]{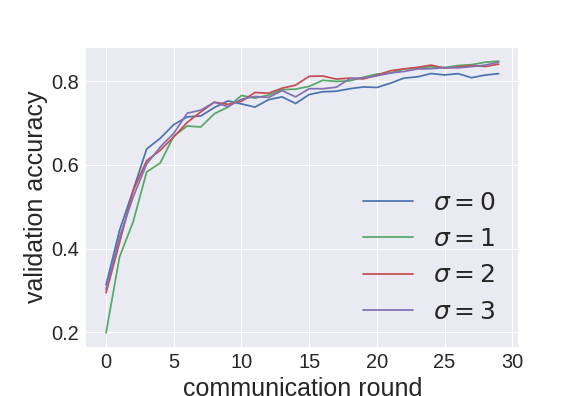}}
    \vspace{-0.7em}
    \caption{Training curves of logistic regression on MNIST with DP-Fed ($\sigma=0$), DP-Fed-LS ($\sigma=1, 2, 3$). (a) and (b): training loss and validation accuracy with uniform subsampling and $(7,1/1000^{1.1)}$-DP.  (c) and (d): training loss and validation accuracy with Poisson subsampling and $(7,1/500^{1.1})$-DP.}
    \label{fig:mnist_loss_acc}
\end{figure}

\subsection{Convolutional Neural Network with IID SVHN dataset}\label{sec:cnn}
In this section, we train a differentially private federated CNN on the extended SVHN dataset \cite{netzer:2011}. SVHN is a dataset of 32$\times$32 colored images of digits from 0 to 9, containing 73,257 training samples and 26,032 testing samples. We enlarge the training set with another 531,131 extended samples and split them into 2,000 clients each containing about 300 samples in an \textbf{IID} fashion \cite{mcmahan:2016}. We also split the testing set by 10K/16K for validation and testing. Our CNN stacks two $5\times 5$ convolutional layers with max-pooling, two fully-connected layers with 384 and 192 units, respectively, and a final softmax output layer (about 3.4M parameters in total) \cite{papernot:2017}. We pretrain the model over the MNIST dataset to speed up the training. 
We set $\batch=50$, local epoch $\epoch=10$ \cite{mcmahan:2016}, sensitivity $\clip=0.7$ (by tuning described above), number of 
communication rounds $T=200$, active client fraction $\tau=0.05$ and weight decay $\lambda_0=4e-5$. Initial local learning rate $\eta_l=0.1$ and will decay by a factor of $\gamma=0.99$ each communication round. We vary the privacy budget by setting the noise multiplier $z=1.5, 1.3, 1.1, 1.0$.

\begin{table}[htbp]
\centering
\caption{Test accuracy of CNN on SVHN with DP-Fed ($\sigma=0$) and DP-Fed-LS ($\sigma=0,5, 1, 1.5$) under different $(\varepsilon,1/2000^{1.1})$-DP guarantees and subsampling methods.}
\begin{tabular}{ccccc|cccc}
\hline
\multicolumn{5}{c}{Uniform Subsampling} &\multicolumn{4}{c}{Poisson Subsampling} \\
\hline
$\varepsilon$   & 5.23 & 6.34 & 7.84 & 8.66 & 2.56 & 3.19 & 4.24 & 5.07\\
$\sigma=0.0$ & 81.40          & 82.46          & 85.18           & 85.84 & 82.29          & 83.82          & 85.53          & 86.56 \\
$\sigma=0.5$ & 82.72        & \textbf{84.65} & \textbf{86.49} & 86.32 & 84.27          & \textbf{85.47} & \textbf{87.00} & \textbf{87.50}\\
$\sigma=1.0$ & \textbf{82.39} & 84.13     & 85.88     & \textbf{86.39} & \textbf{84.65} & 85.38          & 86.37          & 87.26\\
$\sigma=1.5$ & 82.19          & 83.97         & 86.03      & 85.66& 84.23          & 85.12          & 86.58          & 87.35\\
\hline
\end{tabular}
\label{tbl:svhn_acc}
\end{table}

\begin{figure}[htbp]
    \centering
    \subfigure[]{\includegraphics[width=1.5in]{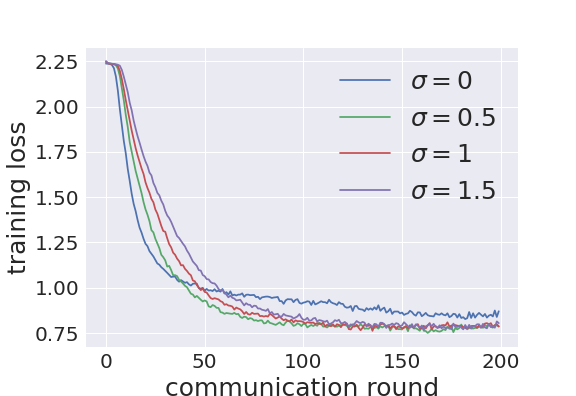}}
    \subfigure[]{\includegraphics[width=1.5in]{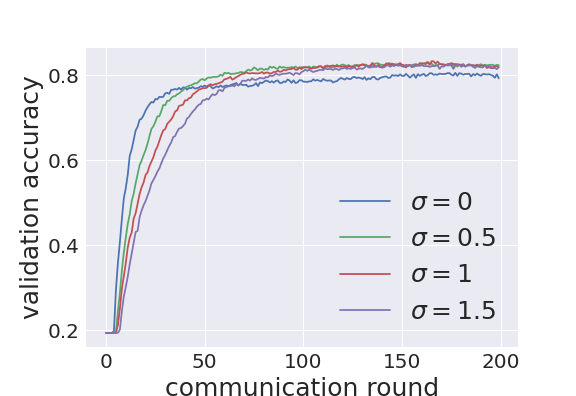}}
    \subfigure[]{\includegraphics[width=1.5in]{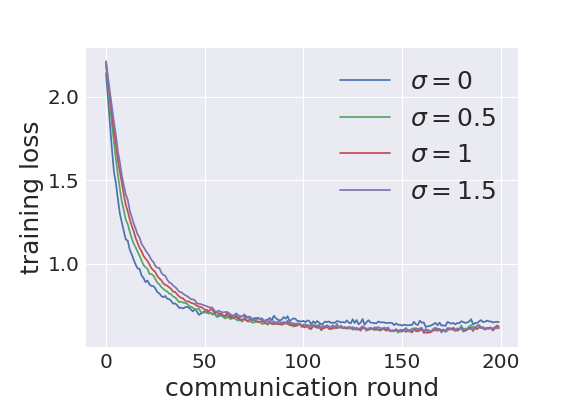}}
    \subfigure[]{\includegraphics[width=1.5in]{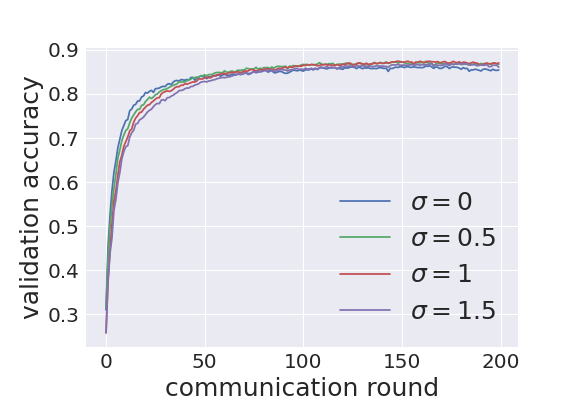}}
    
    \vspace{-0.8em}
    
    \caption{Training curves of CNN on SVHN with DP-Fed ($\sigma=0$), DP-Fed-LS ($\sigma=0.5, 1, 1.5$). (a) and (b): training loss and validation accuracy with uniform subsampling and $(5.23, 1/2000^{1.1})$-DP. (c) and (d): training loss and validation accuracy with Poisson subsampling and $(5.07, 1/2000^{1.1})$-DP.}
    \label{fig:svhn_loss_acc}
\end{figure}

\textit{\textcolor{header1}{Improved test accuracy under the same privacy budget}}. Table~\ref{tbl:svhn_acc}
shows that DP-Fed-LS yields higher accuracy than DP-Fed with both subsampling mechanisms and different DP guarantees.  We show training curves in Figure \ref{fig:svhn_loss_acc}, which are similar to the ones of the logistic regression. Again, training curves of DP-Fed-LS converge slower than that of DP-Fed, especially when uniform subsampling is used. However, DP-Fed-LS generalizes better than DP-Fed at the later stage.

In Table~\ref{tbl:svhn_acc_th_js}, we show the results of another two adaptive denoising estimators: the James-Stein estimator (JS) and the soft-thresholding estimator (TH), which have been shown to be useful for high dimensional parameter estimation and number release \cite{balle2018improving}, comparing with other denoising estimators \cite{barak2007privacy, hay2009accurate, williams2010probabilistic,bernstein2017differentially}. As mentioned in \cite{balle2018improving}, thanks to the fact that we know the parameter $\nu$ exactly, both JS and TH estimators are completely free of tuning parameters. However, as we can see in Table~\ref{tbl:svhn_acc_th_js}, neither of these two estimators performs well in our scenario, compared with Laplacian smoothing in Table~\ref{tbl:svhn_acc}, indicating that high dimensional sparsity assumption \cite{balle2018improving} does not hold here on \textit{federated average of gradients}.

\begin{table}[htbp]
\centering
\caption{Test accuracy of CNN on SVHN James-Stein (JS) and soft-threholding (TH) estimators \cite{balle2018improving} under different $(\varepsilon,1/2000^{1.1})$-DP guarantees and subsampling methods (the same settings as the ones we used in Table~\ref{tbl:svhn_acc}).}
\begin{tabular}{ccccc|cccc}
\hline
\multicolumn{5}{c}{Uniform Subsampling} &\multicolumn{4}{c}{Poisson Subsampling} \\
\hline
$\varepsilon$   & 5.23 & 6.34 & 7.84 & 8.66 & 2.56 & 3.19 & 4.24 & 5.07\\
JS & 52.12 & 52.41 & 59.73 & 61.77 & 56.23 & 55.60 & 60.43 & 60.64 \\
TH & 20.05 & 18.52 & 21.72 & 25.15 & 22.12 & 16.96 & 23.24 & 26.89\\
\hline
\end{tabular}
\label{tbl:svhn_acc_th_js}
\end{table}

\textit{\textcolor{header1}{Stability under large noise, learning rate and different orders of parameter flattening.}} In Figure~\ref{fig:extreme}, we show the training curves where relatively large noise multipliers $z$ are applied with Poisson subsampling and different local learning rates $\eta_l$. Here our CNNs are trained for one run from scratch. When the noise levels are large, the training curves fluctuate a lot. 
In these extreme cases, DP-Fed-LS outperforms DP-Fed by a large margin. For example, when $z=3$ and $\eta_l=0.05$, validation accuracy of DP-Fed starts to drop at the 150th epoch while DP-Fed-LS can still converge. When the learning rate increase to $0.125$, validation accuracy of DP-Fed drops below 0.2 after the 25th epoch while DP-Fed-LS approaches 0.7 at the end. Overall speaking, DP-Fed-LS is more stable against large noise levels and the change of local learning rate than DP-Fed. What's more, from Table~\ref{tab:order}, we notice that DP-Fed-LS is insensitive to the order of parameter flattening and consistently performs better than DP-Fed.

\begin{table}[htbp]
\centering
\caption{Test accuracy of CNN on SVHN with DP-Fed-LS ($\sigma=0.5, 1, 1.5$), with different unfolding ordering of convolutional kernel, under $(4.24,1/2000^{1.1})$-DP guarantees along with Poisson subsampling, and $(7.838,1/2000^{1.1})$-DP guarantees along with Uniform subsampling. ``B", ``C", ``W", and ``H" represent for batch, channel, width, and height. BCWH is the one reported in the paper. The accuracy for pure DP-Fed is 85.18 and 85.53 respectively.}
\begin{tabular}{ccccc|cccc}
\hline
\multicolumn{5}{c}{Uniform Subsampling} & \multicolumn{4}{c}{Poisson Subsampling}\\
\hline
Order   & BCWH & BCHW & BWHC & BHWC & BCWH & BCHW & BWHC & BHWC \\
$\sigma=0.5$ & 86.49 & 87.45  & 86.79  & 86.93 & 87.00 & 86.77 & 86.43 & 87.48\\
$\sigma=1.0$ & 85.88 & 87.41  & 86.97  & 87.42 & 86.37 & 86.15 & 87.21 & 87.24\\
$\sigma=1.5$ & 86.03 & 87.10  & 87.11  & 86.34 & 86.58 & 86.23 & 86.83 & 87.04\\
\hline
\end{tabular}

\label{tab:order}
\end{table}

\begin{figure*}[htbp]
    \centering
    \subfigure[]{\includegraphics[width=1.8in]{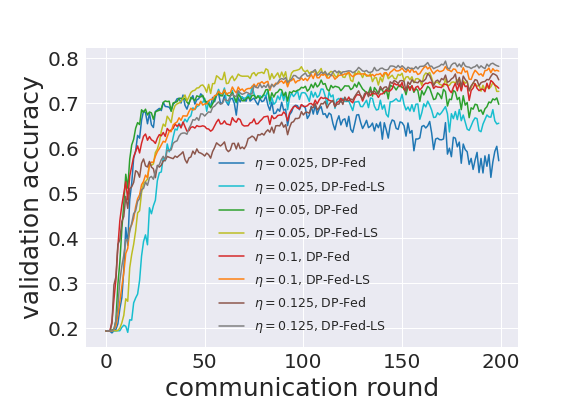}}
    \subfigure[]{\includegraphics[width=1.8in]{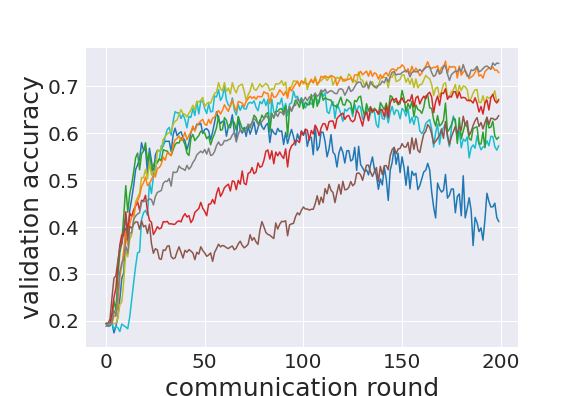}}
    \subfigure[]{\includegraphics[width=1.8in]{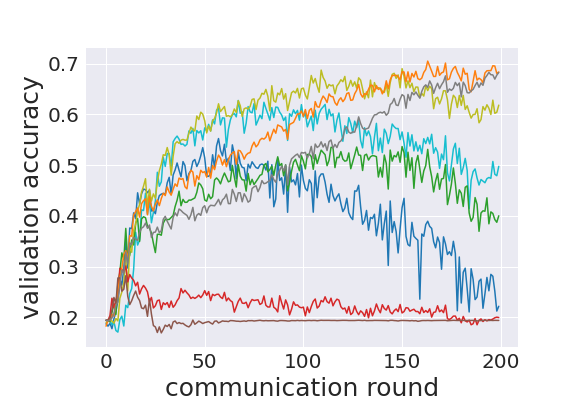}}
    \caption{Training curves of CNN on SVHN where large noise levels are applied, with Poisson subsampling and different local learning rates $\eta_l$. From left tor right, noise multiplier $z=2$, $2.5$ and $3$. For DP-Fed-LS, we set $\sigma=1$. 
    }
    \label{fig:extreme}
\end{figure*}

\subsection{Long Short-Term Memory with Non-IID Shakespeare dataset}\label{sec:lstm}
Here, we train a differentially private LSTM on the Shakespeare dataset \cite{caldas2018leaf, mcmahan:2016}, which is built from all works of William Shakespeare, where each speaking role is considered as a client, whose local database consists of all her/his lines. This is a \textbf{Non-IID} setting. The full dataset contains 1,129 clients and 4,226,158 samples. Each sample consists of 80 successive characters and the task is to predict the next character \cite{caldas2018leaf, mcmahan:2016}. In our setting, we remove the clients that own less than 100 samples to stabilize training, which reduces the total client number to 975. We split the training, validation, and testing set chronologically \cite{caldas2018leaf, mcmahan:2016}, with fractions of 0.7, 0.1, 0.2. Our LSTM first embeds each input character into a 8-dimensional space, after which two LSTM layers are stacked, each have 256 nodes. The outputs will be then fed into a linear layer, of which the number of output nodes equals the number of distinct characters \cite{caldas2018leaf, mcmahan:2016}. In this experiment, we set $\batch=50$, $\epoch=5$ \cite{mcmahan:2016}, $\clip=5$ (by tuning described above), $T=100$, $\tau=0.2$, and $\lambda_0=4e-5$. Initial local learning rate $\eta_l=1.47$ \cite{mcmahan:2016} and will decay by a factor of $\gamma=0.99$ each communication round. We vary the DP budget by setting the noise multiplier $z=1.6,1.4,1.2,1.0$.

\textit{\textcolor{header1}{Improved test accuracy under the same privacy budget}}. The test accuracy in Table~\ref{tbl:lstm_acc} are comparable to the one in \cite{caldas2018leaf}. We can also conclude that DP-Fed-LS provides better utility than DP-Fed. The training curves are plotted in Figure~\ref{fig:lstm_loss_acc}. Generally speaking, the training curves in \textbf{Non-IID} setting suffer from larger fluctuation than the ones in \textbf{IID} setting above. And the curves of DP-Fed-LS are smoother than DP-Fed, which further shows the potential of DP-Fed-LS in real-world applications.

\begin{table}[htbp]
\centering
\caption{Test accuracy of LSTM on Shakespeare with DP-Fed ($\sigma=0$) and DP-Fed-LS ($\sigma=0,5, 1, 1.5$) under different $(\varepsilon,1/975^{1.1})$-DP guarantees and subsamplings.}
\begin{tabular}{ccccc|cccc}
\hline
\multicolumn{5}{c}{Uniform Subsampling} & \multicolumn{4}{c}{Poisson Subsampling} \\
\hline
$\varepsilon$   & 14.94 & 17.69 & 22.43 & 27.24 & 6.78 & 8.22 & 10.41&14.04\\
$\sigma=0.0$ & 38.22 & 38.47 & 39.96 & 41.87& 38.81 & 39.42 & 40.19 & 41.55 \\
$\sigma=0.5$ & 39.14 & 40.27 & 41.95 & 43.76 & 39.07 & 40.02 & 42.02 & 43.59\\
$\sigma=1.0$ & 39.18 & \textbf{40.94} & \textbf{42.60} & 43.90 & \textbf{39.45} & \textbf{41.07} & 42.09 & \textbf{43.78}\\
$\sigma=1.5$ & \textbf{40.16} & 40.89 & 42.50 & \textbf{43.95} & 39.38 & 40.99 & \textbf{42.19} & 43.67\\
\hline
\end{tabular}
\label{tbl:lstm_acc}
\end{table}

\begin{figure}[htbp]
    \centering
    \subfigure[]{\includegraphics[width=1.5in]{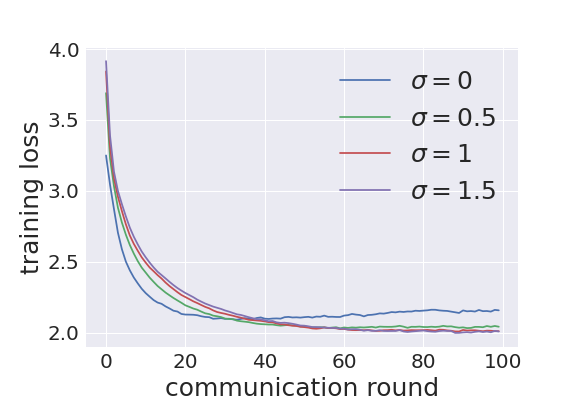}}
    \subfigure[]{\includegraphics[width=1.5in]{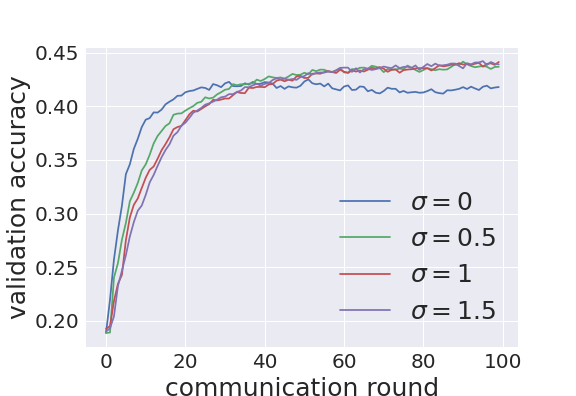}}
    \subfigure[]{\includegraphics[width=1.5in]{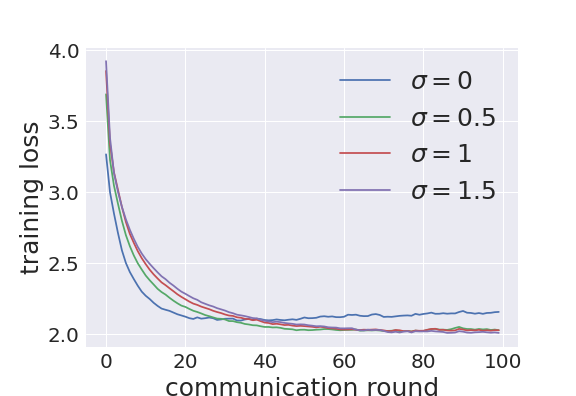}}
    \subfigure[]{\includegraphics[width=1.5in]{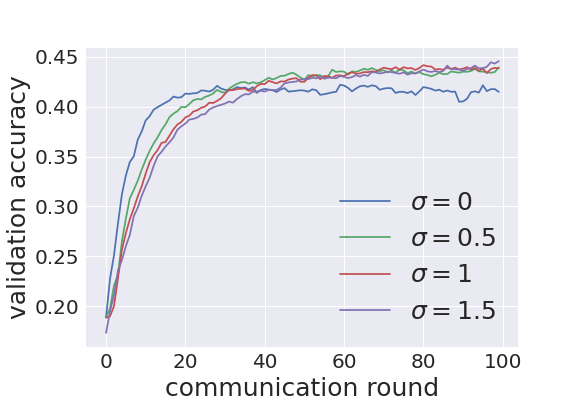}}
    \caption{Training curves of LSTM on Shakespeare dataset with DP-Fed $(\sigma=0)$, DP-Fed-LS $(\sigma=0.5, 1, 1.5)$. (a) and (b): training loss and validation accuracy with uniform subsampling and $(27.24, 1/975^{1.1})$-DP. (c) and (d): training loss and validation accuracy with Poisson subsampling, and $(14.04, 1/975^{1.1})$-DP.}
    \label{fig:lstm_loss_acc}
\end{figure}

\subsection{Membership Inference Attack}\label{sec:mi}
\textit{Membership privacy} is a simple yet quite practical notion of privacy \cite{Shokri:2017, yeom:2018, sablayrolles:2019}. Given a model $\theta$ and sample $z$, \textit{membership inference attack} is to infer the probability that a sample $z$ belongs to the training dataset \cite{sablayrolles:2019}. Specifically, a test set $\mathcal{T}=\{(z_i,m_i)\}$, is constructed with samples from both training data $(m_i=1)$ and hold-out data $(m_i=0)$, where the prior probability $P(m=1|\mathcal{T})=\rho_\mathcal{T}$. Then a successful membership attack can increase the excess probability using knowledge of model $\theta$, $P(m(z)=1|\theta, z\in \mathcal{T})-\rho_\mathcal{T}$.
In \cite{sablayrolles:2019}, Sablayrolles et al. define ($\varepsilon,\delta$)-membership privacy, and show that
$(\varepsilon, \delta)$-membership privacy can guarantee an upper bound $P(m(z)=1|\theta,z\in \mathcal{T})-\rho_\mathcal{T} \leq \frac{c\varepsilon}{4}+\delta$.

To evaluate the membership information leakage of models, 
\textit{threshold attack} \cite{yeom:2018} is adopted in our experiment. It is widely used as a metric to evaluate membership privacy \cite{jayaraman2019evaluating,wu:2020,yeom:2018}. It bases on the intuition that a sample with relatively small loss is more likely to belong to the training set, due to the more or less overfitting of ML models. Specifically, the test set $\mathcal{T}$ consists of both training and hold-out data of equal size (thus $\rho_\mathcal{T}=0.5$). Given a sample $z=(x,y)$ and a model $M_w(x)$, we calculate the loss $\ell(y, M_w(x))$. Then we select a threshold $t$: if $\ell \leq t$, we regard this sample in the training set; otherwise, it belongs to the hold-out set.
As the threshold varies over all possible values in $(0,\mathcal{U})$, where $\mathcal{U}$ is the upper bound for $\ell$, the area under ROC curve (AUC) is used to measure the information leakage. In perfectly-private situation, the AUC should be 0.5, indicating that the adversary could not infer whether a given sample belongs to the training set or not. The larger the AUC, the more membership information leaks.

\textit{\textcolor{header1}{Improved membership privacy}}. We follow the setup in Section \ref{sec:cnn} here while we only split 64K data into 500 clients and set $\tau=0.2$ for training \cite{jayaraman2019evaluating,Shokri:2017,yeom:2018}. Our test set $\mathcal{T}$ for membership inference attack includes 10K training data and 10K testing data of SVHN. In Figure~\ref{fig:mi_auc}, we show the AUC values of threshold attack against different models. We observe that Non-DP model actually suffers high risk of membership leakage. And applying DP can significantly lower the risk. Comparing with DP-Fed, DP-Fed-LS may even further improve the membership privacy.

\begin{figure}[htbp]
    \centering
    \subfigure[]{\includegraphics[width=1.5in]{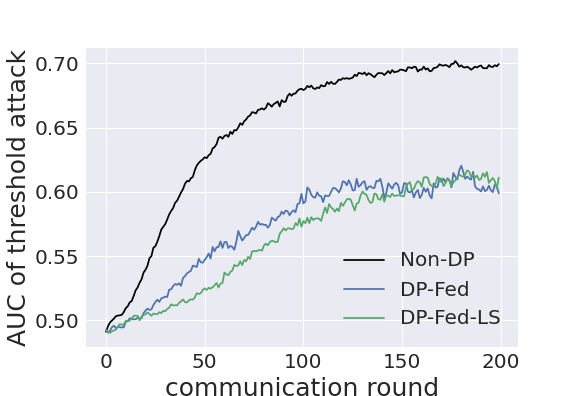}}
    \subfigure[]{\includegraphics[width=1.5in]{imgs/auc_evolution/mi_auc_svhn_uniform_u500_z1.png}}
    \subfigure[]{\includegraphics[width=1.5in]{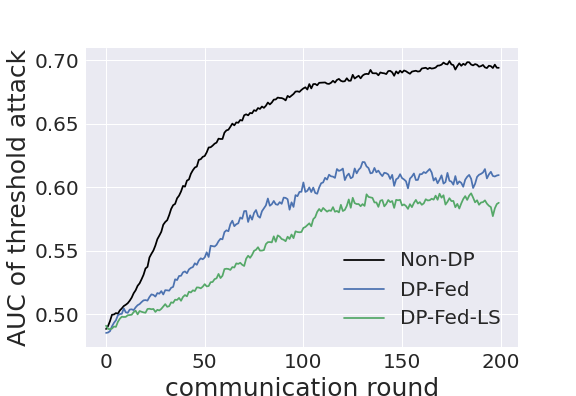}}
    \subfigure[]{\includegraphics[width=1.5in]{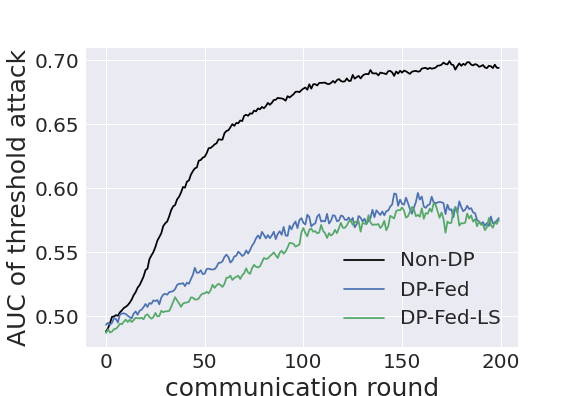}}
    \caption{AUC of threshold attack of different model on SVHN. (a) and (b): uniform subsampling with noise multiplier $z=1$  and $z=1.3$ for DP models. (c) and (d): Poisson subsampling with $z=1$ and $z=1.3$. The larger the AUC, the more membership information leakages. For DP-Fed-LS, the LS parameter $\sigma=1$.}
    \label{fig:mi_auc}
    \vspace{-1em}
\end{figure}

\section{Conclusion}

We introduce Laplacian smoothing to the noisy \textit{federated average of gradients} to improve the generalization accuracy with the same DP guarantee. Privacy bounds in tight closed-form are given under uniform or Poisson subsampling mechanisms, while optimization error bounds help us understand the theoretical effect of LS. Experimental results show that DP-Fed-LS outperforms DP-Fed in both \textbf{IID} and \textbf{Non-IID} settings, regarding accuracy and membership privacy, demonstrating its potential in practical applications.

\appendix
\section{Proof of Theorem \ref{Theorem-privacy-guarantee-federated-uniform}}\label{sec:proof-theorem-uniform}

We firstly introduce the notation of $\ell_2$-sensitivity and composition theorem of RDP.

\begin{definition}[$\ell_2$-Sensitivity] \label{L2-Sensitivity}
For any given function $f(\cdot)$, the $\ell_2$-sensitivity of $f$ is defined by
$$
\Delta (f) = \max_{\|D-D'\|_1=1} \|f(D) - f(D')\|_2,
$$
where $\|D-D'\|_1=1$ means the data sets $D$ and $D'$ differ in only one entry.
\end{definition}

\begin{lemma}[Composition Theorem of RDP \cite{mironov:2017renyi}]\label{lemma:com_post}
	If $k$ randomized mechanisms $\cM_i:\mathcal{D}\rightarrow\cR^d$, for $i\in[k]$, satisfy $(\alpha,\rho_i)$-RDP, then their composition $\big(\cM_1(D),\ldots,\cM_k(D)\big)$ satisfies $(\alpha,\sum_{i=1}^k\rho_i)$-RDP. Moreover, the input of the $i$-th mechanism can be based on outputs of the previous $(i-1)$ mechanisms.
\end{lemma}

Here we are going to provide privacy upper bound for FedAvg (DP-Fed). We drop the superscript $\step$ from $w_j^{t,\step}$ for simplicity, then
\begin{equation}\label{eq:algm2-sgd}
    {w}^{t+1} = {w}^t + \frac{\eta_g}{\sample} \Bigg( \sum_{j \in M_t} {w}^t_j - \sample \cdot {w}^t + {\bf{n}} \Bigg),
\end{equation}
And for the one with Laplacian Smoothing (DP-Fed-LS), it becomes
\begin{equation}\label{eq:algm2-lssgd}
    \tilde {w}^{t+1} = \tilde {w}^t + \frac{\eta_g}{\sample} A_{\sigma}^{-1} \Bigg( \sum_{j \in M_t} \tilde {w}^t_j - \sample \cdot \tilde {w}^t + {\bf{n}} \Bigg),
\end{equation}
where ${\bf{n}} \sim \mathcal{N}(0, \nu^2 I)$, and ${w}_j^t$ is the updated model from client $j$, based on the previous global model $ w^t$.

\begin{proof}
In the following, we will show that the Gaussian noise $\mathcal{N}(0, \nu^2)$ in Eq.~(\ref{eq:algm2-sgd}) for each coordinate of $\nb$, the output of DP-Fed, $\wb$, after $T$ iteration is ($\varepsilon$,$\delta$)-DP. We drop the superscript $\step$ from $w_j^{t,\step}$ for simplicity.

Let us consider the mechanism $\cM_t=\frac{1}{\sample}\sum_{j=1}^K w_j^t - w^t + \frac{1}{\sample} {\bf{n}}$ with query ${\bf{q}}_t=\frac{1}{\sample}\sum_{j=1}^{\client}w_j^t- w^t$ and its subsampled version $\hat \cM_t=\frac{1}{\sample}\sum_{j\in M_t}w_j^t -w^t + \frac{1}{\sample} \nb$. Define the query noise $\nb_q=\nb/\sample$ whose variance is $\nu_q^2:=\nu^2/\sample^2$.
We will firstly evaluate the sensitivity of $w_j^t$. For each local iteration 
$$w_j^t \leftarrow w_j^t - \eta_l g(w_j^t)$$
$$w^t_j \leftarrow w^t + \text{clip}\big( w_j^t - w^{t}, \clip \big),$$
where clip($v,\clip$)  $\leftarrow {v} / {\max}(1, \| v {\|}_2/\clip)$. All the local output $\Delta_j^t \leftarrow w_j^t - w^{t}$ will be inside the $l_2$-norm ball centering around $w^{t}$ with radius $\clip$. 
We have $l_2$-sensitivity of ${\bf{q}}_t$ as $\Delta ({\bf{q}}) = \| w_j^t - w_j^{t'} {\|}_2/\sample \leq 2\clip/\sample.$

According to \cite{mironov:2017renyi}, if we add noise with variance,
\begin{equation}\label{nu_uniform}
    \nu^2=\sample^2\nu_q^2=\frac{14 \tau^2  \alpha T \clip^2}{\lambda \varepsilon},
\end{equation}
the mechanism $\cM_t$ will satisfy $(\alpha,\alpha\Delta^2({\bf{q}})/(2\nu^2_q)) =(\alpha, \lambda \varepsilon/7\tau^2T)$-RDP. By Lemma~\ref{lemma:sub_uniform}, $\hat \cM_t$ will satisfy ($\alpha$,$\lambda \varepsilon/T$)-RDP provided that $\nu_q^2/\Delta^2({\bf{q}}) = \nu^2/(\sample^2\Delta^2({\bf{q}}))\geq \frac{2}{3}$ and $\alpha-1 \leq\frac{ 2\nu_q^2}{3\Delta^2({\bf{q}})} \log \big(1/\tau\alpha(1+\nu_q^2/\Delta^2({\bf{q}}))\big)$. By post-processing theorem, $\tilde \cM_t=A_{\sigma}^{-1}\big( \frac{1}{\sample}\sum_{j\in M_t} w_j^t - w^t + \frac{1}{\sample} \nb \big)$ will also satisfy $(\alpha,\lambda \varepsilon/T)$-RDP. 

Let $\alpha=\log(1/\delta)/((1-\lambda)\varepsilon)+1$, we obtain that $\hat \cM_t$ (and $\tilde \cM_t$) satisfies $(\log(1/\delta)/(1-\lambda)\varepsilon+1,\lambda \varepsilon/T)$-RDP as long as the following inequalities hold
\begin{equation}
    \frac{\nu^2_q}{\Delta^2({\bf{q}})}= \frac{\nu^2}{\sample^2\Delta^2({\bf{q}})}=
    \frac{\nu^2}{4\clip^2}
    \geq \frac{2}{3}
\end{equation}
and 
\begin{equation}
    \alpha - 1 \leq \frac{\nu^2}{6\clip^2} \ln \frac{1}{\tau\alpha(1+\nu^2/4\clip^2)}.
\end{equation}

Therefore, according to Lemma \ref{lemma:com_post}, we have $w^t$ (and $\tilde w^t$) satisfies $(\log(1/\delta)/(1-\lambda)\varepsilon+1,\lambda t \varepsilon/T)$-RDP. Finally, by Lemma \ref{lemma:RDP_to_DP}, we have $w^t$ (and $\tilde w^t$) satisfies ($\lambda t \varepsilon/T+(1-\lambda)\varepsilon$, $\delta$)-DP. Thus, the output of DP-Fed (and DP-Fed-LS), $w$ (and $\tilde {w}$), is ($\varepsilon$,$\delta$)-DP.
\end{proof}

\section{Proof of Theorem~\ref{Theorem-privacy-guarantee-federated-poisson}}\label{sec:proof-theorem-poisson}

\begin{proof}
The proof is identical to proof of Theorem~\ref{Theorem-privacy-guarantee-federated-uniform} except that we use Lemma~\ref{lemma:sub_poisson} instead of Lemma~\ref{lemma:sub_uniform}. According to the definition of Poisson subsampling, we have $l_2$-sensitivity of ${\bf{q}}_t$ as $\Delta ({\bf{q}}) \leq \|w_j^{t'} {\|}_2/\sample \leq \clip/\sample.$ We start from the Eq. (\ref{nu_uniform}) in the proof of Theorem~\ref{Theorem-privacy-guarantee-federated-uniform}. If we add noise with variance 
\begin{equation}\label{nu_poisson}
    \nu^2=\sample^2\nu_q^2=\frac{2 \tau^2  \alpha T \clip^2}{\lambda \varepsilon},
\end{equation}
the mechanism $\cM_t$ will satisfy $(\alpha,\alpha\Delta^2({\bf{q}})/(2\nu^2_q)) =(\alpha, \frac{\lambda \varepsilon}{4\tau^2T})$-RDP. According to Lemma~\ref{lemma:sub_poisson}, $\hat \cM_t$ will satisfy ($\alpha$,$\lambda \varepsilon/T$)-RDP provided that 
\begin{equation}
    \frac{\nu_q^2}{\Delta^2({\bf{q}})} = \frac{\nu^2}{\sample^2\Delta^2({\bf{q}})} = \frac{\nu^2}{\clip^2}\geq
    \frac{5}{9},
\end{equation} 
and 
\begin{equation}
    \alpha - 1 \leq \frac{2\nu^2}{3\clip^2} \ln \frac{1}{\tau\alpha(1+\nu^2/\clip^2)}.
\end{equation}

By post-processing theorem, $\tilde \cM_t=A_{\sigma}^{-1}\big( \frac{1}{\sample}\sum_{j\in M_t} w_j^t - w^t + \frac{1}{\sample} \nb \big)$ will also satisfy $(\alpha,\lambda \varepsilon/T)$-RDP. Let $\alpha=\log(1/\delta)/((1-\lambda)\varepsilon)+1$, we obtain that $\hat \cM_t$ (and $\tilde \cM_t$) satisfies $(\log(1/\delta)/(1-\lambda)\varepsilon+1,\lambda \varepsilon/T)$-RDP. Therefore, according to Lemma \ref{lemma:com_post}, we have $w^t$ (and $\tilde w^t$) satisfies $(\log(1/\delta)/(1-\lambda)\varepsilon+1,\lambda t \varepsilon/T)$-RDP. Finally, by Lemma \ref{lemma:RDP_to_DP}, we have $w^t$ (and $\tilde w^t$) satisfies ($\lambda t \varepsilon/T+(1-\lambda)\varepsilon$, $\delta$)-DP. Thus, the output of DP-Fed (and DP-Fed-LS), $w$ (and $\tilde{w}$), is ($\varepsilon$,$\delta$)-DP.
\end{proof} 

\section{Proof Outline of Theorem~\ref{thm-convergenve}}\label{sec:full-statement-convergence}

\begin{proof}[Proof Sketch]
To prove Theorem~\ref{thm-convergenve}, we establish the following \textbf{Meta Theorem} summarizing the four optimization error terms caused by initial error, heterogeneous clients, stochastic gradient variance and differential privacy noise.

\begin{metatheorem*}\label{meta-thm-convergence}
There exists constant step size $\eta_l$ and $\eta_g$, Gaussian noise $\nu$, communication round $T$, and clipping parameter $\clip$ such that DP-Fed-LS satisfies
\begin{equation}
    \begin{aligned}
      \Er(\bar{w}^T)  \leq \Er_{init} (T) + \Er_{hete} (T) + \Er_{var} (T) + \Er_{dp}(T)
    \end{aligned}
\end{equation}
where $\Er(\bar{w}^T) = \mathbb{E}[f(\bar{w}^T)] - f(w^*)$ for strongly convex and general convex cases while $\Er(\bar{w}^T)= \| \nabla f(\bar{w}) \|_{\textbf{A}_\sigma^{-1}}$ for non-convex case. $\Er_{init}$ is the initial error, $\Er_{hete}$ is introduced by the heterogeneity of clients' data, $\Er_{var}$ accounts for the variance of stochastic gradients, and $\Er_{dp}$ is due to privacy noise under Laplacian smoothing.
\end{metatheorem*}

To see this, the following Theorem~\ref{thm-strongly-convex}, \ref{thm-convex} and \ref{thm-non-convex} instantiate the Meta Theorem for three scenarios, i.e. strongly convex, convex, and non-convex cases of loss functions, respectively, whose detailed proof can be found in Section~\ref{sec:proof-of-convergence} in supplementary material.

\begin{theorem}[$\mu$ Strongly-Convex]\label{thm-strongly-convex}
Under Assumption~\ref{assumption-BGD}, \ref{assumption-smooth}, \ref{assumption-mu-strongly-convex}, \ref{assumption-gradient-variance}, \ref{assumption-G-lipschitz}, $\eta_g \geq 1$, $a_t=(1-\mu_\sigma \eta \step/2)^{-t} /(\sum_{t=0}^T (1-\mu_\sigma \eta \step/2)^{-t})$, $\tilde{\eta} = \min \bigg\{ \frac{2\log(\max(e,\mu_\sigma^2 T D_\sigma/H_\sigma)))}{\mu_\sigma \step T} , \frac{\Lambda_{\min}}{8\step \smooth (1+\heteB^2)} \bigg\}$, $\clip=\eta_l \step \lip$, and $T \geq \frac{1}{\mu_\sigma \tilde{\eta} \step}$, Algorithm~\ref{algorithm:1} with uniform subsampling satisfies
\begin{equation*}
\small
    \Er_{init}  = 3 \mu_\sigma \exp(-\mu_\sigma \tilde{\eta} \step T/2) D_\sigma  \leq \tilde{\mathcal{O}} \bigg(\frac{H_\sigma}{\mu_\sigma T} \bigg), \ \ \ \Er_{var} \leq 2 \tilde{\eta}(1+\frac{\sample}{\eta_g^2}) \frac{\varsigma^2(\sigma)}{\sample} \leq \tilde{\mathcal{O}} \bigg( \frac{(1+\frac{\sample}{\eta_g^2})\varsigma^2(\sigma)}{\mu_\sigma \step \sample T} \bigg )
\end{equation*}
\begin{equation*}
\small
     \Er_{dp} \leq \frac{ 2 \tilde{\eta} d_\sigma \step \lip^2 \nu^2_1}{\sample^2} \leq \tilde{\mathcal{O}} \bigg(\frac{d_\sigma \lip^2 \nu^2_1}{\mu_\sigma \sample^2 T}\bigg), \ \ \ \Er_{hete} \leq 24 \tilde{\eta}^2 \step^2 \smooth \heteG^2 + 8\tilde{\eta} (1-\tau)  \frac{\step}{\sample} \heteG^2 \leq \tilde{\mathcal{O}} \bigg( \frac{\smooth \heteG^2}{\mu_\sigma^2 T^2} + \frac{\heteG^2(1-\tau) }{\mu_\sigma \sample T}  \bigg ).
\end{equation*}
where $\mu_\sigma = \mu \Lambda_{\min}\geq \mu/(1+4\sigma) $,  $H_\sigma=\big( \frac{1}{\eta_g^2 \step} + \frac{1}{\sample \step} \big) \varsigma^2(\sigma) + \frac{4}{\sample} (1-\tau) \heteG^2 + \frac{ \lip^2 \nu^2_1 d_\sigma}{\sample^2}$, $D_\sigma=\|w^{0} - w^* \|_{\textbf{A}_\sigma}^2$, and the effective dimension $d_\sigma=\sum_{i=1}^d \Lambda_i$.

\end{theorem}

\begin{theorem}[General-Convex]\label{thm-convex}
Under Assumption~\ref{assumption-BGD}, \ref{assumption-smooth}, \ref{assumption-convex}, \ref{assumption-gradient-variance}, \ref{assumption-G-lipschitz}, $\eta_g \geq 1$, $a_t=1/(T+1)$, and $\tilde{\eta}=\min \bigg\{\sqrt{\frac{D_\sigma}{H_\sigma T\step^2}}, \sqrt[3]{\frac{D_\sigma}{Q_1 T\step^3}}, \frac{\Lambda_{\min}}{8\step \smooth(1+\heteB^2)} \bigg\}$, $\clip=\eta_l \step \lip$, Algorithm~\ref{algorithm:1} with uniform subsampling satisfies
\begin{equation*}
\footnotesize
    \Er_{init}  = \frac{2}{\tilde{\eta} T \step}D_\sigma  \leq \frac{16\smooth (1+\heteB^2)(1+4\sigma) D_\sigma}{T}, \ \ \ \Er_{var} \leq 2 \tilde{\eta}(1 + \frac{\sample}{\eta_g^2}) \frac{\varsigma^2(\sigma)}{\sample} \leq 2\sqrt{\frac{ (1+\frac{\sample}{\eta_g^2}) \varsigma^2(\sigma)D_\sigma}{\step \sample T}}
\end{equation*}
\begin{equation*}
\footnotesize
     \Er_{dp} \leq \frac{ 2 \tilde{\eta} d_\sigma \step \lip^2 \nu^2_1}{\sample^2} \leq 2\sqrt{\frac{ D_\sigma \lip^2 \nu^2_1 d_\sigma}{\sample^2 T}}, \ \ \ \Er_{hete} \leq \underbrace{24 \smooth \heteG^2}_{Q_1} \tilde{\eta}^2 \step^2+ 8 \tilde{\eta} (1-\tau) \frac{\step}{\sample} \heteG^2 \leq \sqrt[3]{\frac{24\smooth D_\sigma^2 \heteG^2 }{T^2}} + 4\sqrt{\frac{(1-\tau) D_\sigma  \heteG^2}{\sample T}}.
\end{equation*} 
where $H_\sigma=\big( \frac{1}{\eta_g^2 \step} + \frac{1}{\sample \step} \big) \varsigma^2(\sigma) + \frac{4}{\sample} (1-\tau) \heteG^2 + \frac{ \lip^2 \nu^2_1 d_\sigma}{\sample^2}$, $D_\sigma=\|w^0 - w^*\|^{2}_{\textbf{A}_\sigma}$, and the effective dimension $d_\sigma=\sum_{i=1}^d \Lambda_i$.
\end{theorem}

\begin{theorem}[Non-Convex]\label{thm-non-convex}
Under Assumption~\ref{assumption-BGD}, \ref{assumption-smooth}, \ref{assumption-gradient-variance}, \ref{assumption-G-lipschitz}, $\eta_g \geq 1$, $a_t=1/(T+1)$, $\tilde{\eta} = \min\{ \sqrt{\frac{F_0}{H_\sigma T \smooth \step^2}}, \sqrt[3]{\frac{F_0}{Q_2 T \step^3}}, \frac{\Lambda_{\min}}{8 \step \smooth (1+\heteB^2)}\}$, $\clip = \eta_l \step \lip$, Algorithm~\ref{algorithm:1} with uniform subsampling satisfies
\begin{equation*}
\footnotesize
    \Er_{init}  = \frac{8}{\tilde{\eta} T \step}F_0 \leq \frac{64\smooth (1+\heteB^2)(1+4\sigma) F_0}{T}, \ \ \ \Er_{var} \leq 4 \tilde{\eta} \smooth (1+\frac{\sample}{\eta_g^2})  \frac{\varsigma^2(\sigma)}{\sample} \leq 4\sqrt{\frac{ (1+\frac{\sample}{\eta_g^2})\varsigma^2(\sigma) F_0 \smooth}{\step \sample T}}
\end{equation*}
\begin{equation*}
\footnotesize
     \Er_{dp} \leq \frac{ 4 \tilde{\eta} \tilde{d}_\sigma \step \lip^2 \smooth \nu^2_1}{\sample^2} \leq 2\sqrt{\frac{ F_0 \lip^2 \nu^2_1 \smooth \tilde{d}_\sigma}{\sample^2 T}}, 
\end{equation*} 
\begin{equation*}
\footnotesize
    \Er_{hete} \leq \underbrace{32\smooth^2 \heteG^2}_{Q_2} \tilde{\eta}^2 \step^2 + 16 \tilde{\eta} (1-\tau) \smooth  \frac{\step}{\sample}\heteG^2 \leq \sqrt[3]{\frac{32 F_0^2 \heteG^2 \smooth^2}{T^2}} + 8\sqrt{\frac{(1-\tau)F_0 \smooth \heteG^2}{\sample T}}.
\end{equation*}
where $\tilde{H}_\sigma=\big( \frac{1}{\eta_g^2 \step} + \frac{1}{\sample \step} \big) \varsigma^2(\sigma) + \frac{4}{\sample} (1-\tau) \heteG^2 + \frac{ \lip^2 \nu^2_1 \tilde{d}_\sigma}{\sample^2}$, $F_0 = f(w^0) - f(w^*)$, and the effective dimension $\tilde{d}_\sigma=\sum_{i=1}^d \Lambda_i^2$.
\end{theorem}

Finally, Theorem~\ref{thm-convergenve} follows from substituting $\nu_1$ in \ref{thm-strongly-convex}, \ref{thm-convex} and \ref{thm-non-convex} by the one in Theorem ~\ref{Theorem-privacy-guarantee-federated-uniform}. For completeness, we wrap it into a corollary, which ends the proof.
\end{proof} 

\begin{corollary}\label{corollary-convergence-dp-uniform} Assume that $\log(1/\delta) \geq\epsilon$, $\eta_g \geq \sqrt{\sample}$, the conditions on $a_t$, $\clip$ and $\tilde{\eta}$ in Theorem \ref{thm-strongly-convex}, \ref{thm-convex} and \ref{thm-non-convex}, as well as the assumptions in Theorem ~\ref{Theorem-privacy-guarantee-federated-uniform}. Algorithm~\ref{algorithm:1} with uniform subsampling satisfies $(\varepsilon,\delta)$-DP and the following optimization error bounds.


\begin{itemize}
    \item \textbf{$\mu$ Strongly-Convex:} 
    Select $T=\frac{\varepsilon^2 \client^2}{C_0 \lip^2 \sample \log(1/\delta)}$ with $T \geq \frac{1}{\mu_\sigma \tilde{\eta} \step}$ where $\tilde{\eta}$ follows from Theorem~\ref{thm-strongly-convex}. Then 
    \begin{equation*}
        \begin{aligned}
            \mathbb{E} [f(\bar{w}^{T})] - f(w^*)
            & \leq \tilde{\mathcal{O}} \bigg(\frac{( \varsigma^2(\sigma)/\step + (1-\tau) \heteG^2 + d_\sigma) \lip^2 \log(1/\delta)}{\mu_\sigma \varepsilon^2 \client^2} \bigg).
        \end{aligned}
    \end{equation*}
    
    \item \textbf{General-Convex:} Set $T = \frac{\varepsilon^2 \client^2}{C_0 \lip^2 \sample \log(1/\delta)}$. Then
    \begin{equation*} 
        \mathbb{E}[f(\bar{w}^T)] - f(w^*) 
        \leq \frac{\sqrt{(\varsigma^2(\sigma)/\step+ 4(1-\tau)\heteG^2 + d_\sigma)D_\sigma \lip^2 \log(1/\delta)}} {\varepsilon \client}.
    \end{equation*}
    
    \item \textbf{Non-Convex:} Set $T = \frac{\varepsilon^2 \client^2}{C_0 \lip^2 \sample \log(1/\delta)}$. Then 
    \begin{equation*}
        \mathbb{E} \| \nabla f(\bar{w}^T) \|_{\textbf{A}_\sigma^{-1}}^2 
        \leq \frac{\sqrt{(\varsigma^2(\sigma)/\step + 4(1-\tau) \heteG^2 + \tilde{d}_\sigma)F_0  \smooth \lip^2 \log(1/\delta)}} {\varepsilon \client}.
    \end{equation*}
\end{itemize}
\end{corollary}

\bibliographystyle{plain}
\bibliography{references}

\clearpage

\makeatletter
\renewcommand \thesection{S\@arabic\c@section}
\renewcommand\thetable{S\@arabic\c@table}
\renewcommand \thefigure{S\@arabic\c@figure}
\makeatother

\setcounter{section}{0}

\begin{center}
    \large\bfseries Supplymentary Materials
\end{center}

\crefalias{section}{supp}

\section{Proof of Lemma~\ref{lemma:sub_uniform}} \label{sec:proof-lemma-uniform}
\begin{proof}
This proof is inspired by Lemma 3.7 of \cite{wanglingxiao:2019}, while we relax their requirement and get a tighter bound. According to Theorem 9 in \cite{wang2019subsampled}, Gaussian mechanism applied on a subset of size $\sample=\tau\cdot \client$, whose samples are drawn uniformly satisfies $(\alpha, \rho')$-RDP, where 
\begin{equation*}\label{thm9}
\begin{aligned}
    \rho'(\alpha) \leq \frac{1}{\alpha-1} \log\Bigg(1 &+ \tau^2 {\alpha \choose 2} \min \Big\{ 4(e^{\rho(2)}-1), 2e^{\rho(2)} \Big\} + \sum_{j=3}^{\alpha}\tau^{j}{\alpha \choose j} 2 e^{(j-1)\rho(j)} \Bigg)
\end{aligned}
\end{equation*}
where $\rho(j)=j/2\nu^2$. As mentioned in \cite{wang2019subsampled}, the dominant part in the summation on the right hand side arises from the term $\min \big\{ 4(e^{\rho(2)}-1), 2e^{\rho(2)} \big\}$ when $\nu^2$ is relatively large. We will bound this term as a whole instead of bounding it firstly by $4(e^{\rho(2)}-1)$ \cite{wanglingxiao:2019}. For $\nu^2 \geq 0.67$, we have
\begin{equation}\label{eq:sub_wo_replace_1}
\begin{aligned}
    \min & \Big\{ 4(e^{\rho(2)}-1), 2e^{\rho(2)} \Big\} 
    = \min \Big\{ 4(e^{1/\nu^2}-1), 2e^{1/\nu^2} \Big\} \leq 6/\nu^2,
\end{aligned}
\end{equation}
which can be verified numerically as shown in Figure~\ref{fig:lemma_1}.
\begin{figure}[htbp]
    \centering
    \subfigure[]{\includegraphics[width=1.8in]{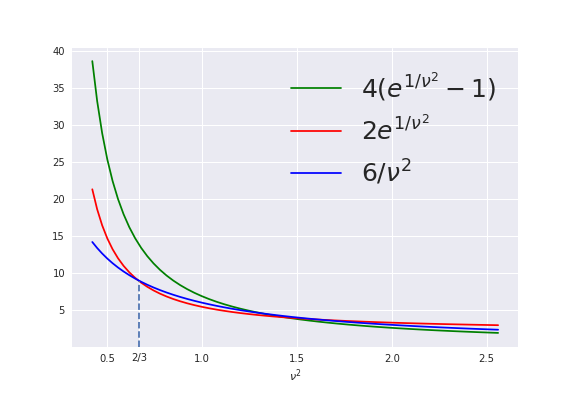}}
    \subfigure[]{\includegraphics[width=1.8in]{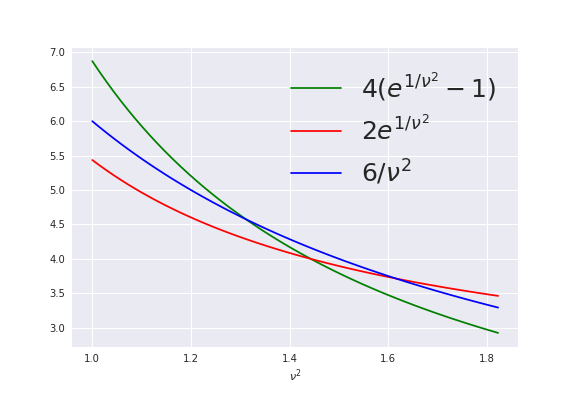}}
    \caption{Numerical comparison of Eq.~(\ref{eq:sub_wo_replace_1}). In (a), we demonstrate the $\min \big\{ 4(e^{1/\nu^2}-1), 2e^{1/\nu^2} \big\} \leq 6/\nu^2$ when $\nu^2\geq \frac{2}{3}$. In (b), we zoom in the range where $\nu \in [1.0, 1.8]$ of (a).}
    \label{fig:lemma_1}
\end{figure}
For the term summing from $j=3$ to $\alpha$, we have
\begin{equation}
\begin{aligned}
    \sum_{j=3}^{\alpha}\tau^{j}{\alpha \choose j} 2 e^{(j-1)\rho(j)} 
    &= \sum_{j=3}^{\alpha}\tau^{j}{\alpha \choose j} 2 e^{\frac{(j-1)j}{2\nu^2}} 
    \leq \sum_{j=3}^{\alpha}\tau^{j}\frac{\alpha^j}{j!} 2 e^{\frac{(j-1)j}{2\nu^2}} \\
    &\leq \sum_{j=3}^{\alpha}\tau^{j}\frac{\alpha^j}{3!} 2 e^{\frac{(\alpha-1)j}{2\nu^2}}
    = \tau^2 \frac{\alpha^2}{3}  \sum_{j=3}^{\alpha}\tau^{j-2} \alpha^{j-2} e^{\frac{(\alpha-1)j}{2\nu^2}}\\
    & \leq \tau^2 {\alpha \choose 2}  \sum_{j=3}^{\alpha}\tau^{j-2} \alpha^{j-2} e^{\frac{(\alpha-1)j}{2\nu^2}}  \\
    & \leq \tau^2 {\alpha \choose 2} \frac{\tau \alpha e^{\frac{3(\alpha-1)}{2\nu^2}}}{1-\tau\alpha e^\frac{\alpha-1}{2\nu^2}} \\
    & \leq \tau^2 {\alpha \choose 2} \frac{\tau \alpha e^{\frac{3(\alpha-1)}{2\nu^2}}}{1-\tau\alpha e^\frac{3(\alpha-1)}{2\nu^2}}
\end{aligned}
\end{equation}
where the first inequality follows from the the fact that ${\alpha \choose j} \leq \frac{\alpha^j}{j!}$, and the last inequality follows from the condition that $\tau\alpha\exp{(\alpha-1)/(2\nu^2)}<1$. In this case, given that 
\begin{equation}\label{lemma_cond1}
    \alpha - 1 \leq \frac{2}{3}\nu^2 \ln \frac{1}{\tau\alpha(1+\nu^2)},
\end{equation}
we have
\begin{equation}\label{eq:sub_wo_replace_2}
    \sum_{j=3}^{\alpha}\tau^{j}{\alpha \choose j} 2 e^{(j-1)\rho(j)} \leq \tau^2 {\alpha \choose 2} \frac{1}{\nu^2} 
\end{equation}
Combining the results in Eq.~(\ref{eq:sub_wo_replace_1}) and Eq.~(\ref{eq:sub_wo_replace_2}), we have
\begin{equation*}
    \begin{aligned}
    \rho'(\alpha) &\leq \frac{1}{\alpha-1} \log \Bigg(1+ {\alpha \choose 2} \frac{6\tau^2}{\nu^2} + {\alpha \choose 2} \frac{\tau^2}{\nu^2} \Bigg) \\
    &\leq \frac{1}{\alpha-1} \tau^2 {\alpha \choose 2} \frac{7}{\nu^2} = 3.5\alpha\tau^2/\nu^2.
    \end{aligned}
\end{equation*}
And condition  $\tau\alpha\exp{(\alpha-1)/(2\nu^2)}<1$ directly follows from Eq.(\ref{lemma_cond1}).
\end{proof}

\section{Proof of Lemma~\ref{lemma:sub_poisson}}\label{sec:proof-lemma-poisson}

\begin{proof}
According to \cite{Mironov2019sampled, zhu2019poission}, Gaussian mechanism applied on a subset where samples are included into the subset with probability ratio $\tau$ independently satifies ($\alpha, \rho'$)-RDP, where
\begin{equation*}
\footnotesize
\begin{aligned}
    \rho'(\alpha) \leq \frac{1}{\alpha-1} \log \Bigg( &(\alpha\tau-\tau+1)(1-\tau)^{\alpha-1}
    + {\alpha \choose 2}(1-\tau)^{\alpha-2} \tau^2 e^{\rho(2)} 
    + \sum_{j=3}^{\alpha} {\alpha \choose j}(1-\tau)^{\alpha-j} \tau^j e^{(j-1)\rho(j)}\Bigg)
\end{aligned}
\end{equation*}
where $\rho(j)=j/2\nu^2$.

We notice that, when $\sigma$ is relatively large, the sum in right-hand side will be dominated by the first two terms. For the first term, we have
\begin{equation}\label{eq:poisson_1}
    (\alpha\tau-\tau+1)(1-\tau)^{\alpha-1} \leq \frac{\alpha\tau-\tau+1}{1+(\alpha-1) \tau} = 1,
\end{equation}
where the first inequality follows from the inequality that
$$(1+x)^n \leq \frac{1}{1-nx}\text{ for } x\in[-1,0], n\in \mathbb{N}.$$
And for the second term, we have 
\begin{equation}\label{eq:poisson_2}
    \tau^2{\alpha \choose 2}(1-\tau)^{\alpha-2}  e^{\frac{1}{\nu^2}} 
    \leq \tau^2{\alpha \choose 2}  e^{\frac{1}{\nu^2}} \leq \tau^2 {\alpha \choose 2} \frac{7}{2\nu^2}
\end{equation}
given that $\nu^2\geq 0.53$. The last inequality is illustrated and verified by numerical comparison in Figure~\ref{fig:lemma_2}.
\begin{figure}[htbp]
    \centering
    \includegraphics[width=2.5in]{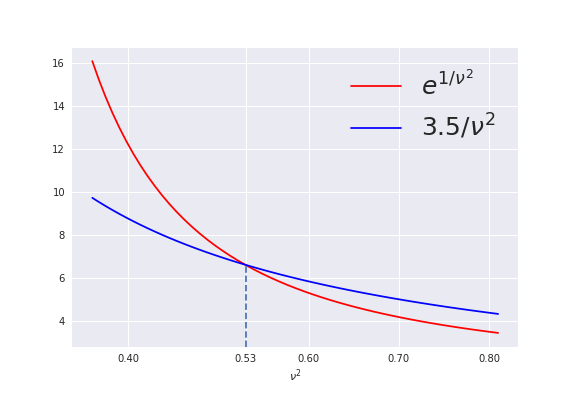}
    \caption{Numerical comparison of Eq.~(\ref{eq:poisson_2}), which demonstrates that $\exp^\frac{1}{\nu^2}\leq \frac{7}{2\nu^2}$ when $\nu^2\geq \frac{5}{9}$.}
    \label{fig:lemma_2}
\end{figure}
And the summation from $j=3$ to $\alpha$ follows Eq.~(\ref{eq:sub_wo_replace_2}) given that
\begin{equation}
    \alpha - 1 \leq \frac{2}{3}\nu^2 \ln \frac{1}{\tau\alpha(1+\nu^2)}.
\end{equation}
Combining Eq. (\ref{eq:poisson_1}), (\ref{eq:poisson_2}) and (\ref{eq:sub_wo_replace_2}), we have 

\begin{equation}
\begin{aligned}
    \rho'(\alpha) 
    &\leq \frac{1}{\alpha-1} \log\bigg( 1+ \tau^2 {\alpha \choose 2} \frac{7}{2\nu^2} + \tau^2 {\alpha \choose 2}\frac{1}{2\nu^2}  \bigg) \leq  \tau^2 {\alpha} \frac{4}{2\nu^2} = 2 \alpha \tau^2/\nu^2.
\end{aligned}
\end{equation}

\end{proof}

\section{Proof of Theorem~\ref{thm-convergenve}}\label{sec:proof-of-convergence}
We firstly provide some useful Lemmas.

\begin{lemma}[Noise reduction of Laplacian smoothing]\label{lemma-ls-noise-reduction}
Consider Gaussian noise $\textbf{n}(\clip) \sim \mathcal{N}(0,\nu^2_{\clip}\textbf{I})$, where $\nu_\clip$ is the noise level scaled by $\clip$, i.e. $\nu_\clip = \clip \nu_1$. We have
\begin{equation*}
    \mathbb{E} \| \textbf{n}(\clip) \|_{\textbf{A}_\sigma^{-1}}^2 \leq \clip^2 \nu^2_1 d_\sigma
\end{equation*}
where $d_\sigma \coloneqq d \zeta_\sigma$ and $\zeta_\sigma\coloneqq \frac{1}{d} \sum_{i=i}^d \frac{1}{1+2\sigma-2\sigma\cos(2\pi i/d)}$.
\end{lemma}

\begin{proof}[Proof of Lemma~\ref{lemma-ls-noise-reduction}]
The proof is inspired by the proof of Lemma 4 in \cite{wang:2019}. Let the eigenvalue decomposition of $\textbf{A}_\sigma^{-1}$ be $\textbf{A}_\sigma^{-1} = {\bf U \Lambda U^\top}$, where $\bf{\Lambda}$ is a diagonal matrix with $\Lambda_i=\frac{1}{1-2\sigma-2\sigma\cos(2\pi i/d)}$, we have
\begin{equation*}
    \begin{aligned}
        \mathbb{E}\| \textbf{n} (\clip)\|_{\textbf{A}_\sigma^{-1}}^2 
        &= \mathbb{E}[\mbox{Tr}({\bf n^\top U \Lambda U^{\top} n} )] \\
        &= \mbox{Tr}({\bf U \Lambda U^{\top}} \mathbb{E}[{ \bf n  n^\top}])  \\
        &= \nu^2_{\clip} \mbox{Tr}({\bf U \Lambda U^{\top}}) \\
        &= \nu^2_{\clip} \sum_{i=i}^d \frac{1}{1+2\sigma-2\sigma\cos(2\pi i/d)} \\
        &= \clip^2 \nu^2_1 d_\sigma
    \end{aligned}
\end{equation*}
where $\zeta_\sigma = \frac{1}{d}\sum_{i=1}^d \Lambda_i$.
\end{proof}

\begin{lemma}[Noise reduction of Laplacian smoothing ]\label{lemma-ls-noise-reduction-2}
Consider Gaussian noise $\textbf{n}(\clip) \sim \mathcal{N}(0,\nu^2_{\clip}\textbf{I})$, where $\nu_\clip$ is the noise level scaled by $\clip$, i.e. $\nu_\clip = \clip \nu_1$. We have
\begin{equation*}
    \mathbb{E} \| \textbf{A}_\sigma^{-1} \textbf{n}(\clip) \|_2^2 \leq \clip^2 \nu^2_1 \tilde{d}_\sigma
\end{equation*}
where $\tilde{d}_\sigma \coloneqq  d \varphi_\sigma$ and $\varphi_\sigma \coloneqq \frac{1}{d} \sum_{i=i}^d \frac{1}{(1+2\sigma-2\sigma\cos(2\pi i/d))^2}$.
\end{lemma}

\begin{proof}[Proof of Lemma \ref{lemma-ls-noise-reduction-2}] 
The proof is inspired by Lemma 5 in \cite{wang:2019}. Let the eigenvalue decomposition of $\textbf{A}_\sigma^{-1}$ be $\textbf{A}_\sigma^{-1} = {\bf U \Lambda U^\top}$, where $\bf{\Lambda}$ is a diagonal matrix with $\Lambda_i=\frac{1}{1+2\sigma-2\sigma\cos(2\pi i/d)}$, we have
\begin{equation*}
    \begin{aligned}
        \mathbb{E} \| \textbf{A}_\sigma^{-1} \textbf{n}(\clip) \|_2^2 
        &= \mathbb{E} [\mbox{Tr}({\bf n^\top U}{\bf \Lambda}^2 \bf{U^\top n})]\\
        &= \mbox{Tr}({\bf U} {\bf \Lambda}^2 {\bf U^\top} \mathbb{E}[{\bf n n^\top}]) \\
        &= \nu^2_{\clip} \mbox{Tr}({\bf U} {\bf \Lambda}^2 {\bf U^\top})\\
        &= \nu^2_\clip \sum_{i=i}^d \frac{1}{(1+2\sigma-2\sigma\cos(2\pi i/d))^2} \\
        &= \clip^2  \nu^2_1 \tilde{d}_\sigma
    \end{aligned}
\end{equation*}
where $\varphi_\sigma = \frac{1}{d}\sum_{i=1}^d \Lambda_i^2$.
\end{proof}

\begin{lemma}[Bounding the divergence of local parameters]\label{lemma-bound-local-divergence-2} Following convexity, Assumption \ref{assumption-BGD}, \ref{assumption-smooth}, \ref{assumption-gradient-variance} and $\eta_l\leq \frac{\Lambda_{\min}}{8\step \smooth (1+\heteB^2)}$, we have
\begin{equation*}
\begin{aligned}
    \frac{1}{\client} \sum_{j=1}^\client \sum_{i=1}^\step \mathbb{E} \|w^t - w_j^{t,i} \|_2^2 
    &\leq 4\step^3\eta_l^2  \heteG^2 + 4\step^3\eta_l^2  \heteB^2 \| \nabla f(w^t)\|_2^2 + 2\step^2 \frac{\eta_l^2}{\Lambda_{\min}}  \varsigma^2(\sigma)\\
    &\leq 4\step^3\eta_l^2  \heteG^2 +8\step^3\eta_l^2 \heteB^2 \smooth \big( f(w^t) - f(w^*) \big)+ 2\step^2 \frac{\eta_l^2}{\Lambda_{\min}}  \varsigma^2(\sigma) 
\end{aligned}
\end{equation*}
where $w_j^{t,i}$ denote the model of client $j$ in $i$-th iteration of the $t$-th communication round.
\end{lemma}

\begin{proof}[Proof of Lemma~\ref{lemma-bound-local-divergence-2}]
The proof of is inspired by Lemma 8 in \cite{karimireddy2020scaffold}, while we consider the $\textbf{A}_\sigma^{-1}$ norm. Recall that the local update made on client $j$ is $w_j^{t,i} = w_j^{t, i-1} + \eta_l \nabla f_j(w_j^{t,i-1}, x_j^{i-1})$. When $i=0$, the $w_j^{t,i}$ will just equal $w^t$. For $i\geq 1$, we have: 
\begin{equation*}
    \begin{aligned}
        \mathbb{E} \| w_j^{t,i} - w^t\|_2^2
        &= \mathbb{E} \| w_j^{t,i-1} - w^t - \eta_l g_j (w_j^{t,i-1})\|_{2}^2 \\
        &= \mathbb{E} \| w_j^{t,i-1} - w^t - \eta_l \nabla f_j (w_j^{t,i-1})\|_{2}^2 + \eta_l^2 \mathbb{E} \| g_j(w_j^{t, i-1}) - \nabla f_j (w_j^{t, i-1}) \|_{2}^2 \\
        &\leq \mathbb{E} \| w_j^{t,i-1} - w^t - \eta_l \nabla f_j (w_j^{t,i-1})\|_{2}^2 + \frac{\eta_l^2}{\Lambda_{\min}} \mathbb{E} \| g_j(w_j^{t, i-1}) - \nabla f_j (w_j^{t, i-1}) \|_{\textbf{A}_\sigma^{-1}}^2 \\
        &\leq \mathbb{E} \| w_j^{t,i-1} - w^t - \eta_l \nabla f_j (w_j^{t,i-1}) \|_{2}^2 + \frac{\eta_l^2}{\Lambda_{\min}} \varsigma_j^2(\sigma) \\
        &\leq \bigg(1-\frac{1}{\step-1}\bigg) \mathbb{E} \|w_j^{t,i-1} - w^t  \|_2^2 + \step \eta_l^2 \| \nabla f_j(w_j^{t, i-1}) \|_2^2 + \frac{\eta_l^2}{\Lambda_{\min}} \varsigma_j^2(\sigma) \\
        &\leq \bigg(1-\frac{1}{\step-1}\bigg) \mathbb{E} \|w_j^{t,i-1} - w^t  \|_2^2 + 2\step \eta_l^2 \mathbb{E} \| \nabla f_j(w_j^{t, i-1}) - \nabla f_j(w^t)\|_2^2 \\
        &\quad + 2\step \eta_l^2 \| \nabla f_j(w^t) \|_2^2 + \frac{\eta_l^2}{\Lambda_{\min}} \varsigma_j^2(\sigma) \\
        &\leq \bigg( 1-\frac{1}{\step-1} + 2\step\eta_l^2 \smooth^2 \bigg) \mathbb{E} \|w_j^{t,i-1} - w^t  \|_2^2 + 2\step \eta_l^2 \| \nabla f_j(w^t) \|_2^2 + \frac{\eta_l^2}{\Lambda_{\min}} \varsigma_j^2(\sigma) \\
        &\leq \bigg( 1-\frac{1}{2(\step-1)} \bigg) \mathbb{E} \|w_j^{t,i-1} - w^t  \|_2^2 + 2\step \eta_l^2 \| \nabla f_j(w^t) \|_2^2 + \frac{\eta_l^2}{\Lambda_{\min}} \varsigma_j^2(\sigma)
        \end{aligned}
    \end{equation*}
where the last inequality comes from the assumption that $\eta_l \leq \frac{1}{2\step \smooth}$. Unrolling the recursion above, we have
\begin{equation*}
    \begin{aligned}
        \mathbb{E} \| w_j^{t,i} - w^t\|_2^2 
        &\leq \sum_{k = 0}^{i}  \big( 2\step\eta_l^2 \| \nabla f_j(w^t) \|_2^2  + \frac{\eta_l^2}{\Lambda_{\min}}\varsigma_j^2(\sigma) \big) \bigg( 1 - \frac{1}{2(\step-1)} \bigg)^{k} \\
        &\leq 2\step \big( 2\step\eta_l^2 \| \nabla f_j(w^t) \|_2^2  + \frac{\eta_l^2}{\Lambda_{\min}} \varsigma_j^2(\sigma) \big) 
    \end{aligned}
\end{equation*}
where the last step is due to 
\begin{equation*}
    \begin{aligned}
        \sum_{k=0}^i \bigg( 1 - \frac{1}{2(\step-1)} \bigg)^k = \frac{1- \big( 1 - \frac{1}{2(\step-1)} \big)^{i+1}}{1-\big( 1 - \frac{1}{2(\step-1)} \big)}
        \leq \frac{1}{1-\big( 1 - \frac{1}{2(\step-1)} \big)} \leq 2(\step-1) \leq 2\step.
    \end{aligned}
\end{equation*}
Taking average over $i$ and $j$, and considering Assumption~\ref{assumption-BGD}, we have
\begin{equation*}
    \begin{aligned}
        \frac{1}{\client} \sum_{j=1}^\client \sum_{i=1}^\step \mathbb{E} \|w^t - w_j^{t,i} \|_2^2
        &\leq \frac{1}{\client} \sum_{j=1}^\client 4\step^3 \eta_l^2 \| \nabla f_j(w^t) \|_2^2 + 2\step^2 \frac{\eta_l^2}{\Lambda_{\min}} \varsigma^2(\sigma) \\
       &\leq 4\step^3\eta_l^2  \heteG^2 + 4\step^3\eta_l^2  \heteB^2 \| \nabla f(w^t)\|_2^2 + 2\step^2 \frac{\eta_l^2}{\Lambda_{\min}}  \varsigma^2(\sigma)\\
        &\leq 4\step^3\eta_l^2  \heteG^2 +8\step^3\eta_l^2 \heteB^2 \smooth \big( f(w^t) - f(w^*) \big)+ 2\step^2 \frac{\eta_l^2}{\Lambda_{\min}}  \varsigma^2(\sigma),
    \end{aligned}
\end{equation*}
which complete the proof.
\end{proof}

\begin{lemma}[Perturbed Strongly Convexity]\label{lemma-perturbed-strongly-convexity}
The proof is inspired by Lemma 5 in  \cite{karimireddy2020scaffold} while we discuss it under $\textbf{A}_\sigma$ norm. The following holds for any $\smooth$-smoothness and $\mu$-strongly convex function $h$, and for any $x,y,z$ in the domain of $h$:
\begin{equation*}
    \langle \nabla h(x), z-y \rangle \geq h(z) - h(y) + \frac{\mu_\sigma}{4} \|y-z \|_{\textbf{A}_\sigma}^2 -\smooth \| z-x \|_2^2.
\end{equation*}
where $\mu_\sigma = \mu \Lambda_{\min}$ and $\Lambda_{\min}$ is the smallest eigenvalue of $\textbf{A}_\sigma^{-1}$.
\end{lemma}

\begin{proof}[Proof of Lemma~\ref{lemma-perturbed-strongly-convexity}]
Given $x$, $y$ and $z$, according to $\smooth$-smoothness and $\mu$-strongly-convexity of $h$, we have
\begin{equation*}
    \begin{aligned}
         \langle\nabla h(x), z-x \rangle &\geq h(z) - h(x) - \frac{\smooth}{2} \| z-x \|^2  \\
         \langle\nabla h(x), x-y \rangle &\geq h(x) - h(y) + \frac{\mu}{2} \| y-x \|^2  \\
    \end{aligned}
\end{equation*}
Further more, we know that $2 \|u\|^2 + 2 \|v\|^2 \geq \|u+v\|^2$. If we let $u = y-x$ and $v=x-z$, we have $2\|y-x \|^2 + 2 \|x-z \|^2 \geq \|y-z\|^2$. In this case, we have 

\begin{equation*}
    \frac{\mu}{2} \| y-x \|^2 \geq \frac{\mu}{4} \|y-z \|^2 - \frac{\mu}{2}\| x-z\|^2 
\end{equation*}
Combining all the inequalities, we have
\begin{equation*}
\begin{aligned}
    \langle \nabla h(x), z-y \rangle 
    &\geq h(z) - h(y) + \frac{\mu}{4} \| y-z \|_2^2 - \frac{ \smooth+\mu}{2} \|z-x \|^2 \\
    &\geq h(z) - h(y) + \frac{\mu}{4} \| y-z \|_2^2 - \smooth \|z-x \|^2 \\
    &\geq h(z) - h(y) + \frac{\mu \Lambda_{\min}}{4} \| y-z \|_{\textbf{A}_\sigma}^2 - \smooth \|z-x \|^2 \\
    &= h(z) - h(y) + \frac{\mu_\sigma}{4} \| y-z \|_{\textbf{A}_\sigma}^2 - \smooth \|z-x \|^2 \\
\end{aligned}
\end{equation*}
where $\mu_\sigma = \mu \Lambda_{\min}$, and $\Lambda_{\min}$ is the smallest eigenvalue of $\textbf{A}_\sigma^{-1}$.
\end{proof}

\begin{lemma}[Subsampling Variance (Lemma B.1 in \cite{lei2017less})]  \label{lemma-subsample-variance}
Given a vector space $\mathcal{X}\in \mathbb{R}^d$ with norm $\|\cdot\|$, we consider a dataset $x_1, x_2, ..., x_N \in \mathcal{X}$. We select a subset $\mathcal{S}$ with size $S$ from the given dataset without replacement. The subsampling mechanism can be uniform subsampling or Poission subsampling. The variance of the subset's average can be bounded by the following upper bound:
\begin{equation*}
    \begin{aligned}
        \mathbb{E} \bigg\| \frac{1}{S} \sum_{j \in \mathcal{S}} x_j - \bar{x} \bigg\|^2 
        &= \frac{1}{S}\bigg(1-\frac{S-1}{N-1} \bigg) \frac{1}{N} \sum_{j=1}^N \|x_j - \bar{x}\|^2,
    \end{aligned}
\end{equation*}
\end{lemma}

\subsection{Setup}

Before the proof of the main theorem, we denote the sever update in round $t$ as $\Delta^t$, which can be expressed as:
\begin{equation*}
    \Delta^t = - \frac{\tilde{\eta}}{\sample} \sum_{j=1}^{\sample}\sum_{i=1}^\step \textbf{A}_\sigma^{-1} \nabla g_j(w_j^{t,i}) + \textbf{A}_\sigma^{-1} \eta_g \frac{\textbf{n}(\clip)}{\sample}
    \quad \mbox{and} \quad
    \mathbb{E}[\Delta^t] = -\frac{\tilde{\eta}}{\sample} \sum_{j=1}^\sample \sum_{i=1}^\step \mathbb{E} \textbf{A}_\sigma^{-1} \nabla f_j(w^{t,i}_j)
\end{equation*}
where $\tilde{\eta}=\eta_l \eta_g$, $\textbf{n}(\clip) \sim \mathcal{N}(0,\nu^2_\clip \textbf{I})$, and $\nu_\clip$ is the noise level as a proportional function of the clipping parameter $\clip$. We get $\nu_\clip = \clip \nu_1$ in Theorem \ref{Theorem-privacy-guarantee-federated-uniform} with clipping parameter $\clip$.

Let the eigenvalue decomposition of $\textbf{A}_\sigma^{-1}$ be $\textbf{A}_\sigma^{-1} = {\bf U \Lambda U^\top}$, where ${\bf \Lambda}=\mbox{diag}(\Lambda_i)$ is a diagonal matrix with 
$$\Lambda_i=\frac{1}{1+2\sigma(1-\cos(2\pi i/d))},$$
and denote the smallest eigenvalue of $\textbf{A}_\sigma^{-1}$ by
$$\Lambda_{\min}=\min_{1\leq i\leq d} \frac{1}{1+2\sigma(1-\cos(2\pi i/d))}\geq \frac{1}{1+4\sigma}.$$

\subsection{Proof of Theorem~\ref{thm-strongly-convex}}

\begin{proof}
We start from the total update of a communication round
\begin{equation}\label{strongly-convex-scaffold-single-step}
    \begin{aligned}
        \mathbb{E}\|w^{t+1} - w^* \|_{\textbf{A}_\sigma}^2 
        &= \mathbb{E}\bigg\| w^t + \Delta^t - w^* \bigg\|_{\textbf{A}_\sigma}^2 \\
        & = \mathbb{E}\bigg\|w^t - w^* - \frac{\tilde{\eta}}{\sample} \sum_{j=1}^{\sample}\sum_{i=1}^\step \textbf{A}_\sigma^{-1} g_j(w_j^{t,i}) + \textbf{A}_\sigma^{-1} \eta_g^2 \frac{\textbf{n}(\clip)}{\sample} \bigg \|_{\textbf{A}_\sigma}^2 \\
        &\leq \mathbb{E}\|w^t-w^*\|_{\textbf{A}_\sigma}^2 
        \underbrace{-\frac{2 \tilde{\eta}}{\client} \sum_{j=1}^{\client}\sum_{i=1}^\step \langle w^t - w^*, \nabla f_j(w_j^{t,i}) \rangle}_{A_1}
        + \eta_g^2 \underbrace{\mathbb{E} \bigg\| \frac{\textbf{n}(\clip)}{\sample} \bigg\|_{\textbf{A}_\sigma^{-1}}^2}_{A_2} \\
        & \quad + \underbrace{\tilde{\eta}^2 \mathbb{E}\bigg\| \frac{1}{\sample} \sum_{j=1}^{\sample}\sum_{i=1}^\step  g_j(w_j^{t,i}) \bigg\|_{\textbf{A}_\sigma^{-1}}^2}_{A_3}
    \end{aligned}
\end{equation}

As for $A_1$, we apply Lemma~\ref{lemma-perturbed-strongly-convexity}, we have
\begin{equation*}
    \begin{aligned}
        A_1 &= \frac{2 \tilde{\eta}}{\client} \sum_{j=1}^{\client}\sum_{i=1}^\step \langle w^* - w^t , \nabla f_j(w_j^{t,i}) \rangle \\
        &\leq \frac{2\tilde{\eta}}{\client} \sum_{j=1}^{\client}\sum_{i=1}^\step \bigg( f_j(w^*) - f_j(w^t) + \smooth \| w_j^{t,i} - w^t \|_2^2 - \frac{\mu_\sigma}{4} \| w^t - w^* \|_{\textbf{A}_\sigma}^2 \bigg) \\
        &\leq - 2\tilde{\eta} \step \big( f(w^t) - f(w^*) \big) +  \frac{2\tilde{\eta} \smooth}{\client} \sum_{j=1}^{\client}\sum_{i=1}^\step \| w_j^{t,i} - w^t \|_2^2 - \frac{\tilde{\eta} \mu_\sigma \step}{2} \| w^t - w^* \|_{\textbf{A}_\sigma}^2
    \end{aligned}
\end{equation*}

As for $A_3$, by the equation $\mathbb{E}X^2 = (\mathbb{E}X)^2 + \mathbb{E}(X - \mathbb{E}X)^2$, we have
\begin{equation}\label{eq_A3}
    \begin{aligned}
        A_3 \leq \underbrace{\tilde{\eta}^2 \mathbb{E} \bigg\| \frac{1}{\sample} \sum_{j=1}^{\sample}\sum_{i=1}^\step \nabla f_j(w_j^{t,i}) \bigg\|_{\textbf{A}_\sigma^{-1}}^2}_{B_1} + \underbrace{\tilde{\eta}^2 \mathbb{E} \bigg\| \frac{1}{\sample} \sum_{j=1}^{\sample}\sum_{i=1}^\step \big( g_j(w_j^{t,i}) - \nabla f_j(w_j^{t,i}) \big) \bigg\|_{\textbf{A}_\sigma^{-1}}^2}_{B_2}
    \end{aligned}
\end{equation}
For $B_1$, we have
\begin{equation}\label{eq_B1}
    \begin{aligned}
        B_1 &= \tilde{\eta}^2 \mathbb{E} \bigg\| \frac{1}{\sample} \sum_{j=1}^{\sample}\sum_{i=1}^\step \big(\nabla f(w_j^{t,i}) -\nabla f_j(w^t) + \nabla  f_j(w^t) \big) \bigg\|_{\textbf{A}_\sigma^{-1}}^2 \\
        &\leq 2\tilde{\eta}^2 \mathbb{E} \bigg\| \frac{1}{\sample} \sum_{j=1}^{\sample}\sum_{i=1}^\step \big(\nabla f(w_j^{t,i}) -\nabla f_j(w^t) \big) \bigg\|_{\textbf{A}_\sigma^{-1}}^2 + 2\tilde{\eta}^2 \step^2  \mathbb{E} \bigg\| \frac{1}{\sample} \sum_{j=1}^{\sample} \nabla f_j(w^t) \bigg\|_{\textbf{A}_\sigma^{-1}}^2 \\
        &\leq \frac{2\tilde{\eta}^2 \step}{\client} \sum_{j=1}^\client \sum_{i=1}^\step \mathbb{E} \| \nabla f(w_j^{t,i}) -\nabla f(w^t) \|_2^2 \\
        & \qquad \qquad + 2\tilde{\eta}^2 \step^2  \mathbb{E} \bigg\| \frac{1}{\sample} \sum_{j=1}^{\sample} \nabla f_j(w^t) - \nabla f(w^t) + \nabla f(w^t) \bigg\|_{\textbf{A}_\sigma^{-1}}^2\\
        &\leq \frac{2\tilde{\eta}^2 \step \smooth^2}{\client} \sum_{j=1}^\client \sum_{i=1}^\step \mathbb{E}\| w_j^{t,i} - w^t \|_2^2 + 2\tilde{\eta}^2 \step^2 \|\nabla f(w^t)\|_2^2 \\
        & \qquad \qquad + 2\tilde{\eta}^2 \step^2  \mathbb{E} \bigg\| \frac{1}{\sample} \sum_{j=1}^{\sample} \nabla f_j(w^t) - \nabla f(w^t) \bigg\|_{\textbf{A}_\sigma^{-1}}^2 \\
        &\leq \frac{2\tilde{\eta} \step \smooth^2}{\client} \sum_{j=1}^\client \sum_{i=1}^\step \mathbb{E}\| w_j^{t,i} - w^t \|_2^2 + 2\tilde{\eta}^2 \step^2 \|\nabla f(w^t)\|_2^2 \\
        & \qquad \qquad + 4\tilde{\eta}^2\step^2 \bigg(1-\frac{\sample}{\client}\bigg) \frac{1}{\sample}  \bigg(\heteG^2 + \heteB^2 \| \nabla f(w^t) \|_2^2\bigg) \\
        &\leq \frac{2\tilde{\eta}^2 \step \smooth^2}{\client} \sum_{j=1}^\client \sum_{i=1}^\step \mathbb{E}\| w_j^{t,i} - w^t \|_2^2 + 8\tilde{\eta}^2 \step^2 (1+\heteB^2) \smooth \big( f(w^t) - f(w^*) \big) \\
        & \qquad \qquad + \frac{4\tilde{\eta}^2 \step^2}{\sample} \bigg( 1-\frac{\sample}{\client} \bigg)  \heteG^2 \\
    \end{aligned}
\end{equation}
where the second last inequality comes from Assumption~\ref{assumption-BGD} and Lemma~\ref{lemma-subsample-variance}. When we apply Lemma \ref{lemma-subsample-variance}, we set $x_j = f_j(w^t)$ and $\bar{x}= \frac{1}{\client}\sum_{j=1}^\client f_j(w^t)$. What's more, we use the inequality $1-\frac{\sample-1}{\client-1} \leq 2\big(1 -\frac{\sample}{\client}\big)$. And the last inequality comes from Assumption~\ref{assumption-smooth}. As for $B_2$, we apply Assumption \ref{assumption-gradient-variance}, then we have
\begin{equation}\label{eq_B2}
    \begin{aligned}
        B_2 &\leq 
        \tilde{\eta}^2 \mathbb{E} \bigg\| \frac{1}{\sample} \sum_{j=1}^{\sample}\sum_{i=1}^\step \big( g_j(w_j^{t,i}) - \nabla f_j(w_j^{t,i}) \big) \bigg\|_{\textbf{A}_\sigma^{-1}}^2 \\
        &\leq \tilde{\eta}^2\mathbb{E} \frac{1}{\sample^2} \sum_{j=1}^{\sample}\sum_{i=1}^\step  \| g_j(w_j^{t,i}) - \nabla f_j(w_j^{t,i}) \|_{\textbf{A}_\sigma^{-1}}^2 \\
        &\leq \frac{\tilde{\eta}^2 \step}{\sample}\varsigma^2(\sigma)
    \end{aligned}
\end{equation}
And by Lemma~\ref{lemma-ls-noise-reduction} and the assumption that $\eta_l \leq \frac{\Lambda_{\min}}{8\step \smooth \eta_g (1+\heteB^2)}$, whence $\tilde{\eta} =\eta_l\eta_g \step \smooth \leq\frac{1}{2}$, we have
\begin{equation*}
    \begin{aligned}
        \mathbb{E}\|w^{t+1} - w^*\|_{\textbf{A}_\sigma}^2  
        &\leq \bigg(1 - \frac{\tilde{\eta} \mu_\sigma \step}{2} \bigg)\mathbb{E}\|w^{t} - w^* \|_{\textbf{A}_\sigma}^2  -  \tilde{\eta} \step \big( f(w^t) - f(w^*) \big)   \\
        &\quad + \underbrace{3 \tilde{\eta} \smooth  \frac{1}{\client} \sum_{j=1}^\client \sum_{i=1}^\step \mathbb{E}\| w_j^{t,i} - w^t \|_{2}^2}_C  + \frac{\tilde{\eta}^2\step}{\sample}\varsigma^2(\sigma) + \eta_g^2 \frac{\eta_l^2 \step^2 \lip^2 \nu^2_1 d_\sigma}{\sample^2} \\
        &\quad + \frac{4\tilde{\eta}^2 \step^2}{\sample} \bigg( 1-\frac{\sample}{\client} \bigg)  \heteG^2
    \end{aligned}
\end{equation*}
According to Lemma~\ref{lemma-bound-local-divergence-2}, and the assumption $\eta_g \geq 1$ ($\tilde{\eta} = \eta_g \eta_l \geq \eta_l$) and $\eta_l \leq \frac{\Lambda_{\min}}{8\step \smooth \eta_g (1+\heteB^2)}$, for $C$, we have
\begin{equation}
    \begin{aligned}
        3 \tilde{\eta} \smooth  \frac{1}{\client} \sum_{j=1}^\client \sum_{i=1}^\step \mathbb{E}\| w_j^{t,i} - w^t \|_{2}^2 
        &\leq  12\tilde{\eta} \eta_l^2 \step^3 \smooth \heteG^2 + \frac{6 \tilde{\eta} \eta_l^2 \smooth \step^2}{\Lambda_{\min}} \varsigma^2(\sigma) \\
        & \qquad \qquad + 24\tilde{\eta} \eta_l^2 \step^3 \heteB^2 \smooth^2 \big( f(w^r) - f(w^*) \big) \\
        &\leq  12\tilde{\eta}^3 \step^3 \smooth \heteG^2 + \frac{\tilde{\eta}^2 \step}{\eta_g^2} \varsigma^2(\sigma) + \frac{1}{2}\tilde{\eta} \step \big( f(w^t) - f(w^*) \big)
    \end{aligned}
\end{equation}
In this case,
\begin{equation*}
    \begin{aligned}
        \mathbb{E}\|w^{t+1} - w^* \|_{\textbf{A}_\sigma}^2  
        &\leq \bigg(1 - \frac{\tilde{\eta} \mu_\sigma \step}{2} \bigg)\mathbb{E}\|w^{t} - w^* \|_{\textbf{A}_\sigma}^2  - \frac{1}{2}\tilde{\eta} \step \big( f(w^t) - f(w^*) \big) \\
        &\quad + \tilde{\eta}^2\step^2 \bigg( \big( \frac{1}{\eta_g^2 \step} + \frac{1}{\sample \step} \big) \varsigma^2(\sigma) + \frac{ \lip^2 \nu^2_1 d_\sigma}{\sample^2} + \frac{4}{\sample} \bigg( 1 - \frac{\sample}{\client}\bigg) \heteG^2 + 12 \tilde{\eta} \step \smooth \heteG^2  \bigg) \\
    \end{aligned}
\end{equation*}
Reorganizing the terms, we have
\begin{equation}\label{eq-convex-intermedia-step}
    \begin{aligned}
        f(w^t) - f(w^*) &\leq \frac{2}{\tilde{\eta} \step} \bigg(1 - \frac{\tilde{\eta} \mu_\sigma \step}{2} \bigg) \mathbb{E}\|w^{t} - w^* \|_{\textbf{A}_\sigma}^2 - \frac{2}{\tilde{\eta} \step} \mathbb{E}\|w^{t+1} - w^* \|_{\textbf{A}_\sigma}^2 \\
        &\quad + 2\tilde{\eta} \step \bigg( \big( \frac{1}{\step \eta_g^2} + \frac{1}{\sample \step} \big) \varsigma^2(\sigma) + \frac{ \lip^2 \nu^2_1 d_\sigma}{\sample^2} + \frac{4}{\sample} \bigg( 1 - \frac{\sample}{\client}\bigg) \heteG^2 + 12 \tilde{\eta} \step \smooth \heteG^2  \bigg)
    \end{aligned}
\end{equation}
By averaging using weights $a_t = q^{-t}$ where $q \triangleq \big(1-\frac{\mu_\sigma \tilde{\eta} \step}{2}\big)$, we have
\begin{equation*}
    \begin{aligned}
      \sum_{t=0}^T a_t \big( \mathbb{E} [f(w^{t})] &- f(w^*) \big) 
      \leq \frac{2}{\tilde{\eta} \step}  \|w^{0} - w^* \|_{\textbf{A}_\sigma}^2 \\
      &+ \sum_{t=0}^T 2 a_t \tilde{\eta} \step \bigg( \big( \frac{1}{\eta_g^2\step} + \frac{1}{\sample \step} \big) \varsigma^2(\sigma) + \frac{ \lip^2 \nu^2_1 d_\sigma}{\sample^2} + \frac{4}{\sample} \bigg( 1 - \frac{\sample}{\client}\bigg) \heteG^2 + 12 \tilde{\eta} \step \smooth \heteG^2  \bigg)
    \end{aligned}
\end{equation*}
Diving by $\sum_{t=0}^T a_t$, we have
\begin{equation*}
    \begin{aligned}
       \mathbb{E} [f(\bar{w}^{T})] - f(w^*)  &\leq \frac{2}{\tilde{\eta} \step \sum_{t=0}^T a_t} \|w^{0} - w^* \|_{\textbf{A}_\sigma}^2\\
       & + 2 \tilde{\eta} \step \bigg( \big( \frac{1}{\eta_g^2 \step} + \frac{1}{\sample \step} \big) \varsigma^2(\sigma) + \frac{ \lip^2 \nu^2_1 d_\sigma}{\sample^2} + \frac{4}{\sample} \bigg( 1 - \frac{\sample}{\client}\bigg) \heteG^2 + 12 \tilde{\eta} \step \smooth \heteG^2  \bigg)
    \end{aligned}
\end{equation*}
Now we consider $\tilde{\eta} \step \sum_{t=0}^T a_t = \tilde{\eta} \step \sum_{t=0}^T q^{-t}$. Since we assume that $T\geq \frac{1}{\mu_\sigma \tilde{\eta} \step}$, we have
\begin{equation*}
    \begin{aligned}
        \tilde{\eta} \step \sum_{t=0}^T a_t = \tilde{\eta} \step q^{-T} \sum_{t=0}^{T} q^t = \tilde{\eta} \step q^{-T} \frac{1-(1-\mu_\sigma \tilde{\eta} \step/2)^{T+1}}{\mu_\sigma \tilde{\eta} \step/2} \geq \frac{2q^{-T}}{3\mu_\sigma}
    \end{aligned}
\end{equation*}
So
\begin{equation*}
    \frac{1}{\tilde{\eta} \step \sum_{t=0}^T a_t} \leq \frac{3}{2} \mu_\sigma q^T =  \frac{3}{2}\mu_\sigma \bigg(1-\frac{\mu_\sigma \tilde{\eta} \step}{2}\bigg)^T \leq \frac{3}{2}\mu_\sigma \exp(-\mu_\sigma \tilde{\eta} \step T /2)
\end{equation*}
In this case, 
\begin{equation}\label{eq-strongly-convex-full-expression}
    \begin{aligned}
       \mathbb{E} [f(\bar{w}^{T})] - f(w^*)  
       &\leq 3\mu_\sigma \exp(-\mu_\sigma \tilde{\eta} \step T/2) \underbrace{\|w^{0} - w^* \|_{\textbf{A}_\sigma}^2}_{D_\sigma}  + 24 \tilde{\eta}^2 \step^2 \smooth \heteG^2  \\
       & \quad + 2 \tilde{\eta} \step \bigg(\underbrace{  \big( \frac{1}{\eta_g^2 \step} + \frac{1}{\sample \step} \big) \varsigma^2(\sigma) + \frac{4}{\sample} \bigg( 1 - \frac{\sample}{\client}\bigg) \heteG^2 + \frac{ \lip^2 \nu^2_1 d_\sigma}{\sample^2}  }_{H_\sigma}\bigg)
    \end{aligned}
\end{equation}
Here we discuss two situations:
\begin{itemize}
    \item If $\frac{1}{\mu_\sigma \step T} \leq \frac{\Lambda_{\min}}{8\step \smooth (1+\heteB^2)} \leq \frac{2\log(\max(e,\mu_\sigma^2 T D_\sigma/H_\sigma)))}{\mu_\sigma \step T} $, we choose $\tilde{\eta}=\frac{\Lambda_{\min}}{8 \step \smooth (1+\heteB^2)}$, then
    \begin{equation*}
    \begin{aligned}
        \mathbb{E} [f(\bar{w}^{T})] - f(w^*) &\leq 3\mu_\sigma D_\sigma \exp\bigg( -\frac{\mu_\sigma \Lambda_{\min} T}{16\smooth(1+\heteB^2)} \bigg)  + \tilde{O}\left(\frac{\smooth \heteG^2}{\mu_\sigma^2 T^2}\right) +\tilde{O}\bigg( \frac{H_\sigma}{\mu_\sigma T}\bigg) \\
        &\leq \tilde{O} \bigg( \frac{H_\sigma}{\mu_\sigma T} \bigg).
    \end{aligned}
    \end{equation*}
    where we use $\tilde{\eta} \step \leq \frac{2\log(\max(e,\mu_\sigma^2 RD_\sigma/H_\sigma)))}{\mu_\sigma T} = \tilde{O}(1/(\mu_\sigma T))$
    \item If $\frac{1}{\mu_\sigma \step T} \leq \frac{2\log(\max(e,\mu_\sigma^2 T D_\sigma/H_\sigma)))}{\mu_\sigma \step T} \leq \frac{1}{8\step \smooth (1+\heteB^2)} $, we choose $\tilde{\eta}=\frac{2\log(\max(e,\mu_\sigma^2 T D_\sigma/H_\sigma)))}{\mu_\sigma \step T}$, then
    \begin{equation*}
    \small
    \begin{aligned}
        \mathbb{E} [f(\bar{w}^{T})] - f(w^*) &\leq 3\mu_\sigma D_\sigma \exp(-\log(\max(e,\mu_\sigma^2 T D_\sigma/H_\sigma))) + \tilde{O}\bigg(\frac{H_\sigma}{\mu_\sigma T} \bigg) + \tilde{O}\left(\frac{\smooth \heteG^2}{\mu_\sigma^2 T^2} \right) \\ 
        &\leq  \tilde{O}\bigg(\frac{H_\sigma}{\mu_\sigma T} \bigg)
    \end{aligned}
    \end{equation*}
\end{itemize}
In this case, we choose $\tilde{\eta} = \min \bigg\{ \frac{2\log(\max(e,\mu_\sigma^2 T D_\sigma/H_\sigma)))}{\mu_\sigma \step T} , \frac{\Lambda_{\min}}{8\step \smooth (1+\heteB^2)} \bigg\}$ \big($T\geq \frac{8\smooth(1+\heteB^2)}{\mu_\sigma \Lambda_{\min}}$\big). Then 
\begin{equation*}
\begin{aligned}
     \mathbb{E} [f(\bar{w}^{T})] - f(w^*) 
     & \leq \tilde{\mathcal{O}} \bigg(\frac{1}{\mu_\sigma T} \bigg(\big( \frac{1}{\eta_g^2 \step} + \frac{1}{\sample \step} \big) \varsigma^2(\sigma) + \frac{4}{\sample} \bigg( 1 - \frac{\sample}{\client}\bigg) \heteG^2 +\frac{ \lip^2 \nu^2_1 d_\sigma}{\sample^2} \bigg) \bigg)
\end{aligned}
\end{equation*}
which completes the proof.
\end{proof}

\subsection{Proof of Theorem~\ref{thm-convex}}

\begin{proof}
We start from Eq. (\ref{eq-convex-intermedia-step}) and set $\mu=0$ for general-convex case:
\begin{equation}
    \begin{aligned}
        f(w^t) - f(w^*) &\leq \frac{2}{\tilde{\eta} \step}\mathbb{E}\|w^{t} - w^* \|_{\textbf{A}_\sigma}^2 - \frac{2}{\tilde{\eta} \step} \mathbb{E}\|w^{t+1} - w^* \|_{\textbf{A}_\sigma}^2 \\
        &\quad + 2\tilde{\eta} \step \bigg( \big( \frac{1}{\eta_g^2 \step} + \frac{1}{\sample \step} \big) \varsigma^2(\sigma) + \frac{ \lip^2 \nu^2_1 d_\sigma}{\sample^2} + \frac{4}{\sample} \bigg( 1 - \frac{\sample}{\client}\bigg) \heteG^2 + 12 \tilde{\eta} \step \smooth \heteG^2  \bigg)
    \end{aligned}
\end{equation}
Summing the above inequality from $t=0$ to $t=T$ and taking average , we have
\begin{equation}\label{eq-convex-full-expression}
\begin{aligned}
    \mathbb{E}\big[f(\bar{w}^t)] - f(w^*)
    &\leq  \frac{2}{ \tilde{\eta}  T \step} \underbrace{\|w^0-w^* \|_{\textbf{A}_\sigma}^2}_{D_\sigma} + \underbrace{24 \smooth \heteG^2}_{Q_1} \tilde{\eta}^2 \step^2\\
    &\quad + 2 \tilde{\eta} \step \bigg( \underbrace{ \big( \frac{1}{\eta_g^2 \step} + \frac{1}{\sample \step} \big) \varsigma^2(\sigma) + \frac{ \lip^2 \nu^2_1 d_\sigma}{\sample^2} + \frac{4}{\sample} \bigg( 1 - \frac{\sample}{\client}\bigg) \heteG^2 }_{H_\sigma} \bigg) 
\end{aligned}
\end{equation}

We set $\tilde{\eta}_{\max}=\frac{\Lambda_{\min}}{8\step \smooth(1+\heteB^2)}$. Here we consider two situations:
\begin{itemize}
    \item If $\tilde{\eta}_{\max}^2 \leq \frac{D_\sigma}{H_\sigma T \step^2}$ and $\tilde{\eta}_{\max}^3 \leq \frac{D_\sigma}{Q_1 T\step^3}$, we set $\tilde{\eta}=\tilde{\eta}_{\max}$, then
    \begin{equation*}
       \mathbb{E}[f(\bar{w}^T)] - f(w^*) 
        \leq \frac{16 \smooth(1+\heteB^2)D_\sigma}{T\Lambda_{\min}} + 2\sqrt{\frac{D_\sigma H_\sigma}{T}} +  \sqrt[3]{\frac{24 D_\sigma^2 \heteG^2 \smooth}{T^2}}
    \end{equation*}
    \item If $\tilde{\eta}_{\max}^2 \geq \frac{D_\sigma}{H_\sigma T \step^2}$ or $\tilde{\eta}_{\max}^3 \geq \frac{D_\sigma}{Q_1 T \step^3}$, we set $\tilde{\eta}=\min \bigg\{\sqrt{\frac{D_\sigma}{H_\sigma T\step^2}}, \sqrt[3]{\frac{D_\sigma}{Q_1 T \step^3}} \bigg\}$, then
    \begin{equation*}
        \mathbb{E}[f(\bar{w}^T)] - f(w^*) 
        \leq 4\sqrt{\frac{D_\sigma H_\sigma}{T}} + \sqrt[3]{\frac{24D_\sigma^2 \heteG^2 \smooth}{T^2}}
    \end{equation*}
\end{itemize}
In conclusion, if we set $\tilde{\eta}=\min \bigg\{\sqrt{\frac{D_\sigma}{H_\sigma T\step^2}}, \sqrt[3]{\frac{D_\sigma}{Q_1 T\step^3}}, \frac{1}{8\step \smooth(1+\heteB^2)} \bigg\}$, we have
\begin{equation*}
    \mathbb{E}[f(\bar{w}^T)] - f(w^*) 
    \leq \frac{16 \smooth(1+\heteB^2)(1+4\sigma)D_\sigma}{T} + 4\sqrt{\frac{D_\sigma H_\sigma}{T}} +  \sqrt[3]{\frac{24 D_\sigma^2 \heteG^2 \smooth}{T^2}}
\end{equation*}
which completes the proof.
\end{proof}

\subsection{Proof of Theorem \ref{thm-non-convex}}

\begin{proof}[Proof of Theorem~\ref{thm-non-convex}]
According to the smoothness of $f$, we have
\begin{equation*}
    \begin{aligned}
        f(w^{t+1}) 
        &\leq f(w^{t}) + \langle \nabla f(w^t), w^{t+1} - w^t \rangle + \frac{\smooth}{2} \| w^{t+1} - w^t \|_2^2 \\
        &\leq f(w^{t}) - \bigg\langle \nabla f(w^t),  \frac{\tilde{\eta}}{\sample} \sum_{j=1}^{\sample}\sum_{i=1}^\step \textbf{A}_\sigma^{-1} g_j(w_j^{t,i})\bigg\rangle + \bigg\langle \nabla f(w^t),  \textbf{A}_\sigma^{-1} \eta_g \frac{\textbf{n}(\clip)}{\sample} \bigg\rangle \\
        &\quad+  \frac{\smooth}{2} \bigg( \bigg\|  \frac{\tilde{\eta}}{\sample} \sum_{j=1}^{\sample}\sum_{i=1}^\step  \textbf{A}_\sigma^{-1} g_j(w_j^{t,i})  \bigg\|_2^2 + \eta_g^2 \frac{\|\textbf{A}_\sigma^{-1} \textbf{n}(\clip) \|_2^2}{\sample^2} \\
        & \quad + 2\bigg\langle \frac{\tilde{\eta}}{\sample} \sum_{j=1}^{\sample}\sum_{i=1}^\step \textbf{A}_\sigma^{-1} g_j(w_j^{t,i}) ,\textbf{A}_\sigma^{-1} \eta_g \frac{\textbf{n}(\clip)}{\sample} \bigg\rangle\bigg)
    \end{aligned}
\end{equation*}
By taking the expectation on both sides, we have
\begin{equation*}
    \begin{aligned}
        \mathbb{E}[f(w^{t+1})] 
        &\leq f(w^t) \underbrace{-\frac{\tilde{\eta}}{\client} \sum_{j=1}^{\client}\sum_{i=1}^\step  \langle \nabla f(w^t), \nabla f_j(w_j^{t,i})\rangle_{\textbf{A}_\sigma^{-1}}}_{A_1} + \underbrace{\frac{\smooth \eta_g^2}{2\sample^2} \mathbb{E}\|\textbf{A}_\sigma^{-1} \textbf{n}(\clip) \|_2^2}_{A_2} \\
        &\quad + \underbrace{\frac{\tilde{\eta}^2\smooth}{2} \mathbb{E} \bigg\|  \frac{1}{\sample} \sum_{j=1}^{\sample}\sum_{i=1}^\step  g_j(w_j^{t,i})  \bigg\|_{\textbf{A}_\sigma^{-1}}^2}_{A_3}
    \end{aligned}
\end{equation*}
According to Eq~(\ref{eq_A3}), (\ref{eq_B1}) and (\ref{eq_B2}), we have
\begin{equation}
\begin{aligned}
    A_3 &\leq  2\tilde{\eta}^2 \step^2 \smooth(1+\heteB^2) \|\nabla f(w^t)\|_2^2 + 2\tilde{\eta}^2\step^2 \smooth \bigg(1-\frac{\sample}{\client}\bigg) \frac{1}{\sample} \heteG^2 \\
    & \quad + \frac{\tilde{\eta}^2 \step \smooth^3}{\client} \sum_{j=1}^\client \sum_{i=1}^\step \mathbb{E}\| w_j^{t,i} - w^t \|_{\textbf{A}_\sigma^{-1}}^2 + \frac{\tilde{\eta}^2 \step \smooth}{2\sample}\varsigma^2(\sigma)
\end{aligned}
\end{equation}

As for $A_1$, we apply the inequality $ab=\frac{1}{2}[(b-a)^2 -a^2] - \frac{1}{2}b^2 \geq \frac{1}{2} [a^2 - (b-a)^2]$, we have 
\begin{equation*}
    \begin{aligned}
        A_1 &\leq -\frac{\tilde{\eta}}{2 \client} \sum_{j=1}^\client \sum_{i=1}^\step \bigg[ \|\nabla f(w^t) \|_{\textbf{A}_\sigma^{-1}}^2 -\| \nabla f_j(w_j^{t,i}) - \nabla f(w^t) \|_{\textbf{A}_\sigma^{-1}}^2 \bigg] \\
        &\leq - \frac{\tilde{\eta}\step}{2} \| \nabla f(w^t) \|_{\textbf{A}_\sigma^{-1}}^2 + \frac{\tilde{\eta} \smooth^2}{2\client}\sum_{j=1}^\client \sum_{i=1}^\step \|w_j^{t,i} - w^t \|_{2}^2
    \end{aligned}
\end{equation*}

According to Lemma~\ref{lemma-ls-noise-reduction-2}, we have
\begin{equation*}
    A_2 \leq \frac{\tilde{\eta}^2 \step^2 \lip^2  \smooth \nu^2_1 \tilde{d}_\sigma}{2\sample^2}
\end{equation*}

Summing up $A_1$, $A_2$ and $A_3$, and using the inequality $\eta_l \leq \frac{\Lambda_{\min}}{8\step \smooth \eta_g (\heteB^2+1)}$, where $\frac{1}{1+4\sigma}\leq \Lambda_{\min} \leq 1$ is the smallest eigenvalue of $\textbf{A}_\sigma^{-1}$, we have
\begin{equation*}
    \begin{aligned}
        \mathbb{E}[f(w^{t+1})] 
        &\leq f(w^t) - \tilde{\eta} \step \bigg(\frac{1}{2} -\frac{ 2\tilde{\eta} \step \smooth (1+\heteB^2)}{\Lambda_{\min}} \bigg) \| \nabla f(w^t) \|_{\textbf{A}_\sigma^{-1}}^2 + \frac{\tilde{\eta}^2 \step \smooth}{2\sample}\varsigma^2(\sigma) \\
        & \quad + \frac{\tilde{\eta}^2 \step^2 \lip^2  \smooth \nu^2_1 \tilde{d}_\sigma}{2\sample^2}  + 2\tilde{\eta}^2\step^2 \smooth \bigg(1-\frac{\sample}{\client}\bigg) \frac{1}{\sample} \heteG^2 \\
        &\quad + \tilde{\eta} \smooth^2\bigg( \frac{1}{2} + \tilde{\eta} \step \smooth  \bigg) \frac{1}{\client} \sum_{j=1}^\client \sum_{i=1}^\step \|w_j^{t,i} - w^t \|_{2}^2\\
        &\leq f(w^t) - \frac{\tilde{\eta} \step}{4}  \| \nabla f(w^t) \|_{\textbf{A}_\sigma^{-1}}^2 + \frac{\tilde{\eta}^2 \step \smooth}{2\sample}\varsigma^2(\sigma) + \frac{\tilde{\eta}^2 \step^2 \lip^2  \smooth \nu^2_1 \tilde{d}_\sigma}{2\sample^2}\\
        &\quad + 2\tilde{\eta}^2\step^2 \smooth \bigg(1-\frac{\sample}{\client}\bigg) \frac{1}{\sample} \heteG^2 + \tilde{\eta} \smooth^2 \frac{1}{\client} \sum_{j=1}^\client \sum_{i=1}^\step \|w_j^{t,i} - w^t \|_{2}^2
    \end{aligned}
\end{equation*}
And according to Lemma~\ref{lemma-bound-local-divergence-2} and the assumption that $\eta_g\geq 1$ ($\tilde{\eta} = \eta_l \eta_g \geq \eta_l$) and $\eta_l \leq \frac{\Lambda_{\min}}{8\step \smooth \eta_g (\heteB^2+1)}$, we have
\begin{equation}
    \begin{aligned}
        \tilde{\eta} \smooth^2 \frac{1}{\client}\sum_{j=1}^\client \sum_{i=1}^\step \|w_j^{t,i} - w^t \|_{2}^2 
        &\leq 4 \step^3\tilde{\eta} \eta_l^2 \smooth^2 \heteG^2 +  4 \step^3\tilde{\eta} \eta_l^2 \smooth^2 \heteB^2 \|\nabla f(w^t)\|_2^2 + 2\step^2 \tilde{\eta} \frac{\eta_l^2}{\Lambda_{\min}} \smooth^2 \varsigma^2(\sigma)\\
        &\leq 4 \step^3\tilde{\eta}^3 \smooth^2 \heteG^2 +  \frac{4 \step^3\tilde{\eta}^3 \smooth^2 \heteB^2}{ \Lambda_{\min}} \|\nabla f(w^t)\|_{\textbf{A}_\sigma^{-1}}^2 + 2\step^2 \tilde{\eta} \frac{\eta_l^2}{\Lambda_{\min}}\smooth^2 \varsigma^2(\sigma)\\
        &\leq 4 \step^3\tilde{\eta}^3 \smooth^2 \heteG^2 +  \frac{\tilde{\eta} \step}{16} \|\nabla f(w^t)\|_{\textbf{A}_\sigma^{-1}}^2 + \frac{\tilde{\eta}^2 \step \smooth}{2\eta_g^2} \varsigma^2(\sigma)
    \end{aligned}
\end{equation}

In this case, we have
\begin{equation*}
    \begin{aligned}
        \mathbb{E}[f(w^{t+1})] 
        &\leq  f(w^t) - \frac{\tilde{\eta} \step}{8}  \| \nabla f(w^t) \|_{\textbf{A}_\sigma^{-1}}^2 + \frac{\tilde{\eta}^2 \step^2 \lip^2  \smooth \nu^2_1 \tilde{d}_\sigma}{2\sample^2} + 2\tilde{\eta}^2\step^2 \smooth \bigg(1-\frac{\sample}{\client}\bigg) \frac{1}{\sample} \heteG^2 \\
        &\quad  + \frac{\tilde{\eta}^2 \step^2 \smooth}{2} \bigg( \frac{1}{\step} + \frac{1}{\sample \step} \bigg)\varsigma^2(\sigma) + 4 \step^3\tilde{\eta}^3 \smooth^2 \heteG^2 \\
        &\leq  f(w^t) - \frac{\tilde{\eta} \step}{8}  \| \nabla f(w^t) \|_{\textbf{A}_\sigma^{-1}}^2 + 4 \step^3\tilde{\eta}^3 \smooth^2 \heteG^2 \\
        &\quad + \frac{\tilde{\eta}^2 \step^2 \smooth}{2} \bigg( \big(\frac{1}{\eta_g^2 \step} + \frac{1}{\sample \step}\big) \varsigma^2(\sigma) + \frac{4}{\sample} \bigg(1-\frac{\sample}{\client} \bigg) \heteG^2 + \frac{ \lip^2  \nu^2_1 \tilde{d}_\sigma}{\sample^2}\bigg)
    \end{aligned}
\end{equation*}

Summing the above inequality from $t=0$ to $t=T$ and taking average , we have
\begin{equation}\label{eq-non-convex-full-expression}
    \begin{aligned}
        \mathbb{E}\| \nabla f(\bar{w}^t) \|_{\textbf{A}_\sigma^{-1}}^2 &\leq \frac{8}{\tilde{\eta} \step T} \big(\underbrace{f(w^0) - f(w^*)}_{F_0}\big)  + \underbrace{32  \smooth^2 \heteG^2}_{Q_2} \step^2 \tilde{\eta}^2 \\
        &\quad  + 4 \tilde{\eta} \step \smooth \bigg( \underbrace{\big(\frac{1}{\eta_g^2 \step} + \frac{1}{\sample \step}\big) \varsigma^2(\sigma) + \frac{4}{\sample} \bigg(1-\frac{\sample}{\client} \bigg) \heteG^2 + \frac{ \lip^2  \nu^2_1 \tilde{d}_\sigma}{\sample^2}}_{\tilde{H}_\sigma}\bigg)
    \end{aligned}
\end{equation}

We set $\eta_{\max} = \frac{\Lambda_{\min}}{8\step \smooth(1+\heteB^2)}$. Here we consider two situations:
\begin{itemize}
    \item If $\tilde{\eta}_{\max}^2 \leq \frac{F_0}{\tilde{H}_\sigma \smooth T \step^2}$ and $\tilde{\eta}_{\max}^3\leq \frac{F_0}{Q_2 T \step^3}$, we set $\tilde{\eta}=\tilde{\eta}_{\max}$, then
    \begin{equation*}
        \mathbb{E} \| \nabla f(\bar{w}^T) \|_{\textbf{A}_\sigma^{-1}}^2 
        \leq \frac{64 \smooth (1+\heteB^2)(1+4\sigma) F_0}{T} + 4 \sqrt{\frac{F_0 \tilde{H}_\sigma \smooth}{T}} + \sqrt[3]{\frac{32 F_0^2 \heteG^2 \smooth^2}{T^2}}
    \end{equation*}
    \item If $\tilde{\eta}_{\max}^2 \geq \frac{F_0}{H_\sigma \smooth T \step^2}$ or $\tilde{\eta}_{\max}^3\geq \frac{F_0}{Q_2 T \step^3}$, we set $\tilde{\eta} = \min\{\sqrt{\frac{F_0}{\tilde{H}_\sigma \smooth T \step^2}}, \sqrt[3]{\frac{F_0}{Q_2 T \step^3}}\}$, then
    \begin{equation*}
        \mathbb{E} \| \nabla f(\bar{w}^T) \|_{\textbf{A}_\sigma^{-1}}^2 
        \leq  12 \sqrt{\frac{F_0 \tilde{H}_\sigma \smooth}{T}} + \sqrt[3]{\frac{32 F_0^2 G^2 \smooth^2}{T^2}}
    \end{equation*}
\end{itemize}

In conclusion, if we set $\tilde{\eta} = \min\{ \sqrt{\frac{F_0}{\tilde{H}_\sigma T \smooth \step^2}}, \sqrt[3]{\frac{F_0}{Q_2 T \step^3}}, \frac{\Lambda_{\min}}{8 \step \smooth (1+\heteB^2)}\}$, we have
\begin{equation*}
    \mathbb{E} \| \nabla f(\bar{w}^T) \|_{\textbf{A}_\sigma^{-1}}^2 
    \leq \frac{64 \smooth (1+\heteB^2)(1+4\sigma) F_0}{T} + 12 \sqrt{\frac{F_0 \tilde{H}_\sigma \smooth}{T}} + \sqrt[3]{\frac{32 F_0^2 \heteG^2 \smooth^2}{T^2}},
\end{equation*}
which completes the proof.
\end{proof}

\subsection{Proof of Corollary~\ref{corollary-convergence-dp-uniform}}
\begin{proof}[Proof of Corollary~\ref{corollary-convergence-dp-uniform}]
We assume $\log(1/\delta) \geq \varepsilon$, then applying Theorem \ref{Theorem-privacy-guarantee-federated-uniform} with $\nu_1$, and $\tau=\frac{\sample}{\client}$, we have
\begin{align*}\label{eq-plug-in-nu}
       \frac{\lip^2 d_\sigma }{\sample^2} \cdot \nu^2_1
        &= \frac{\lip^2 d_\sigma }{\sample^2} \cdot \frac{\tau^2 }{\varepsilon^2} \frac{14 T}{\lambda} \bigg( \frac{\log(1/\delta)}{1-\lambda} + \varepsilon \bigg), && \text{by Theorem \ref{Theorem-privacy-guarantee-federated-uniform}} \\  
        &\leq \frac{\lip^2 d_\sigma }{\sample^2} \cdot \frac{\tau^2 }{\varepsilon^2} \frac{14T}{\lambda} \bigg( \frac{\log(1/\delta)}{1-\lambda} + \log(1/\delta) \bigg), && \text{by assumption $\log(1/\delta) \geq \varepsilon$} \\
        &\leq \frac{\lip^2 d_\sigma }{\sample^2} \cdot \frac{\tau^2 T\log(1/\delta)}{\varepsilon^2} \cdot \underbrace{\frac{14}{\lambda} \bigg( \frac{1}{1-\lambda} + 1 \bigg)}_{C_0} \\
        & = \frac{C_0\lip^2 d_\sigma T \log(1/\delta)}{\varepsilon^2\client^2} 
\end{align*}

\begin{itemize}
    \item \textbf{$\mu$ Strongly-Convex:} 
    following the proof of Theorem~\ref{thm-strongly-convex}, we have
    \begin{equation*}
        \begin{aligned}
            \mathbb{E} [f(\bar{w}^{T})] - f(w^*) 
            &\leq \tilde{\mathcal{O}} \bigg(\frac{1}{\mu_\sigma T} \bigg(\frac{(1+\frac{\sample}{\eta_g^2})\varsigma^2(\sigma)}{\sample \step}  + \frac{4 ( 1 - \tau) }{\sample} \heteG^2 + \frac{d_\sigma C_0 \lip^2 T \log(1/\delta)}{\varepsilon^2 \client^2} \bigg) \bigg) \\
        \end{aligned}
    \end{equation*}
    If we select $T=\frac{\varepsilon^2 \client^2}{C_0 \lip^2 \sample \log(1/\delta)}$ with $T \geq \frac{1}{\mu_\sigma \tilde{\eta} \step}$ where {\footnotesize{$\tilde{\eta} = \min \bigg\{ \frac{2\log(\max(e,\mu_\sigma^2 T D_\sigma/H_\sigma)))}{\mu_\sigma \step T} , \frac{\Lambda_{\min}}{8\step \smooth (1+\heteB^2)} \bigg\}$}}, and assume $\eta_g \geq \sqrt{\sample}$, then we have
    \begin{equation*}
        \begin{aligned}
            \mathbb{E} [f(\bar{w}^{T})] - f(w^*) 
            & \leq \tilde{\mathcal{O}} \bigg(\frac{ (d_\sigma + \varsigma^2(\sigma)/\step + (1-\tau) \heteG^2) \lip^2 \log(1/\delta)}{\mu_\sigma \varepsilon^2 \client^2} \bigg).
        \end{aligned}
    \end{equation*}
    \item \textbf{General-Convex:} Following Theorem~\ref{thm-convex}, we have 
    \begin{equation*}
    \begin{aligned}
        \mathbb{E}[f(\bar{w}^T)] - f(w^*) 
        &\leq \mathcal{O} \bigg( \sqrt{\frac{D_\sigma H_\sigma}{T}} \bigg) 
        = \mathcal{O} \bigg( \sqrt{\frac{D_\sigma}{T} \bigg( \frac{(1+\frac{\sample}{\eta_g^2}) \varsigma^2(\sigma)}{\sample \step} + \frac{4}{\sample} ( 1 - \tau ) \heteG^2  + \frac{ \lip^2 \nu^2_1 d_\sigma}{\sample^2}\bigg)} \\
        &= \mathcal{O} \bigg( \sqrt{\frac{(1+\frac{\sample}{\eta_g^2}) \varsigma^2(\sigma)}{\sample \step} + \frac{4}{\sample} ( 1 - \tau ) \heteG^2 +  \frac{d_\sigma C_0 \lip^2 T \log(1/\delta)}{\varepsilon^2 \client^2} \bigg)} \\
    \end{aligned}
    \end{equation*}
    If we set $T = \frac{\varepsilon^2 \client^2}{C_0 \lip^2 \sample \log(1/\delta)}$ and assume that $\eta_g\geq \sqrt{\sample}$, then we have
    \begin{equation*} 
        \mathbb{E}[f(\bar{w}^T)] - f(w^*) 
        \leq \frac{\sqrt{(\varsigma^2(\sigma)/\step+ 4(1-\tau)\heteG^2 + d_\sigma)D_\sigma \lip^2 \log(1/\delta)}} {\varepsilon \client}.
    \end{equation*}

    \item \textbf{Non-Convex:} Following Theorem~\ref{thm-non-convex}, we have 
    \begin{equation*}
    \begin{aligned}
        \mathbb{E} \| \nabla f(\bar{w}^T) \|_{\textbf{A}_\sigma^{-1}}^2 
        &\leq \mathcal{O} \bigg( \sqrt{\frac{F_0 H_\sigma \smooth}{T}} \bigg)
        = \mathcal{O} \bigg( \sqrt{\frac{F_0 \smooth}{T} \bigg(\frac{(1+\frac{\sample}{\eta_g^2})\varsigma^2(\sigma)}{\sample \step} \! +\! \frac{4}{\sample} (1-\tau) \heteG^2 \!+\! \frac{ \lip^2  \nu^2_1 \tilde{d}_\sigma}{\sample^2} \bigg)} \bigg) \\
        &= \mathcal{O} \bigg( \sqrt{\frac{F_0 \smooth}{T} \bigg(\frac{(1+\frac{\sample}{\eta_g^2})\varsigma^2(\sigma)}{\sample \step}  + \frac{4}{\sample} (1-\tau) \heteG^2 + \frac{\tilde{d}_\sigma C_0 \lip^2 T \log(1/\delta)}{\varepsilon^2 \client^2}  \bigg)} \bigg)
    \end{aligned}
    \end{equation*}
    If we set $T = \frac{\varepsilon^2 \client^2}{C_0 \lip^2 \sample \log(1/\delta)}$ and assume that $\eta_g \geq \sqrt{\sample}$, then we have
    \begin{equation*}
        \mathbb{E} \| \nabla f(\bar{w}^T) \|_{\textbf{A}_\sigma^{-1}}^2
        \leq \frac{\sqrt{(\varsigma^2(\sigma)/\step + 4(1-\tau) \heteG^2 + \tilde{d}_\sigma)F_0  \smooth \lip^2 \log(1/\delta)}} {\varepsilon \client}.
    \end{equation*}
\end{itemize}
which completes the proof.
\end{proof}

\section{Comparison of Theorem~\ref{Theorem-privacy-guarantee-federated-uniform} and \ref{Theorem-privacy-guarantee-federated-poisson} with Moment Accountants}\label{sec:comparison}

In this section, we show that our bounds provided in Theorem~\ref{Theorem-privacy-guarantee-federated-uniform} and Theorem~\ref{Theorem-privacy-guarantee-federated-poisson} are tight by comparing them with the numerical moment accountants in \cite{wang2019subsampled}
and \cite{zhu2019poission, Mironov2019sampled} respectively. We consider two settings where $T=30$, $\tau=0.05$, $\client=500$ and $T=200$, $\tau=0.05$, $\client=2000$, which we uses for the experiment over MNIST and SVHN respectively. Firstly, one thing we need to notice is that, in Theorem~\ref{Theorem-privacy-guarantee-federated-uniform} and \ref{Theorem-privacy-guarantee-federated-poisson}, noise level $\nu$ is in nearly inverse proportion to $\varepsilon$ when $\varepsilon$ is small, where the first term under the square root in Eq. \ref{eq-nu-uniform} and Eq. \ref{eq-nu-poisson} become the major term. However, when $\varepsilon$ is relatively large, like settings we use in our experiment, this relation changes. The slopes of the curves lie in $[-1,-1/2]$, at similar rates. Note that when we apply Theorem~\ref{Theorem-privacy-guarantee-federated-uniform} and \ref{Theorem-privacy-guarantee-federated-poisson}, we will firstly select $\lambda$ satisfying all the proposed conditions by line search. Then we choose the one minimizing the lower bound of $\nu$.

Figure~\ref{comparison_thm_accountant} a) and b) compare Theorem~\ref{Theorem-privacy-guarantee-federated-uniform} with accountant in \cite{wang2019subsampled} under the two settings above. We can notice that the two curves are almost parallel when $\varepsilon$ is relatively large. For Theorem~\ref{Theorem-privacy-guarantee-federated-poisson} and accountant in \cite{zhu2019poission, Mironov2019sampled}, (Figure~\ref{comparison_thm_accountant} d) and e)), we can notice that their curves are getting close when $\varepsilon$ becomes large. If we further choose a large $T=1000$ ($\tau=0.05$ and $\client=2000$), these observations are more obvious, which is shown in Figure~\ref{comparison_thm_accountant} c) and f).
It demonstrate that our closed-form bounds are tight and only differ from numerical moment accountant by a constant.

\begin{figure}[htbp]
    \centering
    \subfigure[]{\includegraphics[width=1.7in]{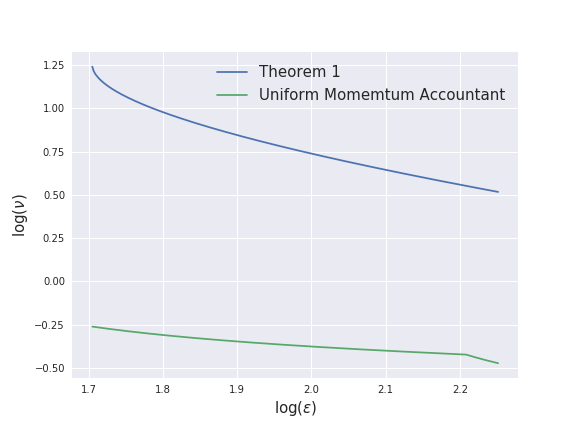}}
    \subfigure[]{\includegraphics[width=1.7in]{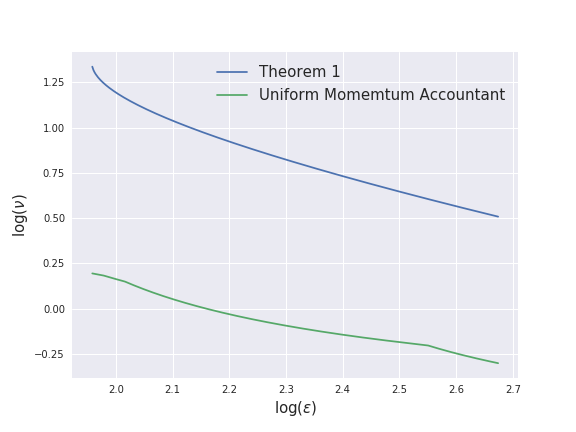}}
    \subfigure[]{\includegraphics[width=1.7in]{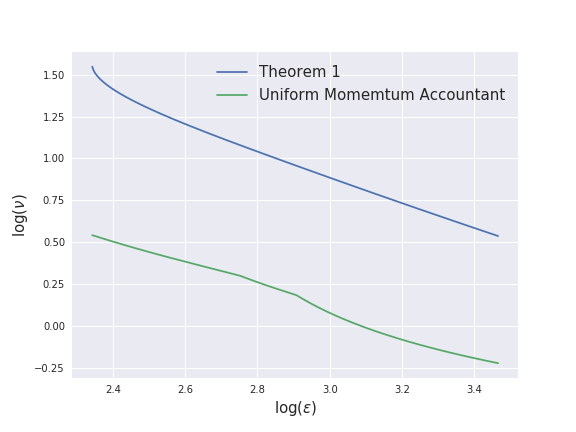}} \\
    \subfigure[]{\includegraphics[width=1.7in]{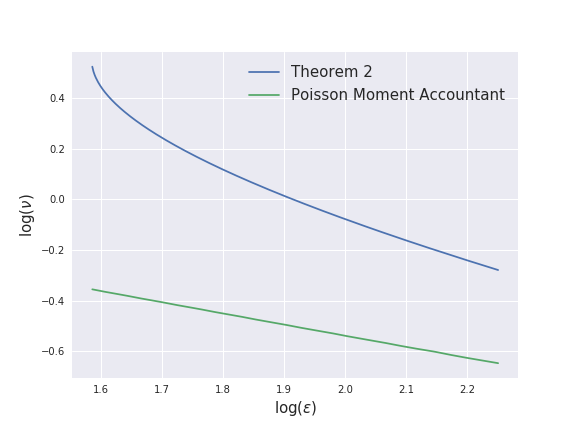}}
    \subfigure[]{\includegraphics[width=1.7in]{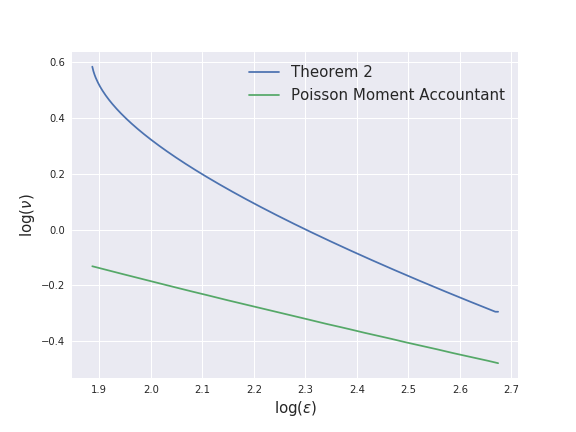}}
    \subfigure[]{\includegraphics[width=1.6in]{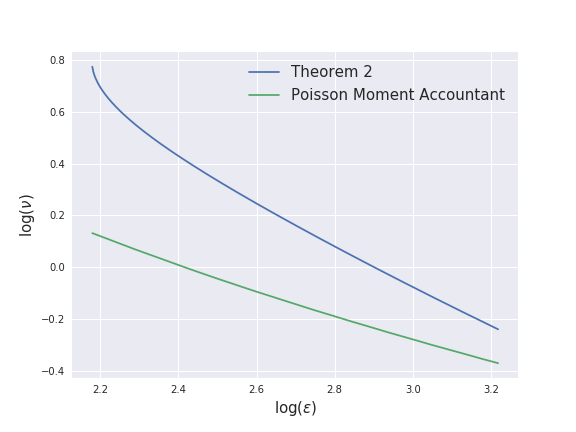}}
    \caption{Comparison of Theorem~\ref{Theorem-privacy-guarantee-federated-uniform} and \ref{Theorem-privacy-guarantee-federated-poisson} with uniform \cite{wang2019subsampled} and Poisson moment accountants \cite{zhu2019poission, Mironov2019sampled}. We can observe that Theorem~\ref{Theorem-privacy-guarantee-federated-uniform} is nearly parallel to moment accountant \cite{wang2019subsampled} and and \ref{Theorem-privacy-guarantee-federated-poisson} is close to moment accountant in \cite{zhu2019poission, Mironov2019sampled} when $\varepsilon$ becomes large. For example, in c), the slopes of least square regression for Theorem~\ref{Theorem-privacy-guarantee-federated-uniform} and moment accountant are -0.80 and -0.73 respectively, while the intercepts are 3.31 and 2.29. It shows that the theoretical bound are of similar rates as numerical moment accountants and differ from moment accountants only by a small constant $e^{3.31-2.29}=$.}
    \label{comparison_thm_accountant}
\end{figure}

\section{Training Curves}\label{sec:curves}

In Figure~\ref{train_mnist_loss_acc}, \ref{train_svhn_loss_acc}, \ref{train_lstm_loss_acc} and \ref{mi_svhn_loss_acc}, we show the training curves of experiment in Section 4, including training loss, validation loss, training accuracy and validation accuracy.

\begin{figure}[htbp]
    \centering
    \subfigure[]{\includegraphics[width=1.5in]{imgs/logistic/loss_train_uniform_mnist.png}}
    \subfigure[]{\includegraphics[width=1.5in]{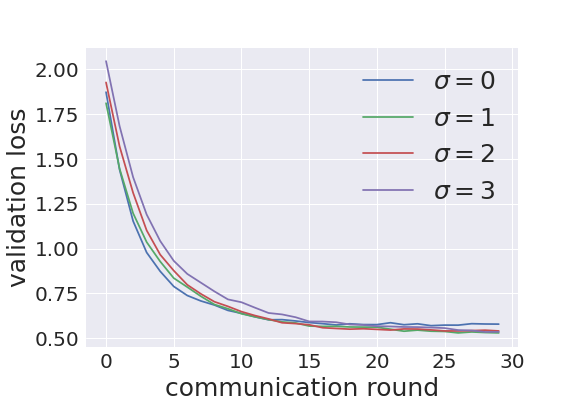}} 
    \subfigure[]{\includegraphics[width=1.5in]{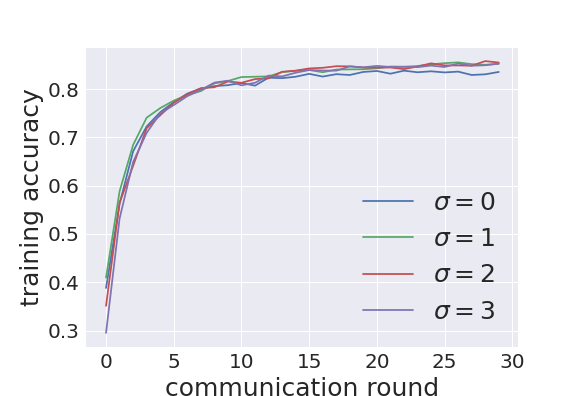}}
    \subfigure[]{\includegraphics[width=1.5in]{imgs/logistic/acc_val_uniform_mnist.png}} \\
    \subfigure[]{\includegraphics[width=1.5in]{imgs/logistic/loss_train_poisson_mnist.png}}
    \subfigure[]{\includegraphics[width=1.5in]{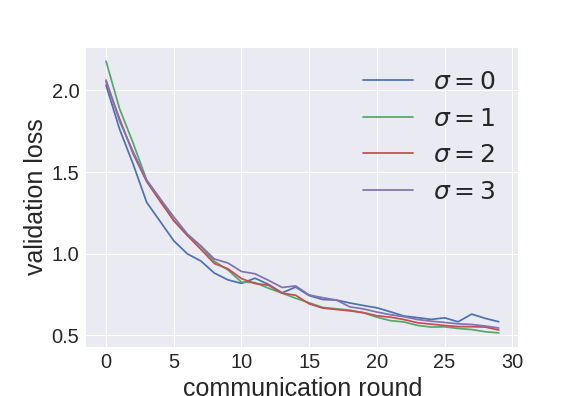}} 
    \subfigure[]{\includegraphics[width=1.5in]{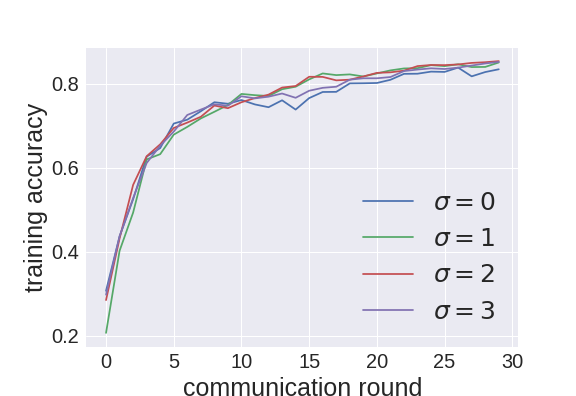}}
    \subfigure[]{\includegraphics[width=1.5in]{imgs/logistic/acc_val_poisson_mnist.png}}
    \caption{Training curves of logistic regression on MNIST (Section \ref{sec:logistic}) with DP-Fed ($\sigma=0$), DP-Fed-LS ($\sigma=1, 2, 3$). (a),(b),(c) and (d): training loss, validation loss, training accuracy and  validation accuracy with uniform subsampling and $(7,1/1000^{1.1)}$-DP.  (c), (d), (e) and (f): training loss, validation loss, training accuracy and validation accuracy with Poisson subsampling and $(7,1/500^{1.1})$-DP.}
    \label{train_mnist_loss_acc}
\end{figure}

\begin{figure}[htbp]
    \centering
    \subfigure[]{\includegraphics[width=1.5in]{imgs/svhn/loss_train_uniform_svhn_eps_5_23.png}}
    \subfigure[]{\includegraphics[width=1.5in]{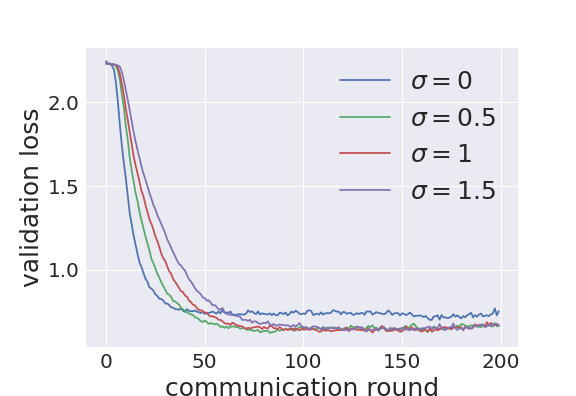}} 
    \subfigure[]{\includegraphics[width=1.5in]{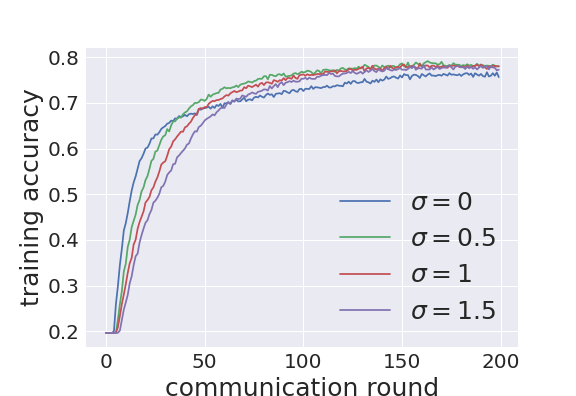}}
    \subfigure[]{\includegraphics[width=1.5in]{imgs/svhn/acc_val_uniform_svhn_eps_5_23.png}}  \\
    \subfigure[]{\includegraphics[width=1.5in]{imgs/svhn/loss_train_poisson_svhn_eps_5_07.png}}
    \subfigure[]{\includegraphics[width=1.5in]{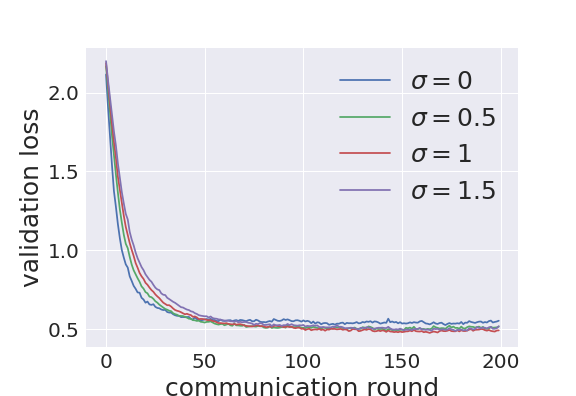}} 
    \subfigure[]{\includegraphics[width=1.5in]{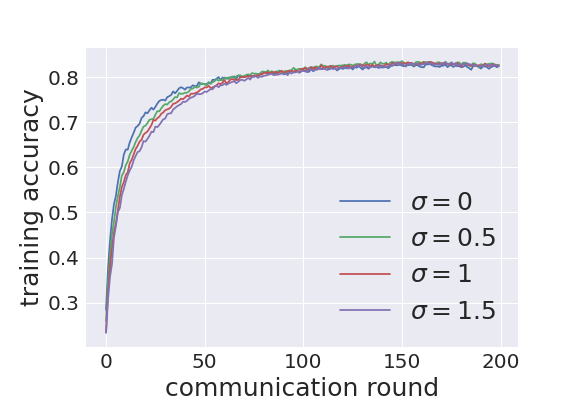}}
    \subfigure[]{\includegraphics[width=1.5in]{imgs/svhn/acc_val_poisson_svhn_eps_5_07.png}}
    \caption{Training curves of CNN on SVHN (Section \ref{sec:cnn}) with DP-Fed ($\sigma=0$), DP-Fed-LS ($\sigma=0.5, 1, 1.5$). (a), (b), (c) and (d): training loss, validation loss, training accuracy and validation accuracy with uniform subsampling, where $(5.23, 1/2000^{1.1})$-DP is applied. (c), (d), (e) and (f): training loss, validation loss and training accuracy and validation accuracy with Poisson subsampling, where $(5.07, 1/2000^{1.1})$-DP is applied.}
    \label{train_svhn_loss_acc}
\end{figure}

\begin{figure}[htbp]
    \centering
    \subfigure[]{\includegraphics[width=1.5in]{imgs/lstm/loss_train_uniform_shakespeare_eps_27_24.png}}
    \subfigure[]{\includegraphics[width=1.5in]{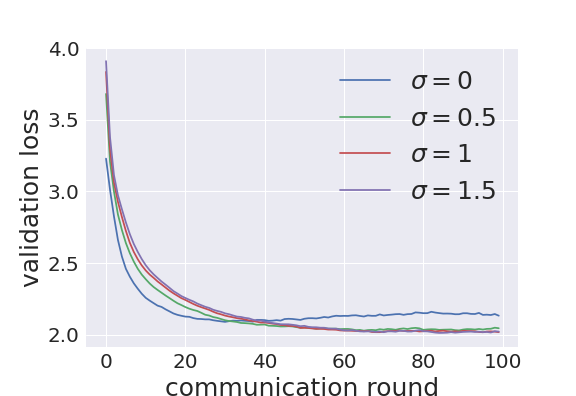}}
    \subfigure[]{\includegraphics[width=1.5in]{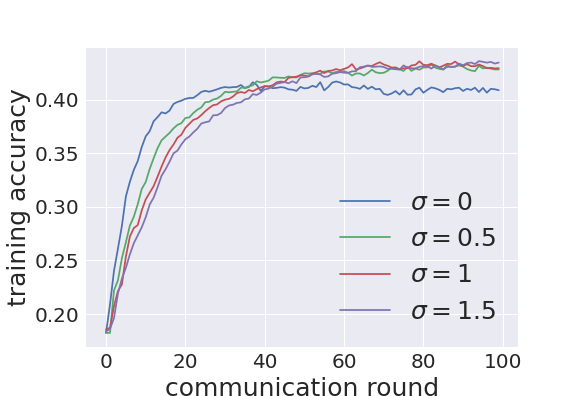}}
    \subfigure[]{\includegraphics[width=1.5in]{imgs/lstm/acc_val_uniform_shakespeare_eps_27_24.png}} \\
    \subfigure[]{\includegraphics[width=1.5in]{imgs/lstm/loss_train_poisson_shakespeare_eps_14_04.png}}
    \subfigure[]{\includegraphics[width=1.5in]{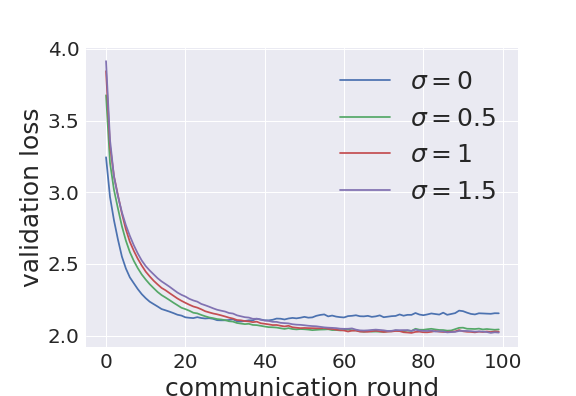}}
    \subfigure[]{\includegraphics[width=1.5in]{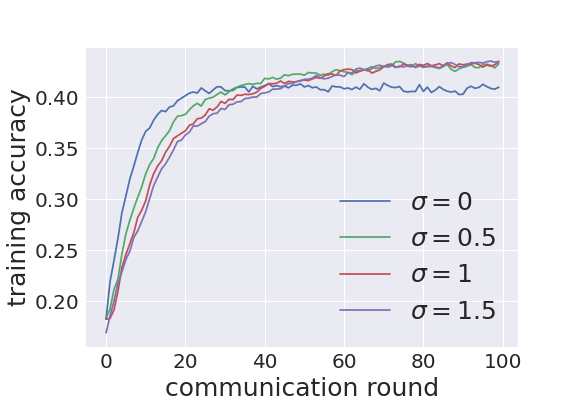}}
    \subfigure[]{\includegraphics[width=1.5in]{imgs/lstm/acc_val_poisson_shakespeare_eps_14_04.png}} 
    \caption{Training curves of LSTM on Shakespeare (Section \ref{sec:lstm}) with DP-Fed ($\sigma=0$), DP-Fed-LS ($\sigma=0.5, 1, 1.5$). (a), (b), (c) and (d): training loss, validation loss, training accuracy and validation accuracy with uniform subsampling, where $(27.24, 1/975^{1.1})$-DP is applied. (c), (d), (e) and (f): training loss, validation loss, training accuracy and validation accuracy with Poisson subsampling, where $(14.04, 1/975^{1.1})$-DP is applied.}
    \label{train_lstm_loss_acc}
\end{figure}

\begin{figure}[htbp]
    \centering
    \subfigure[]{\includegraphics[width=1.50in]{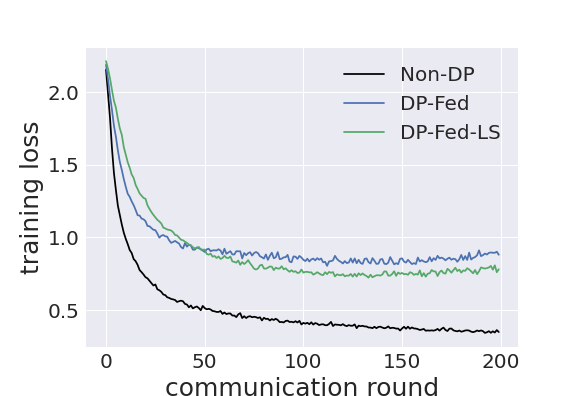}}
    \subfigure[]{\includegraphics[width=1.50in]{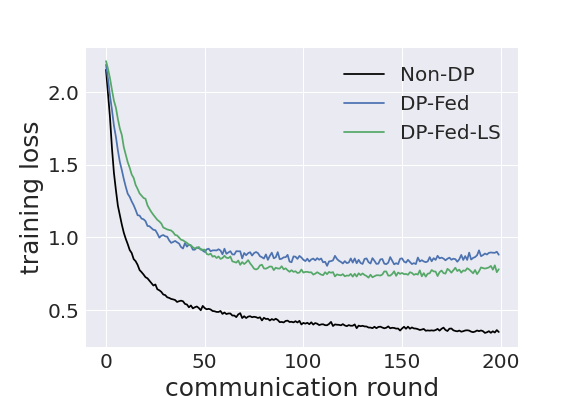}} 
    \subfigure[]{\includegraphics[width=1.50in]{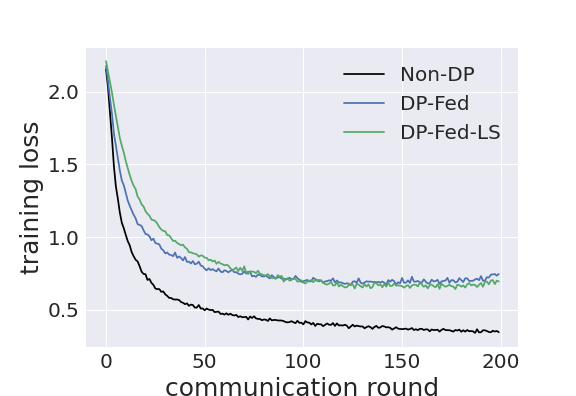}}
    \subfigure[]{\includegraphics[width=1.50in]{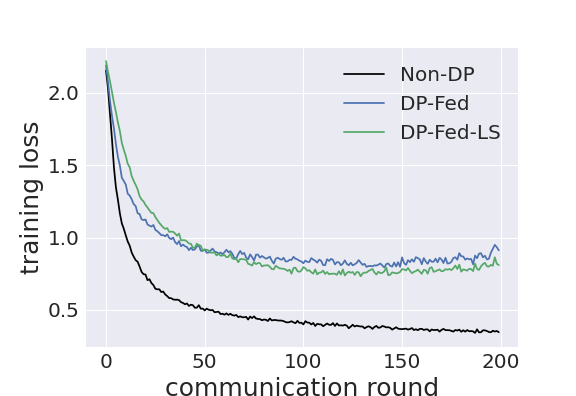}} \\
    \subfigure[]{\includegraphics[width=1.50in]{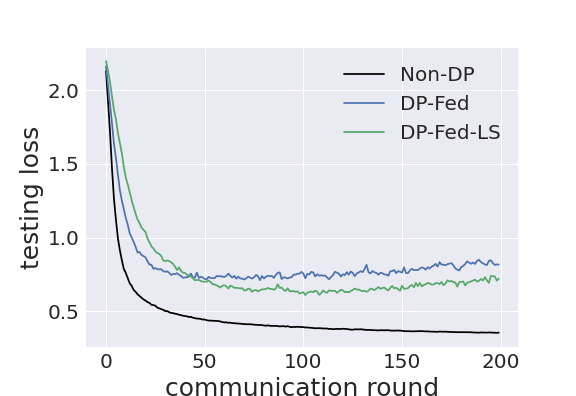}}
    \subfigure[]{\includegraphics[width=1.50in]{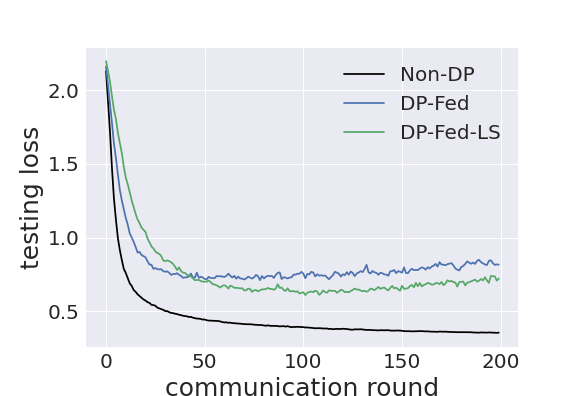}} 
    \subfigure[]{\includegraphics[width=1.50in]{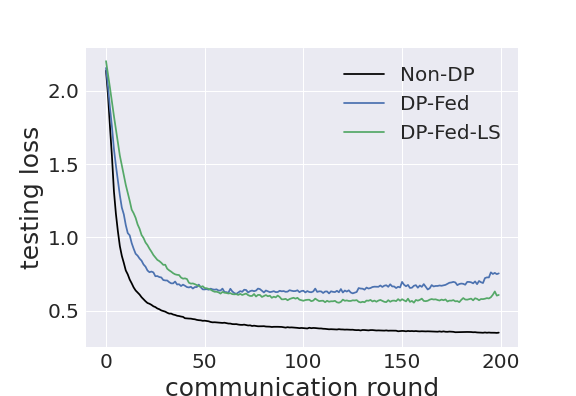}}
    \subfigure[]{\includegraphics[width=1.50in]{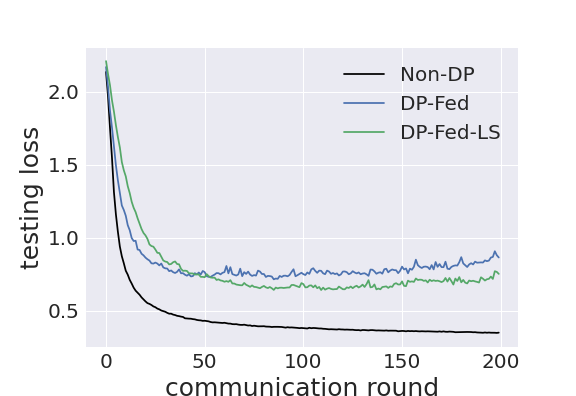}} \\
    \subfigure[]{\includegraphics[width=1.50in]{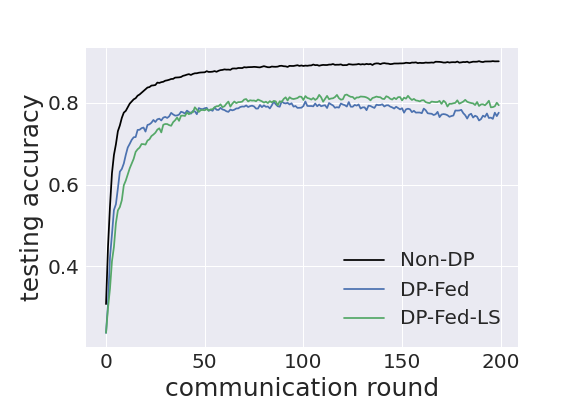}}
    \subfigure[]{\includegraphics[width=1.50in]{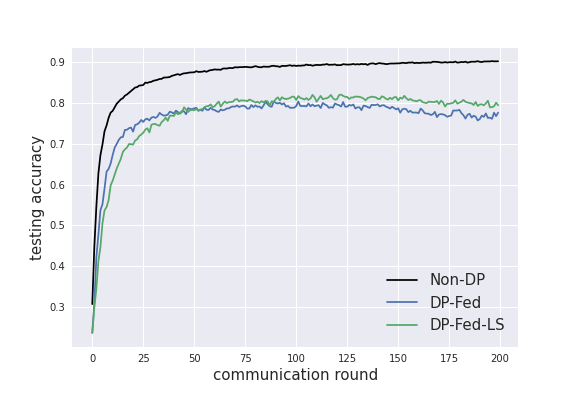}} 
    \subfigure[]{\includegraphics[width=1.50in]{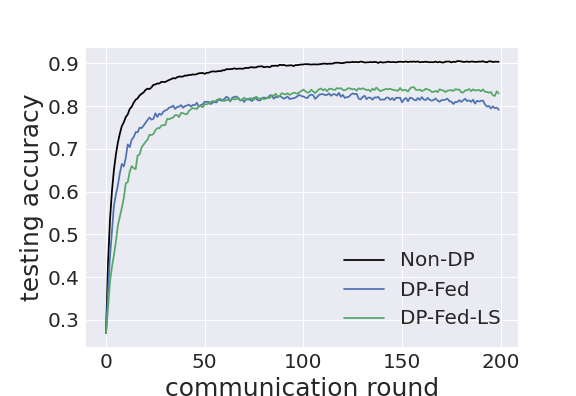}}
    \subfigure[]{\includegraphics[width=1.50in]{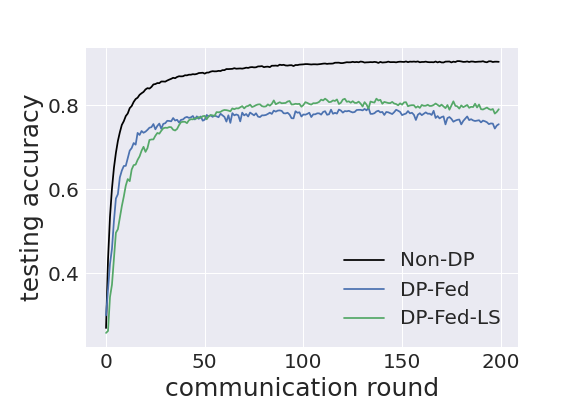}}
    \caption{Training curves of experiment with threshold attack (Section \ref{sec:mi}). (a)(e)(i): training loss, testing loss, test accuracy for different models with uniform subsampling, where noise multiplier $z=1$ for DP models. (b)(f)(j): training loss, test loss, test accuracy for models with uniform subsampling where $z=1.3$ for DP. (c)(g)(k): training loss, test loss, test accuracy for models with Poisson subsampling where $z=1$ for DP. (d)(h)(l): training loss, test loss, test accuracy for models with Poisson subsampling where $z=1.3$ for DP. In all these experiment, the Laplacian smoothing parameter $\sigma=1$.}
    \label{mi_svhn_loss_acc}
\end{figure}

\section{Laplacian Smoothing} \label{sec:ls}

In Figure~\ref{fig:demo_ls}, we compare the evolution curves of Gradient Descent (GD) and Laplacian smoothing Gradient Descent (LSGD). We can notice that the curve (Figure~\ref{fig:demo_ls} (b)) of LSGD is much more smoother than the one of GD. 

To illustrate the Proposition \ref{prop:ls-risk}, in Figure~\ref{fig:demo_ls_2}, we show the efficacy of Laplacian smoothing. We consider signal $y=\sin(x), x\in[0,30]$. We perturb it by Gaussian noise: $\tilde{y} = y + \textbf{n}$ where $\textbf{n} \sim \mathcal{N}(0,\nu^2)$. Then we get the Laplacian smoothing estimate $\hat{y}_{LS}:=\arg\min_u \|u-\tilde{y}\|^2 + \sigma \|\nabla u\|^2$. From Figure~\ref{fig:demo_ls_2} (a), we notice that $\hat{y}_{LS}$ can significantly smooth the noisy signal. Then in Figure~\ref{fig:demo_ls_2} (b), we compute the MSE reduction ratio of Laplacian smoothing estimator: $(\text{MSE}(\tilde{y}) - \text{MSE}(\hat{y}_{LS})) / \text{MSE}(\tilde{y})$ to demonstarte the efficacy of Laplacian smoothing. We see that, when the noise level $\nu$ is small, Laplacian smoothing will introduce higher MSE: the bias introduced by Laplacian smoothing is larger than its variance reduction. However, once the noise level increases, Laplacian smoothing will significantly reduce the MSE. The larger the $\sigma$ is, the more MSE reduction achieved.

\begin{figure}[htbp]
    \centering
    \subfigure[]{\includegraphics[width=2.5in]{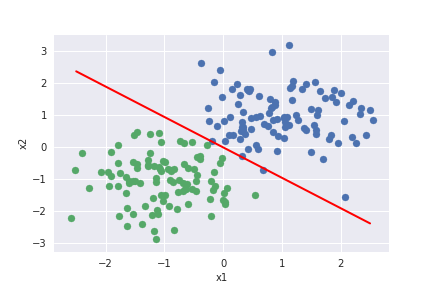}}
    \subfigure[]{\includegraphics[width=2.5in]{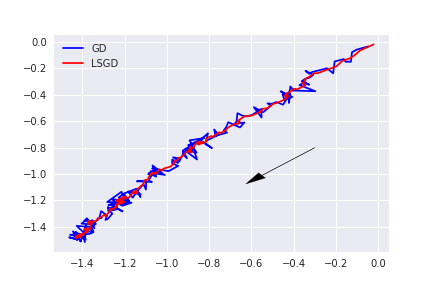}}
    \caption{Demonstration of Laplacian smoothing. We try to use a linear classifier $y=\text{sigmoid}(Wx)$ to separate data points from two distributions, i.e., the blue points ($y=0$) and the green points ($y=1$) in (a). We use gradient descent (GD) and Laplacian smoothing gradient descent (LSGD with $\sigma=1$) with binary cross entropy loss to fulfill this task. Here $W$ is intialized as (0,0) and its perfect solution would be (c,c) for any $c<0$. Gaussian noise with standard deviation of 0.3 is added on the gradients. Learning rate is set to be 0.1. In (b), we plot the evolution curves of $W$ in 100 updates, where we can find that the curve of LSGD is much smoother than the one of GD.}
    \label{fig:demo_ls}
\end{figure}

\begin{figure}[htbp]
    \centering
    \subfigure[]{\includegraphics[width=2.in]{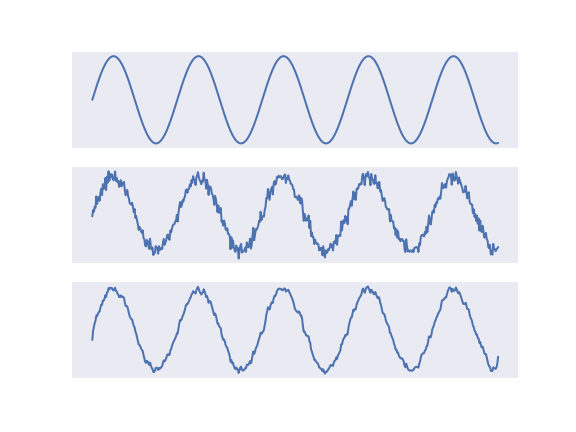}}
    \subfigure[]{\includegraphics[width=2.in]{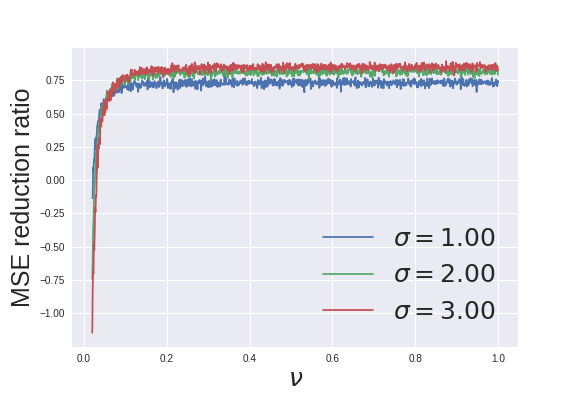}}
    \caption{Efficacy of Laplacian smoothing. In (a), signals from top to bottom are $y=\sin(x)$, $\tilde{y}=y + \mathcal{N}(0, \nu^2)$ with $\nu=0.1$ and $\hat{y}_{LS}$ with $\sigma=1$. In (b), we compute the MSE reduction ratio of Laplacian smoothing estimator $(\text{MSE}(\tilde{y}) - \text{MSE}(\hat{y}_{LS})) / \text{MSE}(\tilde{y})$ along different noise level $\nu$, where $\hat{y}_{LS}:=\arg\min_u \|u-\tilde{y}\|^2 + \sigma \|\nabla u\|^2$.}
    \label{fig:demo_ls_2}
\end{figure}

\section{Details about Table~\ref{tbl:rate-comparision} and Corollory~\ref{corollary-communication-round}}\label{sec:table_compare}
In Table~\ref{tbl:rate-comparision}, for \cite{khaled2020tighter}, the $\log(T)$ term in denominators are ignored. For the communication complexity with strongly-convex condition for \cite{karimireddy2020scaffold} and DP-Fed-LS, the $\log{\sample}$ and $\log{\step}$ terms in numerator are ignored.

For Corollary~\ref{corollary-communication-round}, given fixed noise level $\nu_1$ and communication round T, we would like to determined what $(\varepsilon,\delta)$-DP can be achieved.
Following from Theorem~\ref{Theorem-privacy-guarantee-federated-uniform}, we know that to satisfy $(\varepsilon, \delta)$-DP, we need \begin{equation}
    \nu \geq \frac{\tau \clip}{\varepsilon}\sqrt{\frac{14T}{\lambda}\bigg( \frac{\log(1/\delta)}{1-\lambda} + \varepsilon\bigg)},
\end{equation}
and $\nu$ satisfying $\nu^2/4\clip^2 \geq \frac{2}{3}$ and $\alpha-1 \leq \frac{\nu^2}{6\clip^2}\log(1/(\tau \alpha (1+\nu^2/4\clip^2)))$ for some $\lambda\in(0,1)$, where $\alpha= \log(1/\delta)/((1-\lambda)\varepsilon) +1$.
In other words, we require
\begin{equation}\label{eq-T-sub}
    T \leq \frac{\lambda \varepsilon^2 \nu_1^2}{14 \tau^2 \big(\frac{\log(1/\delta)}{1-\lambda} +\varepsilon \big)}.
\end{equation}
and\begin{equation}\label{eq-tmp-cond1}
    \frac{\nu_1^2\clip^2}{\sample^2 \Delta^2(\textbf{q})} = \frac{\nu_1^2}{4} \geq \frac{2}{3}
\end{equation}
and
\begin{equation}\label{eq-tmp-cond2}
    \alpha - 1\leq \frac{\nu_1^2}{6} \ln \frac{1}{\tau\alpha(1+\nu_1^2/4)}.
\end{equation}
for $\nu_1=\nu/\clip$. In other words, if $\nu_1\geq8/3$ and $\alpha - 1\leq \frac{\nu_1^2}{6} \ln \frac{1}{\tau\alpha(1+\nu_1^2/4)}$, then $(\varepsilon,\delta)$-DP satisfying Eq~(\ref{eq-T-sub}) can be achieved for any $\lambda\in(0,1)$.

In Theorem~\ref{Theorem-privacy-guarantee-federated-uniform}, we select $\lambda \in (0,1)$ such that $\nu_1$'s lower bound can satisfy two inequalities Eq. (\ref{eq-tmp-cond1}) and (\ref{eq-tmp-cond2}). However, in Corollary~\ref{corollary-communication-round}, our first step is to fix the noise level $\nu_1$ such that it directly satisfies Eq. (\ref{eq-tmp-cond1}) and (\ref{eq-tmp-cond2}). In this case, $\lambda\in(0,1)$ is a free parameter. One could select $\lambda \in(0,1)$ such that the upper bound for $T$ is maximized.

\end{document}